\documentclass{article}

\PassOptionsToPackage{numbers, compress}{natbib}

\usepackage{graphicx}
\usepackage{wrapfig}
\usepackage[final]{neurips_2021}

\usepackage[utf8]{inputenc} 
\usepackage[T1]{fontenc}    
\usepackage{hyperref}       
\usepackage{url}            
\usepackage{booktabs}       
\usepackage{nicefrac}       
\usepackage{microtype}      
\usepackage{xcolor}         
\usepackage{soul} 

\usepackage{mathtools}
\usepackage{amssymb}  
\usepackage{bbm}
\usepackage{cases}

\newtheorem{lemma}{Lemma}
\newtheorem{proposition}{Proposition}
\newtheorem{corollary}{Corollary}
\newtheorem{remark}{Remark}
\newtheorem{definition}{Definition}
\newenvironment{proof}
{\par\noindent{\bf Proof.}} 
{\hfill$\scriptstyle\blacksquare$}

\newcommand{\norm}[1]{\left \Vert #1 \right \Vert}
\newcommand{\scalarprod}[2]{\left \langle #1, #2 \right \rangle}

\title{On the Periodic Behavior of Neural Network Training with Batch Normalization and Weight Decay}

\author{Ekaterina Lobacheva$\bf{}^{1}$\thanks{First two authors contributed equally.}\:, Maxim Kodryan$\bf{}^{1}$\footnotemark[1]\:, Nadezhda Chirkova$\bf{}^{1}$\\ 
\bf{Andrey Malinin$\bf{}^{1,2}$, Dmitry Vetrov$\bf{}^{1,3}$}\\
	${}^1$HSE University \quad ${}^2$Yandex \quad ${}^3$AIRI\\
	Moscow, Russia\\
	{\tt elobacheva@hse.ru, mkodryan@hse.ru, nchirkova@hse.ru} \\
	{\tt am969@yandex-team.ru, dvetrov@hse.ru} 
}

\begin{document}

\maketitle

\begin{abstract}
    Training neural networks with batch normalization and weight decay has become a common practice in recent years. In this work, we show that their combined use may result in a surprising periodic behavior of optimization dynamics: the training process regularly exhibits destabilizations that, however, do not lead to complete divergence but cause a new period of training. We rigorously investigate the mechanism underlying the discovered periodic behavior from both empirical and theoretical points of view and analyze the conditions in which it occurs in practice. We also demonstrate that periodic behavior can be regarded as a generalization of two previously opposing perspectives on training with batch normalization and weight decay, namely the equilibrium presumption and the instability presumption.
\end{abstract}

\section{Introduction}

Normalization approaches, such as batch or layer normalization, have become vital for the successful training of modern deep neural networks~\cite{ioffe2015batch,ba2016layer,ulyanov2016instance,salimans2016weight,wu2018group}. Despite much recent work~\cite{bjorck2018understanding,santurkar2018how,ghorbani2019investigation,yang2018mean}, it is still not completely understood how normalization influences the training process. 
In this work, we investigate the surprising periodic behavior that may occur 
when a neural network is trained 
with a commonly used combination of some kind of normalization, in our case batch normalization (BN)~\cite{ioffe2015batch}, and weight decay regularization (WD). Examples of this behavior are provided in Figure~\ref{fig:cycles_demo_intro}.  

The dynamics of neural network training with BN and WD have been examined extensively in literature due to the non-trivial competing influence of BN and WD on the norm of neural network's weights. More precisely, using BN makes (a part of) neural network's weights \emph{scale-invariant}, i.e., multiplying them by a positive constant does not change the network's output. Although scale invariance allows optimizing on a sphere with a fixed weight norm~\cite{cho2017riemannian}, classic SGD-based approaches are usually preferred over constraint optimization methods in practice due to more straightforward implementation. Making an SGD step in the direction of the loss gradient always increases the norm of scale-invariant parameters, while WD aims at decreasing the weight norm (see illustration in Figure~\ref{fig:si_ill}). In sum, training the neural network with BN and WD results in an interplay between two forces: a ``centripetal force'' of the WD and the ``centrifugal force'' of the loss gradients. Many works notice the positive effect of WD on optimization and generalization caused by the control of the scale-invariant weights norm and the subsequent influence on the \emph{effective learning rate}~\cite{van2017l2,hoffer2018norm,zhang2018three,li2020exponential,li2020reconciling,wan2020spherical,roburin2020spherical}, i.e., the learning rate on a unit sphere in the scale-invariant weights space. However, the general dynamics of the norm of the scale-invariant weights are viewed in the literature from two contradicting points, and this work is devoted to resolving this contradiction. 

\begin{figure}
  \centering
  \centerline{
  \begin{tabular}{cc}
  {\small ConvNet on CIFAR-10} & {\small ResNet-18 on CIFAR-100} \\
  \includegraphics[width=0.5\textwidth]{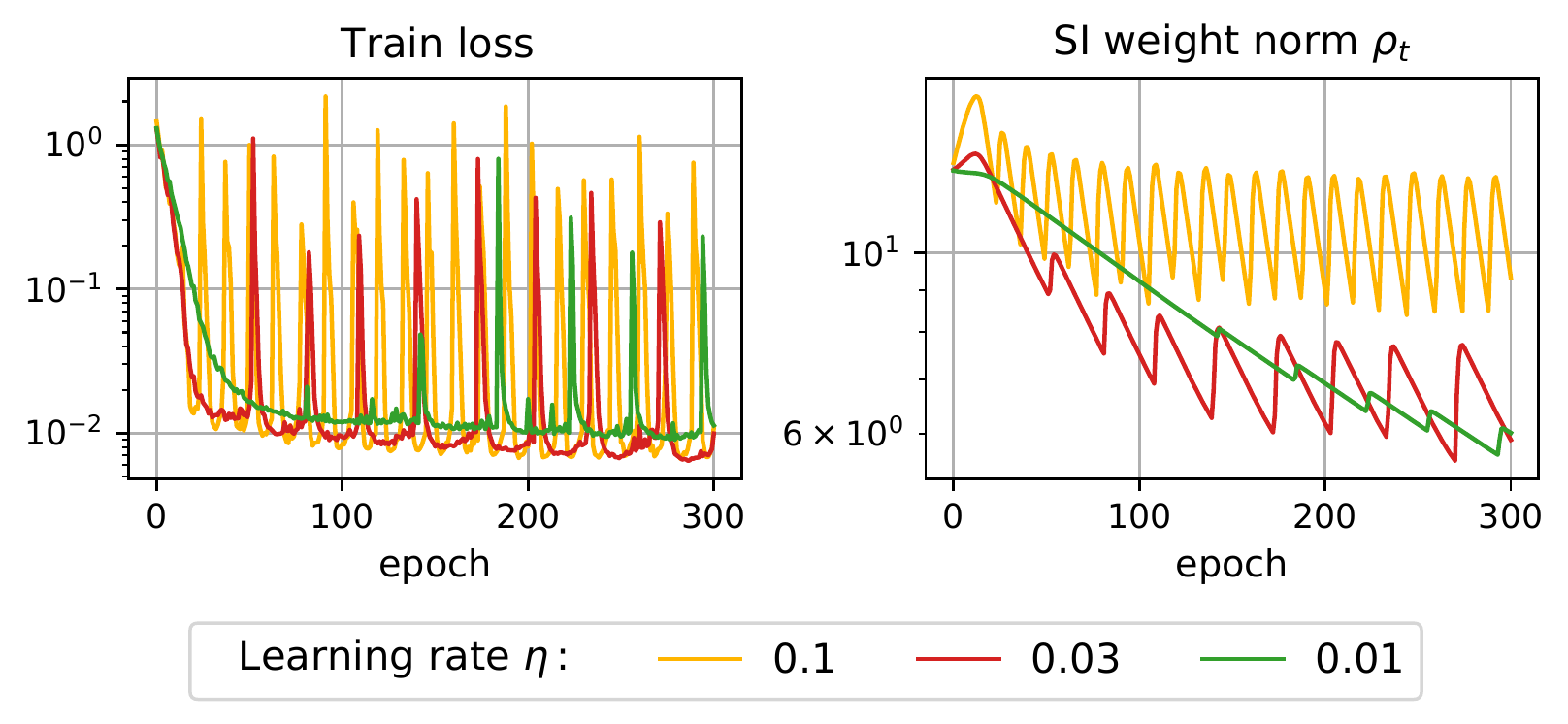} & \includegraphics[width=0.5\textwidth]{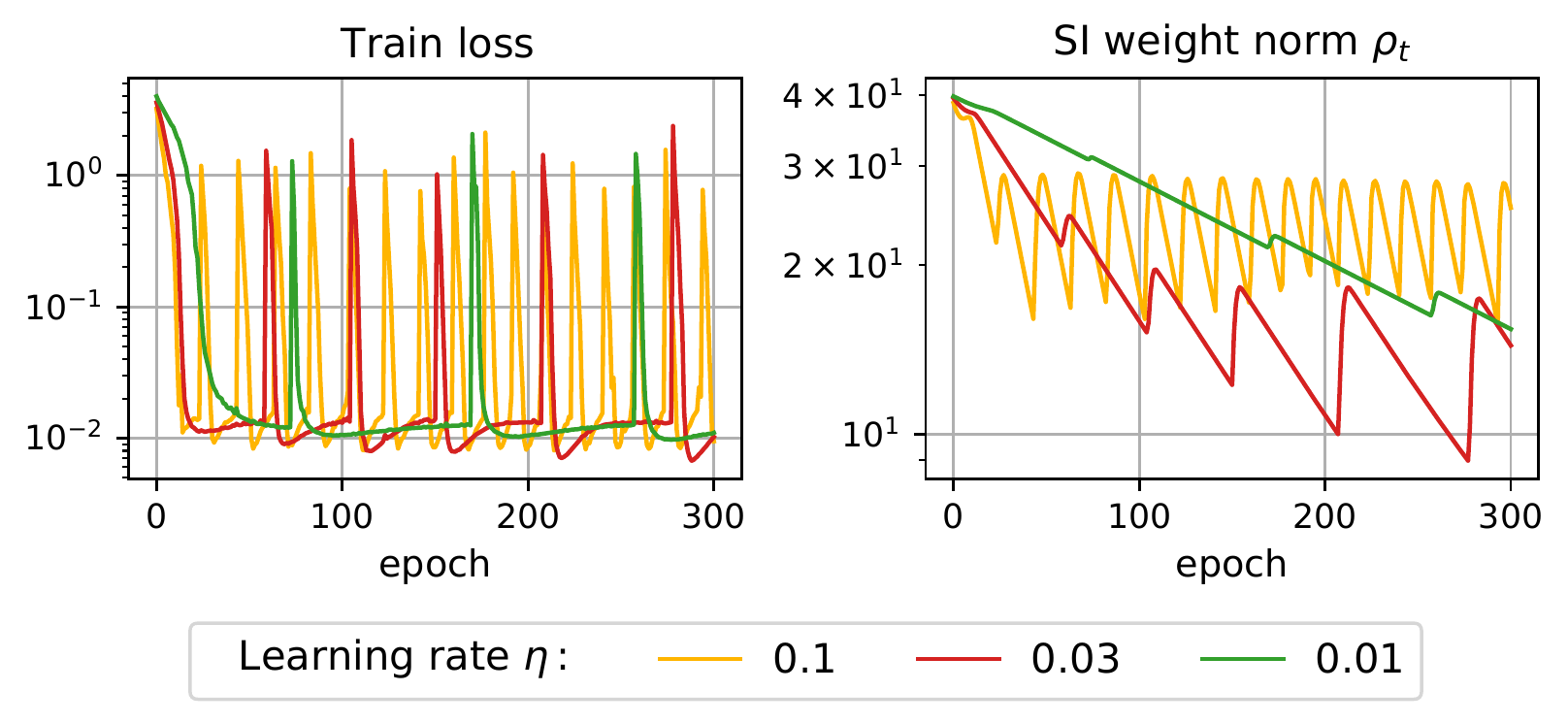}
  \end{tabular}}
  \caption{Periodic behavior of ConvNet on CIFAR-10 and ResNet-18 on CIFAR-100 trained using SGD with weight decay of 0.001 and different learning rates. All weights are trainable, including non-scale-invariant ones.
  }
  \label{fig:cycles_demo_intro}
\end{figure}

On the one hand, \citet{li2020reconciling} claim that learning with SGD, BN, and WD leads to an \emph{equilibrium} state, where the ``centripetal force'' is compensated by the ``centrifugal force'' and eventually the norm of scale-invariant weights (along with other statistics related to the training procedure) will converge to a constant value. Several other works hold a similar equilibrium view~\cite{van2017l2,chiley2019online,wan2020spherical}. On the other hand, a number of works~\cite{li2020understanding,li2020exponential,li2020reconciling} underline that using WD may cause approaching the origin (zero scale-invariant weights), which results in training \emph{instability} due to increasing effective learning rate. Particularly,~\citet{li2020understanding} reveal that approaching the origin in weight-normalized neural networks leads to numerical overflow in gradient updates and subsequent training failure. \citet{li2020exponential} also underline that scale-invariant functions are ill-conditioned near the origin and prove in a simplified setting that loss convergence is impossible if both BN and WD are used (but guaranteed if either of them is disabled). Moreover, despite their equilibrium view, \citet{li2020reconciling} empirically observe that the train loss permanently exhibits oscillations between low and high values when full-batch gradient descent is used.  

In this work, we study the specified contradiction between the \emph{equilibrium} presumption and the \emph{instability} presumption and show that both are true only to some extent. Specifically, we show that the training process converges to a consistent \emph{periodic} behavior, i.e., it regularly exhibits instabilities which, however, do not lead to a complete training failure but cause a new period of training (see Figure~\ref{fig:cycles_demo_intro}). 
Thus, our contributions are as follows.
\begin{itemize}
    \item We discover the periodic behavior of neural network training with BN and WD and reveal its reasons by analyzing the underlying mechanism for fully scale-invariant neural networks trained with standard constant learning rate SGD (Section~\ref{sec:1}) or GD (Appendix~\ref{app:gd}).
    \item We provide a theoretical grounding for our findings by generalizing previous results on the equilibrium condition, analyzing the necessary conditions for destabilization of training, and relating the frequency of destabilization to the choice of hyperparameters (Section~\ref{sec:theory}). 
    \item We conduct a rigorous empirical study of this periodic behavior (Section~\ref{sec:2}) and show its presence in more practical scenarios with momentum, augmentation, and neural networks incorporating trainable non-scale-invariant weights  (Section~\ref{sec:3}), and also with Adam optimizer and other normalization techniques (Appendix~\ref{app:other_setups}).
\end{itemize}
Our source code is available at \url{https://github.com/tipt0p/periodic_behavior_bn_wd}.

\section{Background}
\label{sec:back}
As discussed in the introduction, batch normalization makes (a part of) neural network's weights scale-invariant. In this section, we describe the properties of scale-invariant functions, upon which we build our further reasoning. 
Consider an arbitrary scale-invariant function $f(x)$, i.e., $f(\alpha x) = f(x), \forall x$ and $\forall \alpha > 0$. Then two fundamental properties may be inferred, see Lemma~1.3 in \citet{li2020exponential}: 
\begin{subnumcases}{}
    \scalarprod{\nabla f(x)}{x} = 0, ~ \forall x \label{eq:si_prop_perp} \\
    \nabla f(\alpha x) = \frac{1}{\alpha} \nabla f(x), ~ \forall x,\, \alpha > 0.
    \label{eq:si_prop_hom}
\end{subnumcases}

\begin{wrapfigure}{r}{0.25\textwidth}
  \centering
    \includegraphics[width=\linewidth]{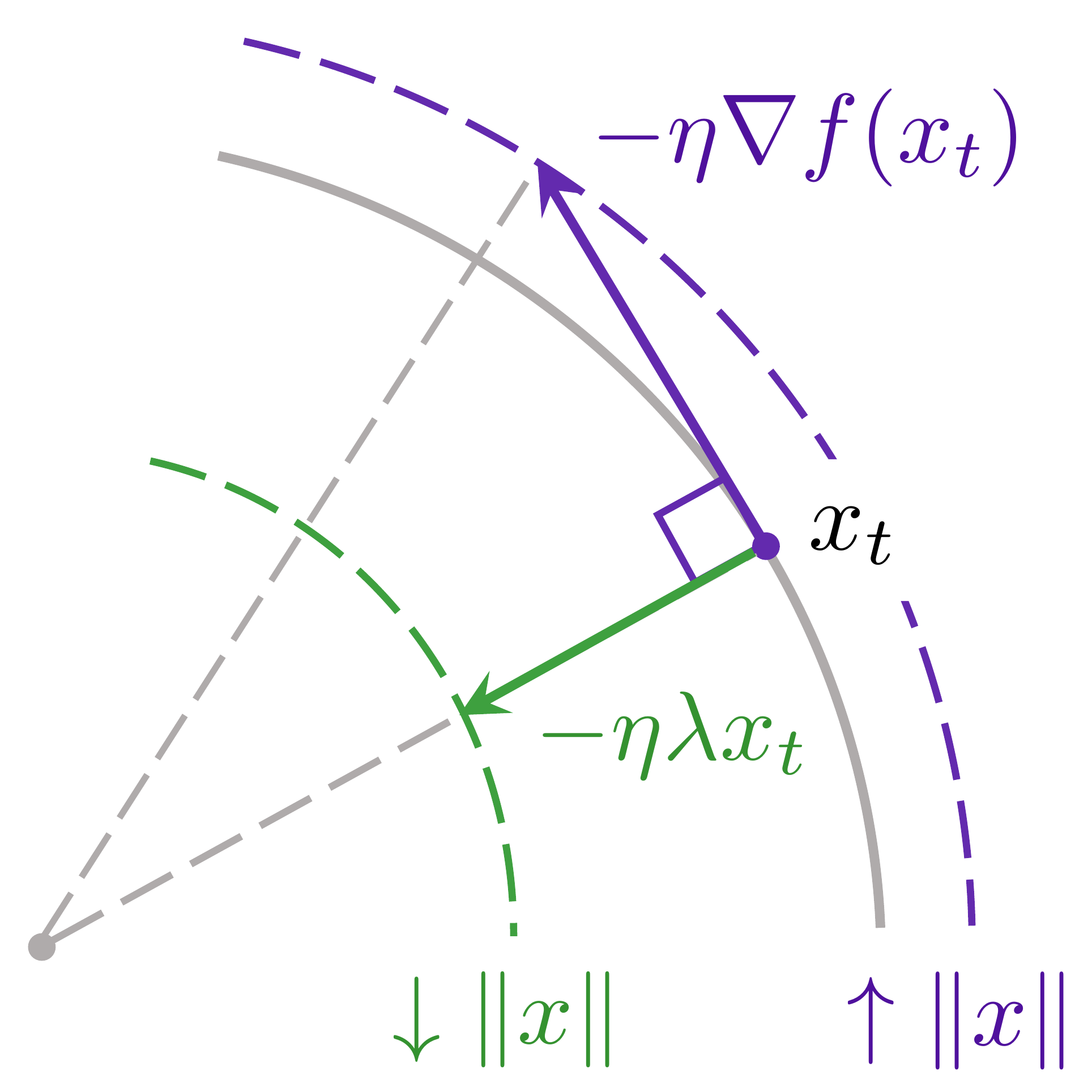}
  \caption{
  An illustration of the ``centripetal force'' of the weight decay and the ``centrifugal force'' of the function gradient in the optimization of scale-invariant functions.
  }
  \label{fig:si_ill}
\end{wrapfigure}

Consider optimizing $f(x)$ w.r.t. $x$ using (S)GD\footnote{Since both stochastic and full-batch gradients of a scale-invariant objective possess properties~\eqref{eq:si_prop_perp} and~\eqref{eq:si_prop_hom}, we do not distinguish between them in our reasoning.} with learning rate $\eta$ and weight decay $\lambda$:
\begin{equation}
    \label{eq:si_dyn}
    x_{t+1} = (1 - \eta \lambda)x_t - \eta \nabla f(x_t).
\end{equation}

The properties above lead to two important corollaries about the dynamics of the optimization process. First, according to property~\eqref{eq:si_prop_perp}, shifting $x$ in the direction of $-\nabla f(x)$, i.e., making a gradient descent step, always increases  $\norm{x}$, while weight decay, on the other hand, decreases $\norm{x}$. See Figure~\ref{fig:si_ill} for the illustration. The interaction  of these ``centripetal'' and ``centrifugal'' forces may cause $\norm{x}$ to change nontrivially during optimization. 
Second, according to property~\eqref{eq:si_prop_hom}, even though function value $f(x)$ is invariant to multiplying $x$ by $\alpha$, the optimization dynamics changes substantially when optimization is performed at different scales of $\norm{x}$. For smaller norms, optimization makes larger steps, which may result in instabilities, while for larger norms, steps are smaller, and optimization process may converge slowly. 

Since scale-invariant $f(x)$ may be seen as a function on a sphere, its optimization dynamics are often analysed on a unit sphere $\norm{x}=1$.
One can obtain equivalent optimization dynamics on the unit sphere as in the initial space by using the notion of \emph{effective gradient} and \emph{effective learning rate}.
The effective gradient is defined as a gradient for a point on a unit sphere and may be obtained by substituting $\alpha=\norm{x}^{-1}$ in \eqref{eq:si_prop_hom}:
$\nabla f(x / \norm{x}) = \nabla f(x) \norm{x}$. The effective learning rate can be defined as $\tilde{\eta} = \eta / \norm{x}^2$~\cite{hoffer2018norm,roburin2020spherical}. 
Change in $\norm{x}$ does not affect the effective gradient by definition 
and is reflected only in the effective learning rate: the lower the norm, the higher the effective learning rate, and the larger the optimization steps.

\section{Methodology and experimental setup}
\label{sec:method}

In order to isolate the effect of the joint use of batch normalization and weight decay and avoid the influence of other factors, we conduct a series of experiments in a simplified setting, when all learnable weights of a neural network are scale-invariant and optimization is performed using SGD with constant learning rate, without momentum or data augmentation. This allows us to better understand the nature of the periodic behavior (Section~\ref{sec:1}) and analyse its empirical properties (Section~\ref{sec:2}).
After that, we return to the setting with a more conventional training of standard neural networks and show that the periodic behavior occurs in this scenario as well (Section~\ref{sec:3}).

We conduct experiments with ResNet-18 and a simple 3-layer batch-normalized convolutional neural network (ConvNet)\footnote{Both architectures in the implementation of  \url{https://github.com/g-benton/hessian-eff-dim}.} on CIFAR-10~\cite{cifar10} and CIFAR-100~\cite{cifar100} datasets.
To make standard networks fully scale-invariant, we rely on the approach of~\citet{li2020exponential}, i.e., we insert additional BN layers and fix the non-scale-invariant weights to be constant. Specifically, we use zero mean and unit variance in batch normalization layers instead of learnable location and scale parameters and freeze the weights of the last layer at random initialization. 
The latter action does not hurt the performance in practice~\cite{hoffer2018fix}. However, we find that even with low train error, the training dynamics with the fixed last layer may still substantially differ from conventional training, as the neural network exhibits low confidence in predictions. To achieve high confidence for all objects and, consequently, low train loss, we increase the norm of the last layer's weights to 10. The influence of this rescaling is shown in Appendix~\ref{app:last_layer}. 

We optimize cross-entropy loss, use the batch size of 128 and train neural networks for 1000 epochs to show the consistency of the discovered periodic behavior. We consider a range of learning rates, $\{10^{-k}, 3 \cdot 10^{-k}\}_{k=0, 1, 2, 3}$ and choose the most representative ones for each visualization, since it is difficult to distinguish many periodic functions on one plot. For fully scale-invariant neural networks, training with a fixed weight decay -- learning rate product converges to similar behavior, regardless of their ratio: we show it empirically and discuss it from the theoretical point of view in Appendix~\ref{app:fixed_lr_wd}; the same was noticed in~\cite{li2020exponential,li2020reconciling}. Thus, in the main text, we provide only the results for the  varied learning rate and the fixed  weight decay of $0.001$. Results for the varied weight decay are presented in Appendix~\ref{app:var_wd}.

At each training epoch, we log standard train~/~test metrics, the norm of scale-invariant weights (SI weight norm), which is in the focus of this research, and metrics characterizing training dynamics on a unit sphere: effective learning rate and the norm of effective gradients (mean over mini-batches).
We plot the two latter metrics over two axes of the \emph{phase diagram} to visualize their simultaneous dynamics that
will help us to understand the mechanism underlying the periodic behavior.

\begin{figure}
  \centering
  \includegraphics[width=\textwidth]{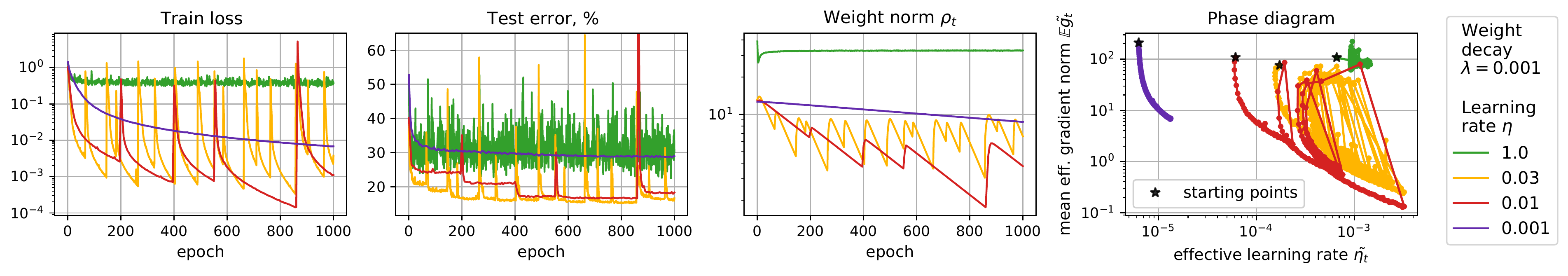}
  \caption{
  Periodic behavior of scale-invariant ConvNet on CIFAR-10.
  }
  \label{fig:cycles_demo}
\end{figure}

\section{Periodic behavior and its underlying mechanism}
\label{sec:1}
\begin{figure}
  \centering
  \includegraphics[width=\textwidth]{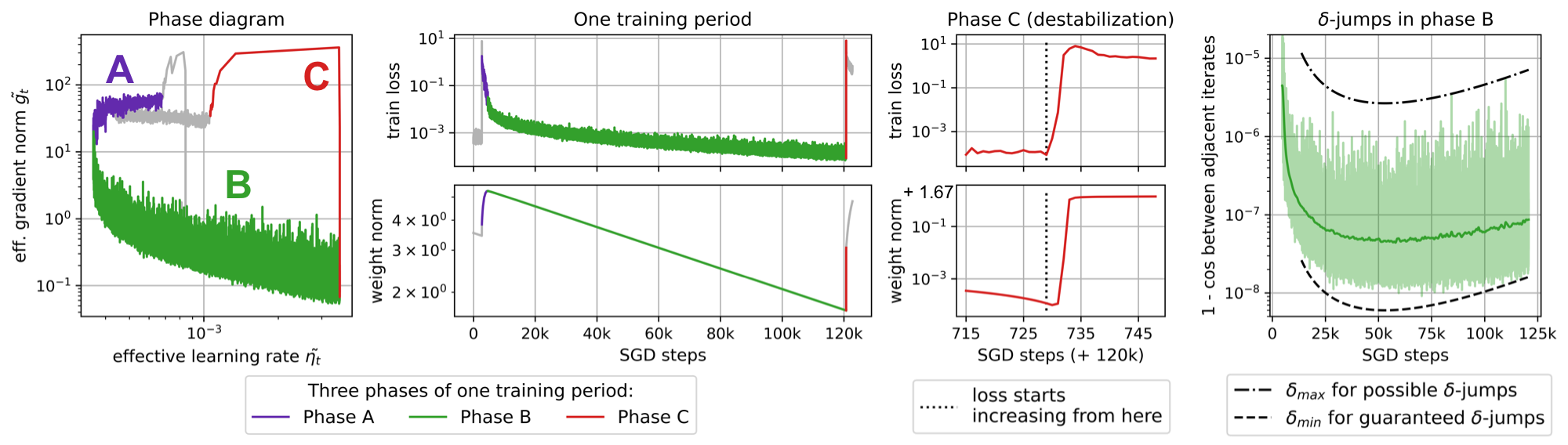} 
  \caption{
  A closer look at one training period for scale-invariant ConvNet on CIFAR-10 trained using SGD with weight decay of 0.001 and the learning rate of 0.01. Three phases of the training period are highlighted. The train loss and the effective gradient norm computed over a mini-batch are logged after each SGD step (one epoch consists of 391 SGD steps). The rightmost plot compares empirically observed cosine distance between weights at adjacent SGD steps with theoretically derived bounds in Section~\ref{sec:deltajumps}. Cosine distance is presented along with the smoothed trend.  
  }
  \label{fig:one_cycle}
\end{figure}

As discussed in the previous section, we begin our study by considering a simplified setting with a fully scale-invariant neural network trained with standard SGD.
Figure~\ref{fig:cycles_demo} shows the presence of the periodic behavior for a scale-invariant ConvNet on the CIFAR-10 dataset for a range of learning rates. In Appendix~\ref{app:fully_invariant}, we show the presence of the periodic behavior for other dataset-architecture pairs. 
The same periodic behavior is also present for neural network training with full-batch gradient descent, see Appendix~\ref{app:gd}. Moreover, this behavior can be observed even when optimizing common scale-invariant functions using the gradient descent method with weight decay (see Appendix~\ref{app:simplefunc}).

The observed periodic behavior occurs because of the interaction between batch normalization and weight decay, particularly because of their competing influence on the weight norm. As discussed in Section~\ref{sec:back}, weight decay aims at decreasing the weight norm, while loss gradients aim at increasing the weight norm due to scale invariance caused by batch normalization (see Figure~\ref{fig:si_ill}). These two forces alternately outweigh each other for quite long periods of training, resulting in periodic behavior. 

Let us examine a single period in greater detail by analyzing Figure~\ref{fig:one_cycle} that shows the dynamics of relevant training metrics logged after each SGD step of ConvNet training. At the beginning of the period, the train loss is high, and the large gradients of the loss outweigh weight decay. This results in increasing weight norm and decreasing effective learning rate, i.e., we move along phase $A$ of the phase diagram. SGD continues optimizing train loss, and at some point, train loss and its gradients become small and outweighed by weight decay. As a result, the weight norm starts decreasing, and the effective learning rate increases, i.e., we move along phase $B$ of the phase diagram. We note that the transition between phases $A$ and $B$ correlates with achieving near-zero train error.
When the weight norm becomes too small, and the effective learning rate becomes too high, SGD makes several large steps and leaves the low loss region. Gradients grow along with train loss and, multiplied by a high effective learning rate, lead to the fast growth of the weight norm, i.e., we move along phase $C$ of the phase diagram. The detailed plot of phase $C$ in Figure~\ref{fig:one_cycle} confirms that train loss starts increasing earlier than the weight norm. When the weight norm becomes large, the effective learning rate becomes low and stops the process of divergence. After that, a new period of training begins.

\begin{figure}
  \centering
  \centerline{
  \begin{tabular}{cc}
  {\small Fix weight norm at initialization} & {\small Fix weight norm before destabilization}  \\
  \includegraphics[width=0.5\textwidth]{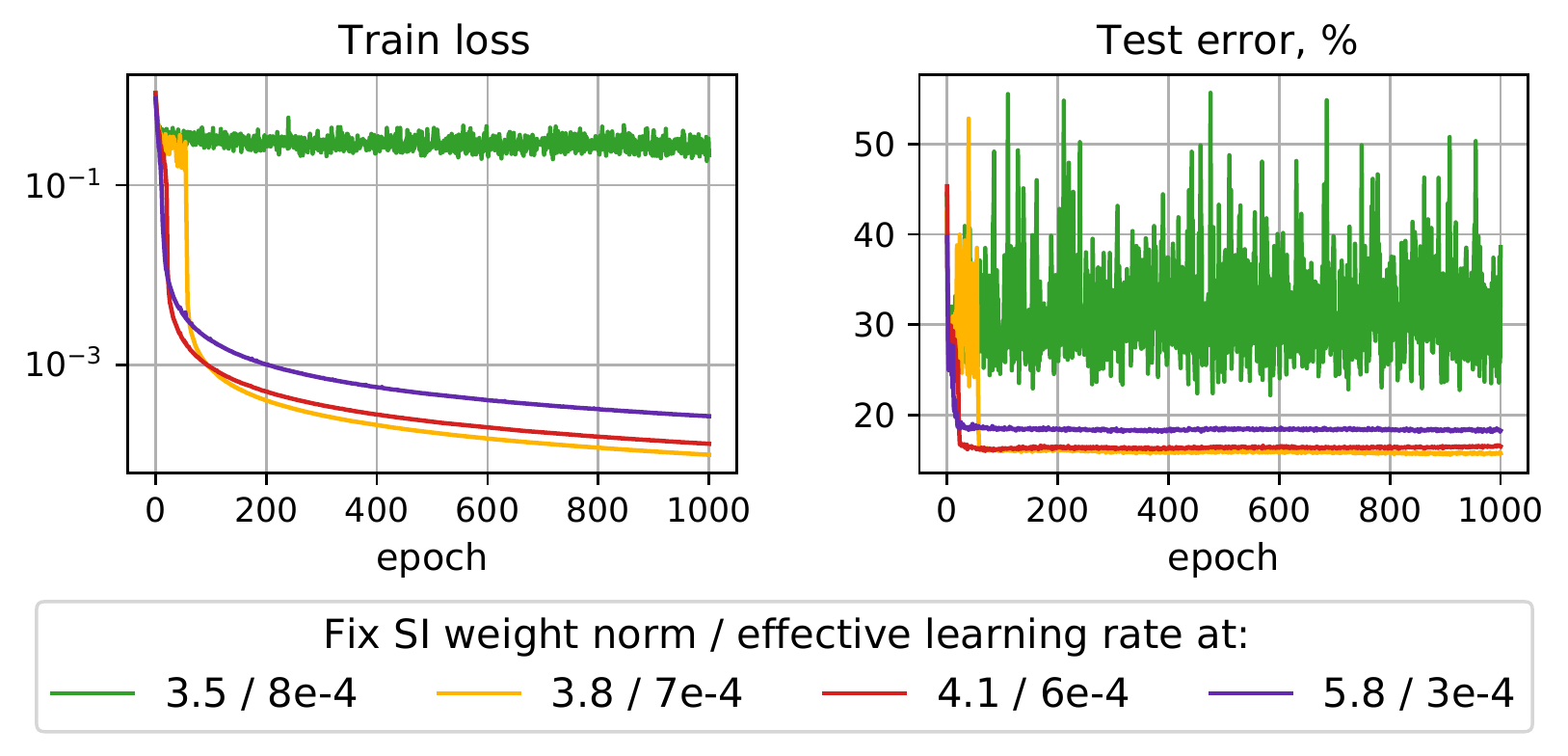} & \includegraphics[width=0.5\textwidth]{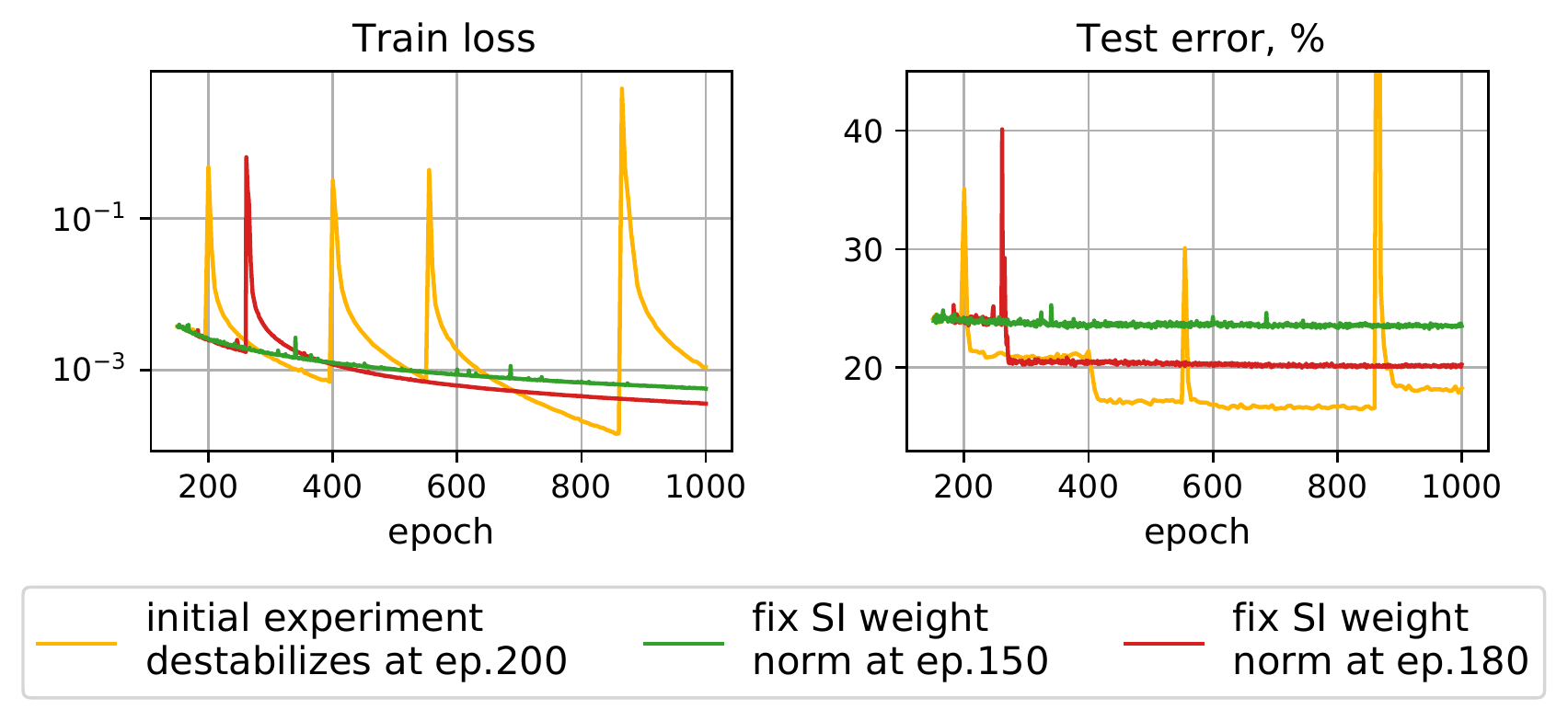}
  \end{tabular}}
  \caption{
  The absence of the periodic behavior for training with the fixed weight norm. Scale-invariant ConvNet on CIFAR-10 trained using SGD with weight decay of 0.001 and learning rate of 0.01. Left pair: the weight norm is fixed at random initialization of different scales. Right pair: the weight norm is fixed at some epoch of regular training before destabilization.
  }
  \label{fig:fix_elr_sgd}
\end{figure}

We also conducted an ablation experiment to show that the discovered periodic behavior is indeed a result of the competing influence of BN and WD on the weight norm.
To do so, we prohibit this influence and train the network on a sphere by fixing the weight norm and rescaling the weights after each SGD step. We firstly fix the weight norm at random initialization, considering different values of the initialization weight norm and hence different (fixed) effective learning rates. Figure~\ref{fig:fix_elr_sgd} (left pair) shows that in this case, there is no periodic behavior, and the train loss either converges (for relatively low effective learning rates) or gets stuck at high values (for high effective learning rates). We repeat this ablation fixing the weight norm at some epoch preceding destabilization in the experiment where we observe the periodic behavior. Specifically, as an initial experiment, we use the one with the learning rate of 0.01 from Figure~\ref{fig:cycles_demo} and fix the weight norm at the 150-th and 180-th epochs, preceding the destabilization at epoch 200. Figure~\ref{fig:fix_elr_sgd} (right pair) shows the absence of the periodic behavior in both cases. When we fix the weight norm closer to the destabilization at the 180-th epoch, we observe a single increase in train loss, as the training process has already become unstable. However, after converging from this increase, train loss never destabilizes again. 

\section{Theoretical grounding for periodic behavior}
\label{sec:theory}
In this section, we theoretically investigate the reasons for the training destabilization between phases $B$ and $C$, and after that, we generalize the overall training process equilibrium condition of~\citet{li2020reconciling} taking into account the discovered periodic behavior. 
To do so, we study the optimization dynamics of an arbitrary scale-invariant function $f(x)$ trained using (S)GD with learning rate $\eta$ and weight decay of strength $\lambda$~\eqref{eq:si_dyn}.
Hereinafter, we will assume that the $\eta \lambda$ product is small, i.e., we  can suppress $\mathcal{O}\left((\eta \lambda)^2\right)$ terms. We also refer to Appendix~\ref{app:theory} for the proofs, derivations, and further discussion on our theoretical results.

We recall that (stochastic) gradients of an arbitrary scale-invariant function $f(x)$ possess two fundamental properties~\eqref{eq:si_prop_perp} and~\eqref{eq:si_prop_hom}.
Based on these properties, we obtain the dynamics of the parameters norm induced by Eq.~\eqref{eq:si_dyn} which we also leverage in our analysis (derivation of this and other equations is deferred to Appendix~\ref{app:derivations}):
\begin{equation}
    \label{eq:si_pnorm_dyn}
    \rho_{t+1}^2 = (1 - \eta \lambda)^2 \rho_t^2 + \eta^2 \tilde{g}_t^2 / \rho_t^2,
\end{equation}
where  $\rho_t = \norm{x_t}$ denotes the parameters norm, $g_t = \norm{\nabla f(x_t)}$~--- the gradient norm, $\tilde{g}_t = \norm{\nabla f(x_t / \norm{x_t})} = \rho_t g_t$~--- the effective gradient norm.
In this work, we also use the notion of effective learning rate which is formally defined as $\tilde{\eta}_t = \eta / \rho_t^2$.

\subsection{The notion of \boldmath$\delta$-jumps}
\label{sec:deltajumps}

As scale-invariant functions are essentially defined on a sphere, cosine distance is a natural choice for a metric in parameter space. The following notion defines a situation when adjacent iterates become distant from each other, indicating training destabilization.

\begin{definition}
    We say that dynamics~\eqref{eq:si_dyn} performed a $\boldsymbol{\delta}$\textbf{-jump} once the cosine distance between adjacent iterates exceeds some value $\delta > 0$: 
    \[1 - \cos(x_t, x_{t+1}) > \delta.\]
\end{definition}

We conjecture that \emph{the necessary condition for the training dynamics' destabilization is performing $\delta$-jumps with sensible values of $\delta$}. 
Otherwise, as long as adjacent iterates remain too close, the model (and hence its training dynamics) cannot change significantly. 
This holds strictly, for instance, if the Lipschitz constant of $f$ is bounded (at least locally), which can be relevant for neural networks with BN~\cite{santurkar2018how}. 
But even if $f$ has unstable regions on a unit sphere with very high or even unbounded Lipschitz constant, our analysis is still relevant since making larger steps in such regions would lead to a higher chance of divergence. 
Further, we show that the closer we approach the origin, the larger effective steps (steps on a unit sphere) we start making, thereby paving the way for destabilization.

Now, our question is, \emph{given the value $\delta$, what are the conditions for a $\delta$-jump to occur?} By assuming that effective gradients are bounded, i.e., we can set two values $0 \le \ell \le L < +\infty$ such that $\tilde{g}_t \in \left[\ell, L\right]$, we answer this question in the following proposition. Proof can be seen in Appendix~\ref{app:jump_cond}.

\begin{proposition}
\label{prop:jump_cond}
If $f(x)$ is a scale-invariant function optimized according to dynamics~\eqref{eq:si_dyn} with bounded effective gradients $0 \le \ell \le \tilde{g}_t \le L < +\infty$, then for sufficiently small $\delta$ and assuming  $(1 - \eta \lambda) \lessapprox 1$, the following approximate conditions on $\delta$-jump hold:
\begin{subnumcases}{}
    \rho_t^2 \lessapprox \frac{\eta L}{\sqrt{2 \delta}} \implies \text{$\delta$-jump is possible}, \label{eq:jump_nec} \\
    \rho_t^2 \lessapprox \frac{\eta \ell}{\sqrt{2 \delta}}  \implies \text{$\delta$-jump is guaranteed}. \label{eq:jump_suf}
\end{subnumcases}
\end{proposition}  

\begin{remark}
    \label{rem:grad_bounds}
    Our results hold for any values $\ell$, $L$ bounding the effective gradient norm, but the tighter these bounds are, the more precisely our theory describes the properties of the actual dynamics, thus we generally assume that $\ell$ and $L$ are taken as local bounds on $\tilde{g}_t$ valid for several current iterations.
\end{remark} 

To connect our theoretical results with practice, we examine the behavior of effective steps length of a scale-invariant neural network, compare it with theoretical bounds and observe gradually increasing destabilization of training dynamics. The rightmost plot of Figure~\ref{fig:one_cycle} visualizes the cosine distance $1 - \cos(x_t, x_{t+1})$ between neural network's weights $x_t$ and $x_{t+1}$ at adjacent SGD steps for phase $B$ of the training period. The dashed lines denote the theoretical upper and lower bounds on the cosine distance corresponding to the maximal and minimal $\delta$-jumps derived from Eq.~\eqref{eq:jump_nec} and~\eqref{eq:jump_suf}, respectively: $\delta_{\max} = \frac{\eta^2 L^2}{2 \rho_t^4}$, $\delta_{\min} = \frac{\eta^2 \ell^2}{2 \rho_t^4}$. To obtain those, we calculated the network's parameters norm $\rho_t$ at each iteration and chose $\ell$ and $L$ as smooth functions locally bounding the effective gradient norm in phase $B$ (see Appendix~\ref{app:grad_bounds} for details). We can see that both bounds, along with the measured cosine distance, start growing in the second half of the phase. This indicates that the performed $\delta$-jumps are gradually increasing, hence instability accumulates until the training diverges. We note that such a long-lasting increase in cosine distance is common but, in general, not obligatory in the case of training with SGD because SGD may exhibit destabilization even with small $\delta$-jumps due to stochasticity. For full-batch GD, this effect is even more prominent, see Appendix~\ref{app:gd}.

Next, we formulate a proposition about how the initial parameters norm value $\rho_0$ and hyperparameters $\eta$ and $\lambda$ affect the time of $\delta$-jumps occurrence and hence the frequency of the periods since training dynamics destabilization is closely connected with $\delta$-jumps. 
Proof can be found in Appendix~\ref{app:jump_time}.
Note that $\rho_0$ should be interpreted as the norm at some initial moment $t = 0$ of a given period, when the conditions of the proposition are met, (typically, at the beginning of phase B) rather than the norm after initialization, i.e., at the very first iteration of training. 

\vspace{1cm}
\begin{proposition}
\label{prop:jump_time}
Denote $\kappa = \sqrt{\frac{\eta}{2 \lambda}}$.
Under the assumptions of Proposition~\ref{prop:jump_cond}:
\begin{enumerate}
    \item if $\rho_0^2 > \kappa \ell$ and $\delta < \eta \lambda \frac{L^2}{\ell^2}$, then the \textbf{minimal} time required for the $\delta$-jump to occur:
    \begin{equation}
        \label{eq:t_min}
        t_{\min} = \max \left\{0, \frac{\log\left(\rho_0^2 - \kappa \ell\right) - \log\left(\frac{\eta L}{\sqrt{2 \delta}} - \kappa \ell \right)}{-\log(1 - 4\eta \lambda)} \right\};
    \end{equation}
    \item if $\rho_0^2 > \kappa L$ and $\delta < \eta \lambda \frac{\ell^2}{L^2}$, then the \textbf{maximal} time required for the $\delta$-jump to occur:
    \begin{equation}
        \label{eq:t_max}
        t_{\max} = \max \left\{0, \frac{\log\left(\rho_0^2 - \kappa L\right) - \log\left(\frac{\eta \ell}{\sqrt{2 \delta}} - \kappa L \right)}{-\log(1 - 2 \eta \lambda)} \right\}.
    \end{equation}
\end{enumerate}
\end{proposition}

\begin{corollary}
\label{cor:jump_time}
    Since both $t_{\max}$ and $t_{\min}$ are inversely proportional to $\eta \lambda$ as $-\log(1 - \varepsilon) \approx \varepsilon$ for small $\varepsilon$,
    $\delta$-jumps (and hence periods) must occur more often for larger values of $\eta \lambda$.
\end{corollary}

\subsection{Generalization of the equilibrium condition}

We now generalize the equilibrium condition of~\citet{li2020reconciling} and characterize the behavior of the parameters norm globally in the following proposition.
The proof is provided in Appendix~\ref{app:norm_eq}.

\begin{proposition}
\label{prop:norm_eq}
Denote $\kappa = \sqrt{\frac{\eta}{2 \lambda}}$.
Under the assumptions of Proposition~\ref{prop:jump_cond}, if $2 \eta \lambda L \le \ell$, then 
\begin{equation}
    \label{eq:norm_eq}
    \kappa \ell \le \rho_t^2 \le \kappa L,\, t \gg 1.
\end{equation}
Furthermore, if $\rho_0^2 > \kappa L$, then $\rho_t^2$ converges linearly to $[\kappa \ell, \kappa L]$ interval in $\mathcal{O}\left(1 / \eta \lambda\right)$ time.
\end{proposition}

Note that~\citet{li2020reconciling} and~\citet{wan2020spherical} similarly predict that the equilibrium state can be reached in a linear rate regime. The condition $2 \eta \lambda L \le \ell$ is generally fulfilled in practice for small $\eta \lambda$ product even for globally chosen bounds $\ell$, $L$. A similar assumption is made, e.g., in Theorem~1 in~\citet{wan2020spherical}. We discuss it in more detail (including the non-fulfillment case) in Appendix~\ref{app:ell_cond}.

Proposition~\ref{prop:norm_eq} generalizes the results of~\citet{li2020reconciling} who claimed that the effective learning rate $\tilde{\eta}_t = \eta / \rho_t^2$ converges to a constant. Their derivation relies on the assumption of stabilization of the effective gradient variance, which contradicts the observed periodic behavior. We relax this assumption by putting bounds on the effective gradient norm, thus bounding the parameters norm limits. These bounds can be either local, which defines the local trend of parameters norm dynamics, or global, which describes its general behavior. Also, note that we, in some sense, extend the results of~\citet{wan2020spherical} as we provide the exact limiting interval for $\rho_t^2$, not just bound its variance.

\section{Empirical study of the periodic behavior}
\label{sec:2}

After discussing the reasons for the occurrence of the periodic behavior, we now further analyze its properties. In particular, we investigate: how hyperparameters affect the periodic behavior, how the periodic behavior evolves over epochs, and how minima achieved in different training periods differ both in parameter and functional space. In this section, we again consider the simplified setting with a fully scale-invariant neural network trained with standard SGD.

\paragraph{Influence of hyperparameters.} 
We investigate the influence of two key training hyperparameters: learning rate and weight decay, but since the training dynamics mainly depend on their product (see discussion in Section~\ref{sec:method} and Appendix~\ref{app:fixed_lr_wd}), we only vary the learning rate. The results for ConvNet on CIFAR-10 are given in Figure~\ref{fig:cycles_demo}, the results for other dataset-architecture pairs are presented in Appendix~\ref{app:fully_invariant}, and the results on the variable weight decay are given in Appendix~\ref{app:var_wd}. Our first observation is that with higher learning rates, consistent periodic behavior occurs at larger weight norms. This is because SGD with a high learning rate can only converge with relatively small gradient norms, which are achieved at large weight norms according to Eq.~\eqref{eq:si_prop_hom}. This observation also agrees with Proposition~\ref{prop:norm_eq} in Section~\ref{sec:theory}. The second observation is that the periodic behavior is present for a wide range of learning rates, e.g., $0.003-0.3$ for ConvNet on CIFAR-10, and the higher the learning rate, the shorter the periods, which agrees with Corollary~\ref{cor:jump_time} in Section~\ref{sec:theory}. When using a learning rate that is too high, e.g., 1 in Figure~\ref{fig:cycles_demo}, we expect training to yield very large weight norms, however weight decay prohibits us from reaching them, thus the gradients are not able to shrink sufficiently, and training gets stuck at high train loss. On the other hand, using a learning rate that is too low, e.g., 0.001 in Figure~\ref{fig:cycles_demo}, leads to prolonged training which does not reach a small enough weight norm to yield a high effective learning rate, resulting in the absence of the periodic behavior in the given number of epochs. We note that the periodic behavior is present for the learning rates giving the lowest test error. In Appendix~\ref{app:var_wd}, we show that varying weight decay leads to similar effects: the periodic behavior is present for a wide range of reasonable weight decays but is absent for too low or too high weight decays, and the higher the weight decay, the faster the periods.  

\paragraph{Dynamics of periodic behavior.} 
We now analyze how the discovered periodic behavior evolves over epochs. As discussed in the previous paragraph, consistent periodic behavior occurs at larger weight norms with higher learning rates. However, the initialization may have a substantially different norm. Thus we observe a \emph{warm-up stage} in some plots of Figures~\ref{fig:cycles_demo_intro} and~\ref{fig:cycles_demo}, when the beginning of training is spent on moving towards the appropriate norm of scale-invariant weights. Expectedly, this warm-up stage is more prolonged for lower learning rates. Reaching the proper weight norm initiates a consistent periodic behavior. During the warm-up stage, SGD can still exhibit regular destabilization happening at higher weight norms than in the stage of consistent periodic behavior. We hypothesize that at the early stage of training, SGD converges to less stable basins with larger effective gradients, in which destabilization happens at larger norms of the scale-invariant parameters. We notice that test error decreases after each warm-up destabilization episode and reaches a lower level than training with a fixed effective learning rate, as shown in Figure~\ref{fig:fix_elr_sgd} (right pair). In other words, the performance may benefit from the repeating destabilization. 

\begin{figure}
  \centering
  \centerline{
  \begin{tabular}{cc}
  {\small ConvNet on CIFAR-10} & {\small ResNet-18 on CIFAR-100} \\
  \includegraphics[width=0.5\textwidth]{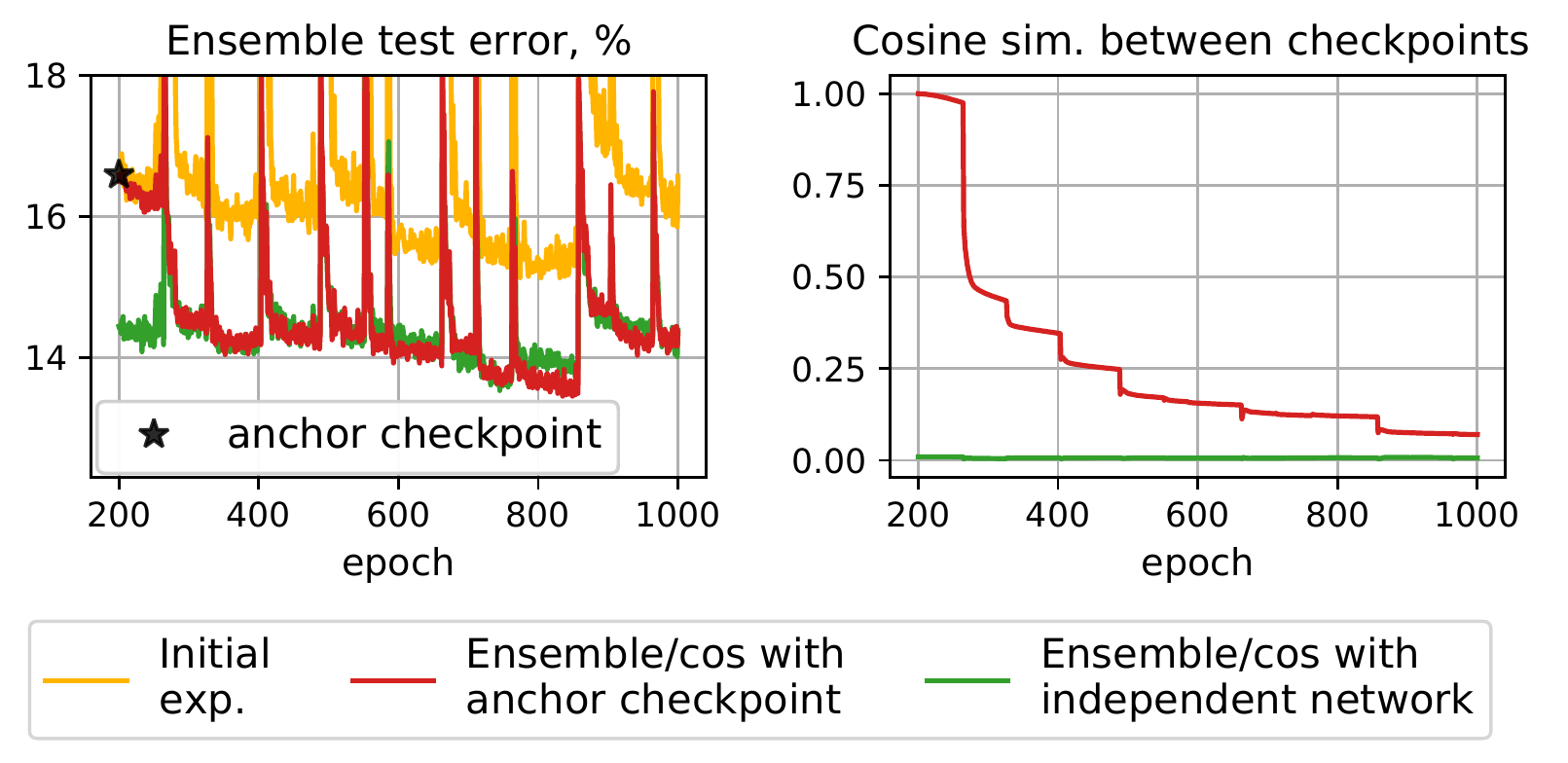} & \includegraphics[width=0.5\textwidth]{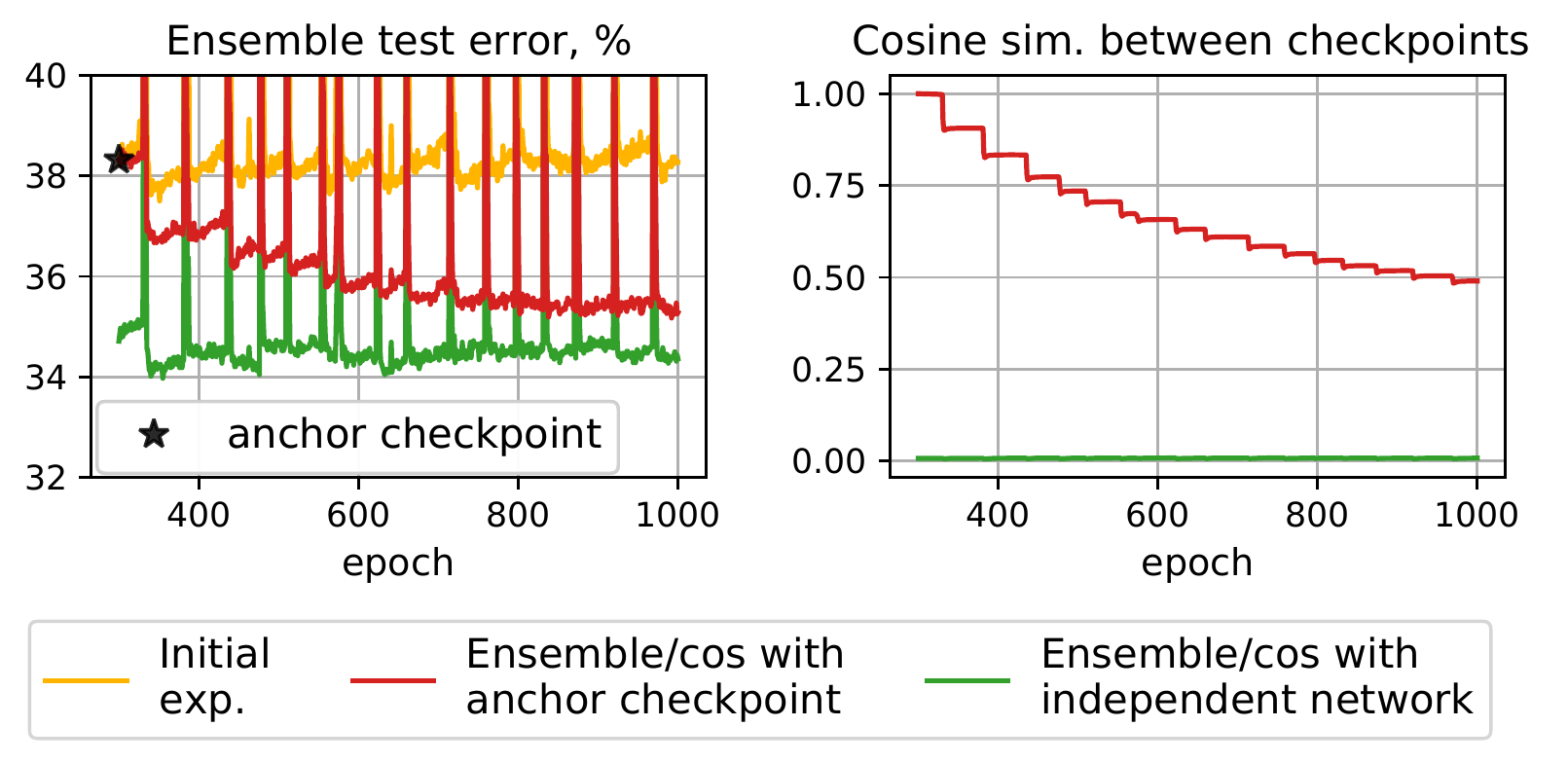}
  \end{tabular}}
  \caption{
  Similarity in the weight space (cosine sim.) and the functional space (ensemble test error) for different checkpoints of training scale-invariant ConvNet on CIFAR-10 (left pair) and ResNet on CIFAR-100 (right pair) using SGD with weight decay of 0.001 and learning rate of 0.03. 
  }
  \label{fig:ensembles}
\end{figure}

\paragraph{Minima achieved at different training periods.} 
Next, we aim at understanding whether minima achieved in different training periods are close in weight and functional spaces. We use the cosine similarity function for the weight space and estimate similarity in the functional space by comparing with ensembles of independent models, following~\citet{fort2019deep}. 
If training process converges to the same minimum in each period, then cosine similarity between two minima achieved in different periods should be close to one and their ensemble error should be close to the error of a single network. On the contrary, if destabilization moves training so far that it is equivalent to retraining a model from a new random initialization, then the cosine similarity between the two minima should be close to zero and their ensemble error should be close to the error of an ensemble of two independently trained networks.

The setup of the experiment is as follows. We select some initial experiment and its checkpoint (called anchor checkpoint) corresponding to the minimum achieved when the training process has already converged to the consistent periodic behavior. After that, we measure weight/function similarities between the anchor checkpoint and all the subsequent checkpoints of the initial experiment~--- this is our primary measurement. 
For comparison, we independently train one more neural network with the same hyperparameters as in the initial experiment but from a different random initialization, select its checkpoint with the same test error as that of the anchor checkpoint, and measure the similarity between this new checkpoint and the checkpoints of the original experiment.

The results for ConvNet on CIFAR-10 and ResNet-18 on CIFAR-100 are presented in Figure~\ref{fig:ensembles}, the results for other dataset-architecture pairs are given in Appendix~\ref{app:ensembles}. Inside one training period, checkpoints do not step far from the anchor checkpoint, i.e., the cosine similarity is close to one, and the ensemble test error is close to the error of a single network. However, when the next training period begins after destabilization, SGD moves to another  region in the weight space, and both similarities start decreasing: the cosine similarity drops, and ensemble test error becomes smaller than that of a single network. Each following training period moves SGD farther away from the anchor checkpoint. For networks on CIFAR-100, late training periods continue to be correlated with the anchor checkpoint, i.e., the cosine similarity only reaches $\sim0.5$ value, and ensemble test error does not reach the level of the independent networks ensemble.  Still, both similarities continue decreasing. For networks on CIFAR-10, the cosine similarity decreases faster, and the ensemble test error quickly reaches the test error of an ensemble of two independently trained networks.  To sum up, minima achieved at two neighboring training periods are substantially different, but their similarity is usually higher than that of two independently trained networks. 

\section{Periodic behavior in a practical setting}
\label{sec:3}

\begin{figure}
  \centering
  \centerline{
 \begin{tabular}{cc}
 {\small Add modifications one at a time} & {\small Add all modifications together} \\
  \includegraphics[width=0.5\textwidth]{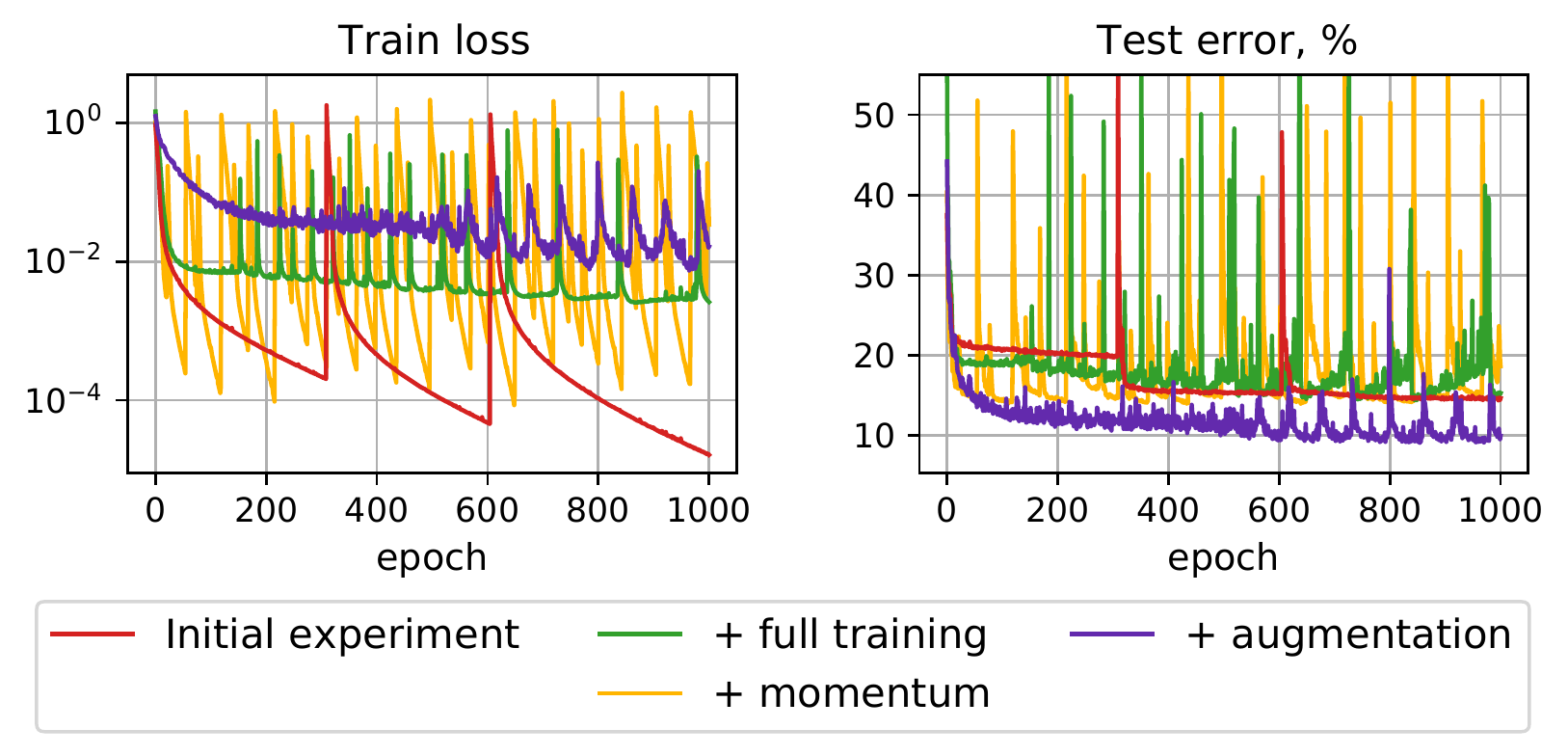} & \includegraphics[width=0.5\textwidth]{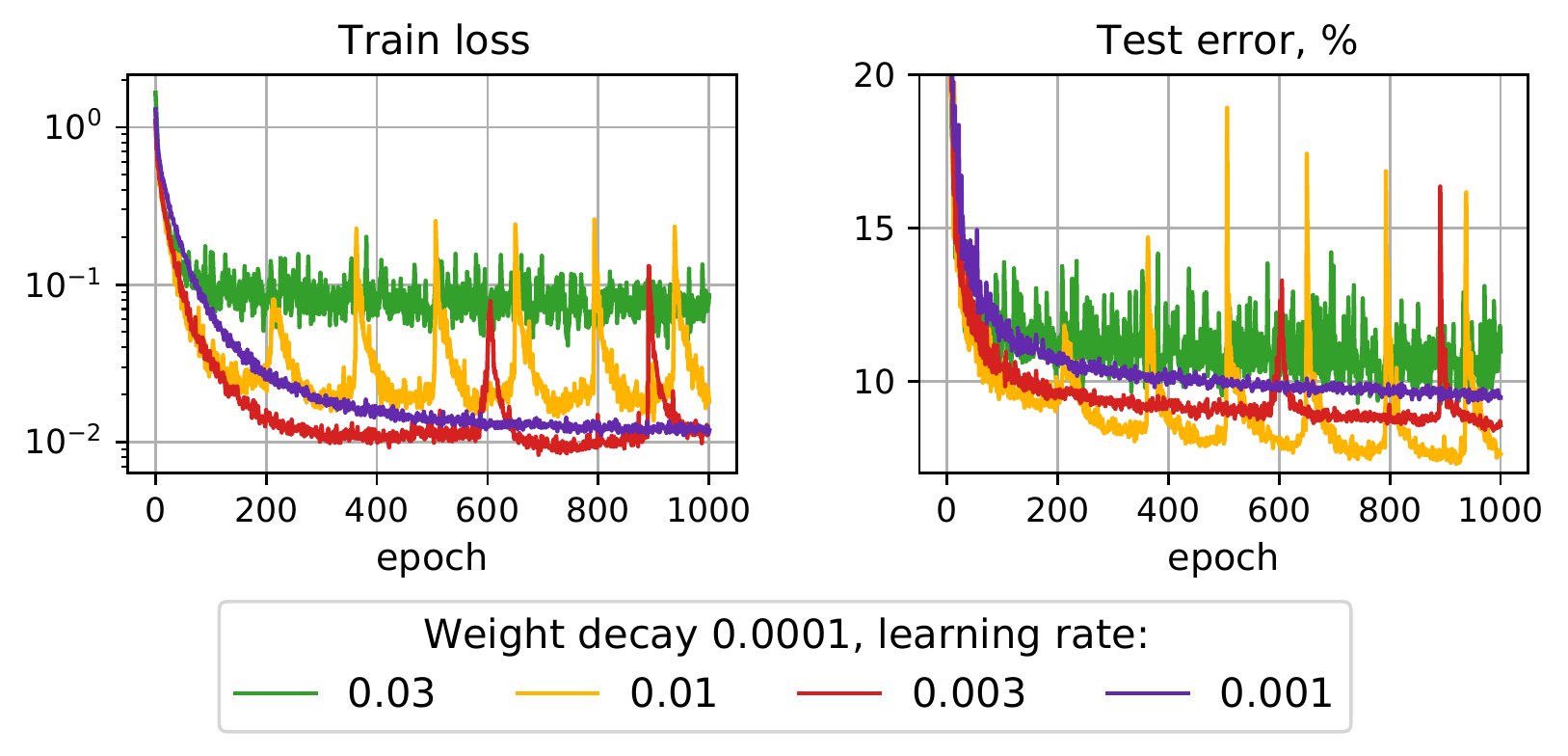}
  \end{tabular}}
  \caption{Periodic behavior of ConvNet (of increased width) on CIFAR-10 trained with the more practical modifications. Left: weight decay of 0.001, learning rate of 0.01.
  }
  \label{fig:mom_aug_full}
\end{figure}

In the previous sections, we conducted experiments with scale-invariant neural networks trained with the simplest version of SGD. This allowed us to analyze the periodic behavior of train loss in detail. However, in practice, a portion of the weights of a neural network are not scale-invariant, e.g., the weights of the last layer and learnable BN parameters. Furthermore, networks are trained using more advanced procedures, e.g., SGD with momentum, data augmentation, and a learning rate schedule. At the same time, periodic behavior was mainly not noticed in previous works to the best of our knowledge. 
In this section, we show the presence of the periodic behavior for standard neural networks trained with momentum and data augmentation and discuss why periodic behavior may be not observed in practice. In Appendix~\ref{app:other_setups}, we also show the presence of the periodic behavior for the networks with other normalization approaches or trained with Adam~\cite{adam}.

We select one of our initial experiments and add modifications one at a time to see their effects more clearly. We also present the results for training with all modifications turned on together. The plots for ConvNet on CIFAR-10 are given in Figure~\ref{fig:mom_aug_full} and for other setups~--- in Appendix~\ref{app:other_setups}. In this section, we use a wider ConvNet, as the standard version is too small to learn the augmented dataset.

\paragraph{Training non-scale-invariant weights.}
To achieve full scale-invariance, we froze the weights of the last layer and the parameters of BN layers since they all are not scale-invariant. We now consider the conventional procedure that implies training all neural network weights. In addition to the results presented in Figure~\ref{fig:mom_aug_full} (left), we refer the reader to Figure~\ref{fig:cycles_demo_intro}. We observe that training non-scale-invariant weights retains the periodic behavior and affects the frequency of periods. The last-mentioned effect relates to the trainable last layer that automatically adjusts prediction confidence. In Appendix~\ref{app:last_layer}, we show that variable prediction confidence results in different periodic behavior.

\paragraph{SGD with momentum.} 
Next, we investigate the effect of using a more complex optimization algorithm. We consider SGD with momentum as the algorithm most commonly used for training convolutional neural networks. We observe that using momentum does not break the periodic behavior and increases the frequency of periods. This agrees with the commonly observed effect that momentum speeds up training~\cite{sutskever2013importance}, i.e., momentum accelerates phases $A$ and $B$. Interestingly, momentum does not prevent destabilization. 

\paragraph{Data augmentation.} 
We next consider training on the dataset with standard CIFAR-10 data augmentations, see details in Appendix~\ref{app:details}. Augmentation prevents over-fitting to the training data, which results in less confident predictions and larger train loss. As a result, train loss gradients outweigh WD more easily. If the number of parameters in the neural network is insufficient to achieve low train loss gradients, phase $A$ never ends (at least in 1000 epochs), resulting in the absence of the periodic behavior. On the other hand, a sufficiently large neural network learns the augmented dataset at some epoch and proceeds to phase $B$, launching the periodic process. Still, we note that the periodic behavior begins much later than for the network trained without augmentation. This is one of the main reasons why the periodic behavior is often not observed in practice: it requires a much larger number of epochs to start than conventionally used for training.

\paragraph{All modifications together.} 
In the two right plots of Figure~\ref{fig:mom_aug_full}, we visualize training with momentum, data augmentation, and unfrozen non-scale-invariant parameters used simultaneously and observe the presence of the periodic behavior. 

So, what factors do interfere with observing the periodic behavior in practice?  We underline two main factors. First, the interplay between different modifications narrows the range of hyperparameter values for which periodic behavior is present. When non-scale-invariant parameters are trained, the model converges to low test error only with specific values of weight decay. Moreover, with data augmentation, periodic behavior occurs only with relatively high learning rates (with lower learning rates, the training is too slow to reach phase $C$ in 1000 epochs), while with momentum, using too high learning rates may result in training failure in phase $A$. In sum, periodic behavior appears only for a limited range of hyperparameters. Despite that, we note that the model generally achieves its best performance exactly in this range. Second, practical settings also imply learning rate schedules and a relatively small number of epochs, which do not preserve periodic behavior.
We provide further discussion on comparison of our experimental setup with other works in Appendix~\ref{app:comp_setups}.

\section{Conclusion}
\label{sec:concl}

In this work, we described the periodic behavior of neural network training with batch normalization and weight decay occuring due to their competing influence on the norm of the scale-invariant weights. The discovered periodic behavior clarifies the contradiction between the equilibrium and instability presumptions regarding training with BN and WD and generalizes both points of view. In our empirical study, we investigated what factors and in what fashion influence the discovered periodic behavior. In our theoretical study, we introduced the notion of $\delta$-jumps to describe training destabilization, the cornerstone of the periodic behavior, and generalized the equilibrium conditions in a way that better describes the empirical observations.

\paragraph{Limitations and negative societal impact.}
We discuss only conventional training of convolutional neural networks for image classification and do not consider other architectures and tasks. However, we believe that our findings extrapolate to training any kind of neural network with some type of normalization and weight decay. We also focus on a particular source of instability induced by BN and WD, yet, other factors may make training unstable~\cite{cohen2021gradient}. This is an exciting direction for future research. To the best of our knowledge, our work does not have any direct negative societal impact. However, while conducting the study, we had to spend many GPU hours, which, unfortunately, could negatively affect the environment.

\begin{ack}
We would like to thank Sofia Sibagatova for the help with Appendix~\ref{app:other_setups}.
The work was supported in part by the Russian Science Foundation grant \textnumero 19-71-30020. 
The empirical results were supported in part through the computational resources of HPC facilities at HSE University~\cite{hpc}. Additional revenues of the authors for the last three years: laboratory sponsorship by Samsung Research, Samsung Electronics; travel support by Google, NTNU, DESY, UCM.
\end{ack}

\bibliographystyle{apalike}
\bibliography{ref}

\begin{thebibliography}{}

\bibitem[Arora et~al., 2019]{arora2018theoretical}
Arora, S., Li, Z., and Lyu, K. (2019).
\newblock Theoretical analysis of auto rate-tuning by batch normalization.
\newblock In {\em International Conference on Learning Representations}.

\bibitem[Ba et~al., 2016]{ba2016layer}
Ba, J.~L., Kiros, J.~R., and Hinton, G.~E. (2016).
\newblock Layer normalization.
\newblock {\em arXiv preprint arXiv:1607.06450}.

\bibitem[Bjorck et~al., 2018]{bjorck2018understanding}
Bjorck, N., Gomes, C.~P., Selman, B., and Weinberger, K.~Q. (2018).
\newblock Understanding batch normalization.
\newblock {\em Advances in Neural Information Processing Systems}, 31.

\bibitem[Carmon et~al., 2017]{carmon2017lower}
Carmon, Y., Duchi, J., Hinder, O., and Sidford, A. (2017).
\newblock Lower bounds for finding stationary points ii: First-order methods.
\newblock {\em Mathematical Programming}, 185.

\bibitem[Chiley et~al., 2019]{chiley2019online}
Chiley, V., Sharapov, I., Kosson, A., Koster, U., Reece, R., Samaniego de~la
  Fuente, S., Subbiah, V., and James, M. (2019).
\newblock Online normalization for training neural networks.
\newblock {\em Advances in Neural Information Processing Systems}, 32.

\bibitem[Cho and Lee, 2017]{cho2017riemannian}
Cho, M. and Lee, J. (2017).
\newblock Riemannian approach to batch normalization.
\newblock {\em Advances in Neural Information Processing Systems}, 30.

\bibitem[Cohen et~al., 2021]{cohen2021gradient}
Cohen, J., Kaur, S., Li, Y., Kolter, J.~Z., and Talwalkar, A. (2021).
\newblock Gradient descent on neural networks typically occurs at the edge of
  stability.
\newblock In {\em International Conference on Learning Representations}.

\bibitem[Fort et~al., 2019]{fort2019deep}
Fort, S., Hu, H., and Lakshminarayanan, B. (2019).
\newblock Deep ensembles: A loss landscape perspective.
\newblock {\em arXiv preprint arXiv:1912.02757}.

\bibitem[Ghorbani et~al., 2019]{ghorbani2019investigation}
Ghorbani, B., Krishnan, S., and Xiao, Y. (2019).
\newblock An investigation into neural net optimization via hessian eigenvalue
  density.
\newblock In {\em International Conference on Machine Learning}, pages
  2232--2241. PMLR.

\bibitem[Hoffer et~al., 2018a]{hoffer2018norm}
Hoffer, E., Banner, R., Golan, I., and Soudry, D. (2018a).
\newblock Norm matters: efficient and accurate normalization schemes in deep
  networks.
\newblock {\em Advances in Neural Information Processing Systems}, 31.

\bibitem[Hoffer et~al., 2018b]{hoffer2018fix}
Hoffer, E., Hubara, I., and Soudry, D. (2018b).
\newblock Fix your classifier: the marginal value of training the last weight
  layer.
\newblock In {\em International Conference on Learning Representations}.

\bibitem[Ioffe and Szegedy, 2015]{ioffe2015batch}
Ioffe, S. and Szegedy, C. (2015).
\newblock Batch normalization: Accelerating deep network training by reducing
  internal covariate shift.
\newblock In {\em International conference on machine learning}, pages
  448--456. PMLR.

\bibitem[Kingma and Ba, 2015]{adam}
Kingma, D.~P. and Ba, J. (2015).
\newblock Adam: {A} method for stochastic optimization.
\newblock In {\em International Conference on Learning Representations}.

\bibitem[Kostenetskiy et~al., 2021]{hpc}
Kostenetskiy, P.~S., Chulkevich, R.~A., and Kozyrev, V.~I. (2021).
\newblock {HPC} resources of the higher school of economics.
\newblock {\em Journal of Physics: Conference Series}, 1740:012050.

\bibitem[Krizhevsky et~al., a]{cifar10}
Krizhevsky, A., Nair, V., and Hinton, G.
\newblock {CIFAR-10} (canadian institute for advanced research).

\bibitem[Krizhevsky et~al., b]{cifar100}
Krizhevsky, A., Nair, V., and Hinton, G.
\newblock {CIFAR-100} (canadian institute for advanced research).

\bibitem[Li et~al., 2020a]{li2020understanding}
Li, X., Chen, S., and Yang, J. (2020a).
\newblock Understanding the disharmony between weight normalization family and
  weight decay.
\newblock In {\em Proceedings of the AAAI Conference on Artificial
  Intelligence}, volume~34, pages 4715--4722.

\bibitem[Li and Arora, 2020]{li2020exponential}
Li, Z. and Arora, S. (2020).
\newblock An exponential learning rate schedule for deep learning.
\newblock In {\em International Conference on Learning Representations}.

\bibitem[Li et~al., 2020b]{li2020reconciling}
Li, Z., Lyu, K., and Arora, S. (2020b).
\newblock Reconciling modern deep learning with traditional optimization
  analyses: The intrinsic learning rate.
\newblock {\em Advances in Neural Information Processing Systems}, 33.

\bibitem[Roburin et~al., 2020]{roburin2020spherical}
Roburin, S., de~Mont-Marin, Y., Bursuc, A., Marlet, R., P{\'e}rez, P., and
  Aubry, M. (2020).
\newblock A spherical analysis of adam with batch normalization.
\newblock {\em arXiv preprint arXiv:2006.13382}.

\bibitem[Salimans and Kingma, 2016]{salimans2016weight}
Salimans, T. and Kingma, D.~P. (2016).
\newblock Weight normalization: a simple reparameterization to accelerate
  training of deep neural networks.
\newblock {\em Advances in Neural Information Processing Systems}, 29.

\bibitem[Santurkar et~al., 2018]{santurkar2018how}
Santurkar, S., Tsipras, D., Ilyas, A., and Madry, A. (2018).
\newblock How does batch normalization help optimization?
\newblock {\em Advances in Neural Information Processing Systems}, 31.

\bibitem[Sutskever et~al., 2013]{sutskever2013importance}
Sutskever, I., Martens, J., Dahl, G., and Hinton, G. (2013).
\newblock On the importance of initialization and momentum in deep learning.
\newblock In {\em International conference on machine learning}, pages
  1139--1147. PMLR.

\bibitem[Ulyanov et~al., 2016]{ulyanov2016instance}
Ulyanov, D., Vedaldi, A., and Lempitsky, V. (2016).
\newblock Instance normalization: The missing ingredient for fast stylization.
\newblock {\em arXiv preprint arXiv:1607.08022}.

\bibitem[Van~Laarhoven, 2017]{van2017l2}
Van~Laarhoven, T. (2017).
\newblock L2 regularization versus batch and weight normalization.
\newblock {\em arXiv preprint arXiv:1706.05350}.

\bibitem[Wan et~al., 2020]{wan2020spherical}
Wan, R., Zhu, Z., Zhang, X., and Sun, J. (2020).
\newblock Spherical motion dynamics: Learning dynamics of neural network with
  normalization, weight decay, and sgd.
\newblock {\em arXiv preprint arXiv:2006.08419}.

\bibitem[Wu and He, 2018]{wu2018group}
Wu, Y. and He, K. (2018).
\newblock Group normalization.
\newblock In {\em Proceedings of the European conference on computer vision
  (ECCV)}, pages 3--19.

\bibitem[Yang et~al., 2019]{yang2018mean}
Yang, G., Pennington, J., Rao, V., Sohl-Dickstein, J., and Schoenholz, S.~S.
  (2019).
\newblock A mean field theory of batch normalization.
\newblock In {\em International Conference on Learning Representations}.

\bibitem[Zhang et~al., 2019]{zhang2018three}
Zhang, G., Wang, C., Xu, B., and Grosse, R. (2019).
\newblock Three mechanisms of weight decay regularization.
\newblock In {\em International Conference on Learning Representations}.

\end{thebibliography}

\newpage
\appendix

\section{Theoretical results}
\label{app:theory}

This section contains details on our theoretical results.

\subsection{Invariance to hyperparameters rescaling}
\label{app:hyp_resc_inv}

Based on properties~\eqref{eq:si_prop_perp} and~\eqref{eq:si_prop_hom}, we derive a simple yet useful proposition tying together different hyperparameter settings of initialization $x_0$, learning rate $\eta$, and weight decay coefficient $\lambda$. This proposition provides grounds for fixing the initialization scale in our experiments and iterating over learning rates and weight decay coefficients when studying the dependence of the behavior of scale-invariant neural networks on hyperparameters. 

\begin{proposition}
\label{prop:hyp_resc_inv}
Given $f(x)$ is scale-invariant and optimized according to Eq.~\eqref{eq:si_dyn}, settings $(x_0, \eta, \lambda)$ and $(x_0^\prime, \eta^\prime, \lambda^\prime) = (c x_0, c^2 \eta, \lambda / c^2),\, c > 0$ lead to equivalent dynamics in function space.
\end{proposition}
\begin{proof}
Eq.~\eqref{eq:si_dyn} and property~\eqref{eq:si_prop_hom} give $x_{t+1} = \norm{x_t} \left[ (1 - \eta \lambda)\frac{x_t}{\norm{x_t}} - \tilde{\eta}_t \nabla f(x_t / \norm{x_t}) \right]$, where $\tilde{\eta}_t = \frac{\eta}{\norm{x_t}^2}$ is the effective learning rate.
Since the term in square brackets does not depend on the scale of $x_t$ provided that the effective learning rate and $\eta \lambda$ product are unchanged, by induction, from $x_t^\prime = c x_t$ we have $x_{t+1}^\prime = c x_{t+1}$, hence $f(x_{t+1}^\prime) = f(x_{t+1})$.
\end{proof}

\subsection{Derivations}
\label{app:derivations}

\paragraph{Parameters norm dynamics~\eqref{eq:si_pnorm_dyn}}
\begin{align*}
    \rho_{t+1}^2 &= 
    \scalarprod{x_{t+1}}{x_{t+1}} = 
    \left\{ \text{Eq.~\eqref{eq:si_dyn}} \right\} = 
    (1 - \eta \lambda)^2 \rho_t^2 + \eta^2 g_t^2 + 2 \eta (1 - \eta \lambda) \scalarprod{\nabla f(x_t)}{x_t} = \\
    &= \left\{ \text{property~\eqref{eq:si_prop_perp}} \right\} =
    (1 - \eta \lambda)^2 \rho_t^2 + \eta^2 g_t^2 = 
    \left\{ \text{property~\eqref{eq:si_prop_hom}, i.e., $g_t = \tilde{g}_t / \rho_t$} \right\} = \\
    &= (1 - \eta \lambda)^2 \rho_t^2 + \eta^2 \tilde{g}_t^2 / \rho_t^2
\end{align*}

\paragraph{Cosine distance between adjacent iterates~\eqref{eq:cos}}
\begin{align*}
    \cos(x_t, x_{t+1}) &= 
    \frac{\scalarprod{x_t}{x_{t+1}}}{\rho_t \rho_{t+1}} =
    \left\{ \text{Eq.~\eqref{eq:si_dyn}} \right\} = 
    \frac{(1 - \eta \lambda) \scalarprod{x_t}{x_t} - \eta \scalarprod{\nabla f(x_t)}{x_t}}{\rho_t \rho_{t+1}} = \\
    &= \left\{ \text{property~\eqref{eq:si_prop_perp}} \right\} =
    \frac{(1 - \eta \lambda) \rho_t}{\rho_{t+1}} =
    \left\{ \text{Eq.~\eqref{eq:si_pnorm_dyn}} \right\} =
    \left( 1 + \frac{\eta^2 \tilde{g}_t^2}{(1 - \eta \lambda)^2 \rho_t^4} \right)^{-1/2}
\end{align*}

\paragraph{$\delta$-jump conditions~\eqref{eq:cos_dist_jump}}
\begin{align*}
    1 - \cos(x_t, x_{t+1}) > \delta &\iff
    \left\{ \text{Eq.~\eqref{eq:cos}} \right\} \iff
    \left( 1 + \frac{\eta^2 \tilde{g}_t^2}{(1 - \eta \lambda)^2 \rho_t^4} \right)^{-1/2} < 1 - \delta \iff \\
    &\iff 1 + \frac{\eta^2 \tilde{g}_t^2}{(1 - \eta \lambda)^2 \rho_t^4} > \frac{1}{(1 - \delta)^2} = 1 + 2 \delta + \mathcal{O}(\delta^2) \gtrapprox 1 + 2 \delta.
\end{align*}
Omitting $\mathcal{O}(\delta^2)$ leaves the condition necessary and also approximately sufficient for small $\delta$:
\begin{align*}
    1 - \cos(x_t, x_{t+1}) > \delta &\implies \frac{\eta^2 \tilde{g}_t^2}{(1 - \eta \lambda)^2 \rho_t^4} > 2 \delta \iff
    \rho_t^2 < \frac{\eta \tilde{g}_t}{(1 - \eta \lambda) \sqrt{2 \delta}}.
\end{align*}

\subsection{Proof of Proposition~\ref{prop:jump_cond}}
\label{app:jump_cond}

For the convenience of reading, we defer the derivation details of all equations to Appendix~\ref{app:derivations}.

\begin{proof}
Using property~\eqref{eq:si_prop_perp} and Eq.~\eqref{eq:si_pnorm_dyn}, we obtain the exact value of the cosine between adjacent iterates:
\begin{equation}
    \label{eq:cos}
    \cos(x_t, x_{t+1}) = \left( 1 + \frac{\eta^2 \tilde{g}_t^2}{(1 - \eta \lambda)^2 \rho_t^4} \right)^{-1/2}.
\end{equation}
From Eq.~\eqref{eq:cos} we deduce the following $\delta$-jump condition:
\begin{equation}
    \label{eq:cos_dist_jump}
    1 - \cos(x_t, x_{t+1}) > \delta \implies \rho_t^2 < \frac{\eta \tilde{g}_t}{(1 - \eta \lambda) \sqrt{2 \delta}}.
\end{equation}
During the derivation, we omitted $\mathcal{O}(\delta^2)$ terms.
This implies that the right inequality represents not only the necessary but also (approximately) the sufficient condition for a $\delta$-jump when $\delta$ is small.

Assuming $(1 - \eta \lambda) \lessapprox 1$ and substituting the effective gradient bounds $\ell$ and $L$ into Eq.~\eqref{eq:cos_dist_jump} in place of $\tilde{g}_t$ finally yields the approximate necessary and sufficient $\delta$-jump conditions~\eqref{eq:jump_nec} and~\eqref{eq:jump_suf}, respectively.
\end{proof}

\subsection{On $\beta$-undetermined recurrent sequences}
\label{app:beta_undet_sequence}

Here we provide some results related to convergence of sequences of the following kind:
\begin{equation}
    \label{eq:beta_undet_sequence}
    x_{t+1} = (1 - \alpha) x_t + \frac{\beta_t}{x_t},
\end{equation}
where $\alpha$ is a fixed coefficient, and $\beta_t$ may vary from iteration to iteration.
We assume $x_0 > 0$, $0 < \alpha < 0.5$, and $\beta_t \in [a, b],\, \forall t$, where $0 \le a \le b < +\infty$ are some fixed values.
We call sequences of type~\eqref{eq:beta_undet_sequence} \emph{$\beta$-undetermined} recurrent sequences.

\subsubsection{$\beta$-determined sequences}
\label{app:beta_det_sequence}

To derive the basic properties of $\beta$-undetermined sequences~\eqref{eq:beta_undet_sequence}, we first consider \emph{$\beta$-determined} recurrent sequences:
\begin{equation}
    \label{eq:beta_det_sequence}
    x_{t+1} = (1 - \alpha) x_t + \frac{\beta}{x_t},
\end{equation}
where $\beta$ is now a fixed non-negative value. 

If $\beta = 0$,~\eqref{eq:beta_det_sequence} boils down to a classical linear sequence converging to zero at rate $1 - \alpha$.
Assume now that $\beta > 0$.
First of all, $x^* = \sqrt{\frac{\beta}{\alpha}}$ is the only stationary point of sequence~\eqref{eq:beta_det_sequence}.
This holds from solving the following equation:
\[
    x_{t+1} = x_t \iff x_t = x^* = \sqrt{\frac{\beta}{\alpha}}.
\]
Suppose $x_t = \gamma_t x^*$. 
By dividing the left and right sides of Eq.~\eqref{eq:beta_det_sequence} by $x^*$, we can derive the formula for $\gamma_{t+1}$ as a function of $\gamma_t$ which we denote as $\varphi(\gamma_t)$:
\begin{equation}
    \label{eq:gamma_tp1}
    \gamma_{t+1} = \varphi(\gamma_t) = (1 - \alpha) \gamma_t + \frac{\alpha}{\gamma_t}.
\end{equation}
The sequence induced by~\eqref{eq:gamma_tp1} is a special case of Eq.~\eqref{eq:beta_det_sequence} with a stationary point $\gamma^* = 1$.
One important property is that $\gamma_{t+1}$ does not depend on $\beta$ explicitly, only on $\gamma_t$ and $\alpha$.
This unifies the convergence analysis for sequences with different $\beta$ coefficients.

For function~\eqref{eq:gamma_tp1} the following facts hold (see Figure~\ref{fig:gamma_tp1} for an illustration):
\begin{subnumcases}{}
    \gamma_t < 1 \implies \gamma_{t+1} > \gamma_t \text{: the sequence is increasing once it's below $x^*$}, \label{eq:gamma_inc} \\
    \gamma_t > 1 \implies 1 < \gamma_{t+1} < \gamma_t \text{: the sequence is decreasing once it's above $x^*$}, \label{eq:gamma_dec} \\
    \gamma_{t+1} = 1 \iff \gamma_t = \frac{\alpha}{1 - \alpha} \vee \gamma_t = 1 \text{: pre-stationary conditions}, \label{eq:gamma_prestat} \\
    \gamma_{t+1} < 1 \iff \gamma_t \in \left(\frac{\alpha}{1 - \alpha}, 1\right) \text{: conditions for staying below the stationary point}, \label{eq:gamma_below} \\
    \varphi(\gamma_t) \text{ is a decreasing function for } \gamma_t < \sqrt{\frac{\alpha}{1 - \alpha}}, \label{eq:gamma_dec_func} \\
    \varphi(\gamma_t) \text{ is an increasing function for } \gamma_t > \sqrt{\frac{\alpha}{1 - \alpha}}, \label{eq:gamma_inc_func} \\
    \gamma_{t+1} = \min_{\gamma_t} \varphi(\gamma_t) = 2 \sqrt{\alpha (1 - \alpha)} \iff \gamma_t = \sqrt{\frac{\alpha}{1 - \alpha}} \text{: the lowest achievable value}. \label{eq:gamma_min}
\end{subnumcases}

\begin{figure}
  \centering
  \includegraphics[width=0.75\textwidth]{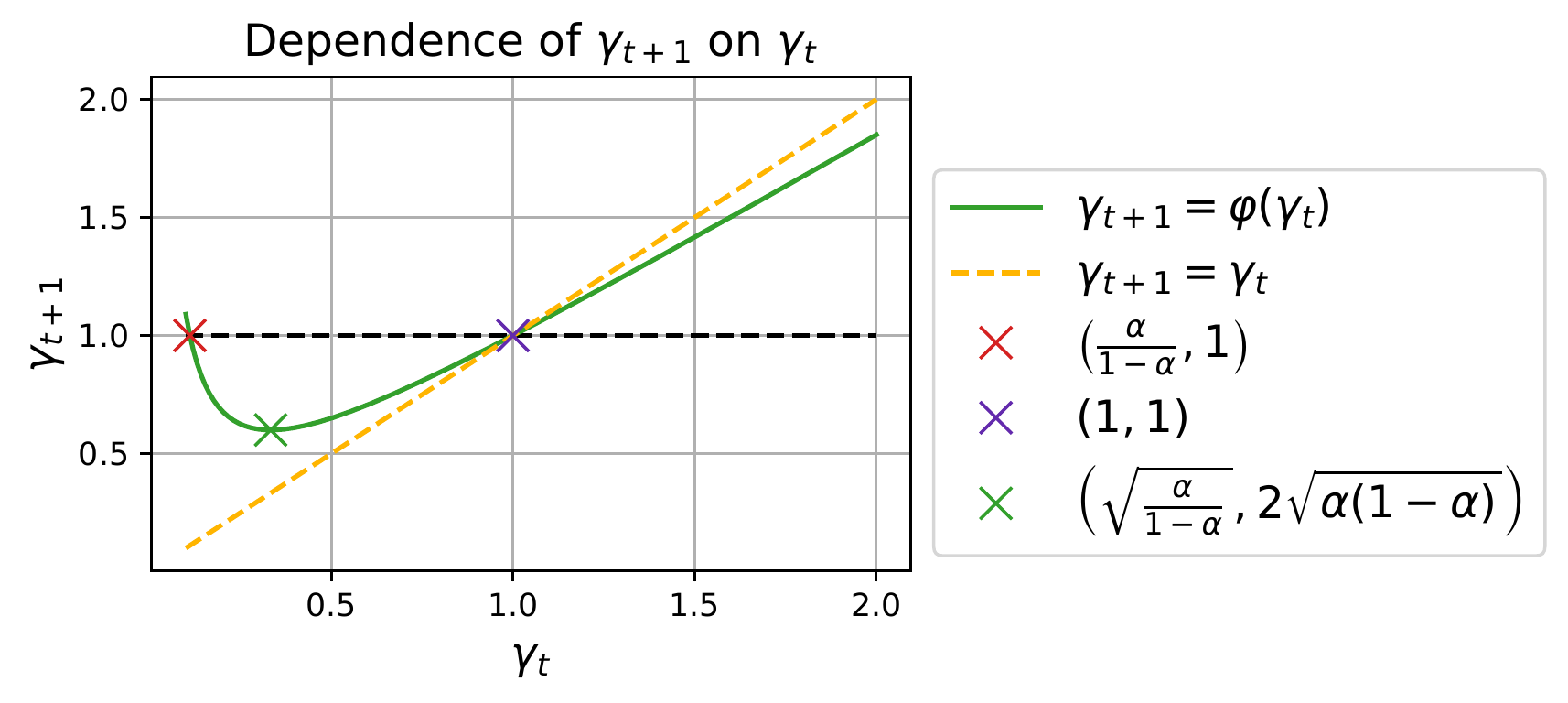}
  \caption{Dependence of $\gamma_{t+1}$ on $\gamma_t$ from Eq.~\eqref{eq:gamma_tp1} for $\alpha = 0.1$.}
  \label{fig:gamma_tp1}
\end{figure}

Note that for $0 < \alpha < 0.5$ we have
\[
    \frac{\alpha}{1 - \alpha} < \sqrt{\frac{\alpha}{1 - \alpha}} < 2 \sqrt{\alpha (1 - \alpha)} < 1.
\]
Properties~\eqref{eq:gamma_dec} and~\eqref{eq:gamma_below} imply that $x_{t+1}$ can ``hop'' over $x^*$ if only $x_t < \frac{\alpha}{1 - \alpha} x^*$.
Otherwise, $x_t$ is monotonically approaching its stationary point.
That is an important threshold that will help derive the convergence of $\beta$-undetermined sequences to a specific equilibrium interval.

The derivative of $\varphi(\gamma_t)$ can help estimate the convergence rate of the sequence~\eqref{eq:beta_det_sequence} to its stationary point.
Specifically, using the mean value theorem, we obtain that
\begin{equation}
    \label{eq:mvt}
    x_{t+1} - x^* = x^* \left(\gamma_{t+1} - 1\right) = x^* \left(\varphi(\gamma_t) - \varphi(1)\right) = x^* \varphi^{\prime}(\xi) \left(\gamma_t - 1\right),
\end{equation}
where $\xi$ is some point between $1$ and $\gamma_t$.
Therefore, by bounding the derivative $\varphi^{\prime}(\gamma_t)$, we can also bound the $x_t$ convergence to $x^*$.

Suppose that $\gamma_0 > 1$.
From~\eqref{eq:gamma_dec} it follows that $\gamma_t > 1,\, \forall t$.
In this case, we can bound the derivative of $\varphi(\gamma_t)$ for $\gamma_t > 1$ and obtain the approximate convergence rates for~\eqref{eq:beta_det_sequence}:
\[
    1 - 2\alpha < \varphi^{\prime}(\gamma_t) = (1 - \alpha) - \frac{\alpha}{\gamma_t^2} < 1 - \alpha,\, \gamma_t > 1,
\]
which, after recursively applying~\eqref{eq:mvt}, yields
\[
    (1 - 2 \alpha)^t (\gamma_0 - 1) < \gamma_{t} - 1 < (1 - \alpha)^t (\gamma_0 - 1),\, t \ge 1,
\]
or equivalently, formulating this for~\eqref{eq:beta_det_sequence} as a lemma:

\begin{lemma}
\label{lem:beta_det_bounds}
    For an arbitrary $\beta$-determined sequence~\eqref{eq:beta_det_sequence} with $\beta \ge 0$, given $x^* = \sqrt{\frac{\beta}{\alpha}}$ and $x_0 > x^*$, the following bounds on its convergence rate hold:
    \[
        (1 - 2 \alpha)^t (x_0 - x^*) \le x_t - x^* \le (1 - \alpha)^t (x_0 - x^*),\, \forall t.
    \]
\end{lemma}

This is the main result concerning the convergence of $\beta$-determined sequences~\eqref{eq:beta_det_sequence}.
Note that Lemma~\ref{lem:beta_det_bounds} also covers the case of $\beta = 0$ because then $x^* = 0$ and $x_t = (1 - \alpha)^t x_0$.

\subsubsection{$\beta$-undetermined sequences convergence bounds}

Now, we can return back to the $\beta$-undetermined sequences~\eqref{eq:beta_undet_sequence} and derive its convergence bounds.
The following lemma allows to bound an arbitrary $\beta$-undetermined sequence with $\beta$-determined ones.

\begin{lemma}
\label{lem:beta_undet_bounds}
    For an arbitrary $\beta$-undetermined sequence of type~\eqref{eq:beta_undet_sequence} with $0 \le a \le \beta_t \le b < +\infty$ the following $\beta$-determined bounds hold.
    \begin{enumerate}
        \item Let $x_{a,t}$ be a $\beta$-determined sequence~\eqref{eq:beta_det_sequence} with $\beta = a$ and $x_{a,0} = x_0$. 
        Then $x_{a,t} \le x_t,\, \forall t$.
        \item Let $x_{b,t}$ be a $\beta$-determined sequence~\eqref{eq:beta_det_sequence} with $\beta = b$ and $x_{b,0} = x_0$. 
        Then, if $x_t > \sqrt{\frac{b}{1 - \alpha}},\, t = 0, \dots, T$, we have $x_t \le x_{b,t},\, t = 0, \dots, T + 1$.
    \end{enumerate}
\end{lemma}
\begin{proof}
    We will prove the first statement since the second one can be proved similarly. 
    
    Let $\sqrt{\frac{a}{1 - \alpha}} < x_{a, t} \le x_t$.
    Then the following inequalities hold: 
    \[
    x_{a, t + 1} \le (1 - \alpha) x_t + \frac{a}{x_t} \le x_{t + 1}.
    \]
    The first inequality holds since $x_{a, t + 1}$ is a monotonically increasing function of $x_{a, t}$ due to~\eqref{eq:gamma_inc_func}.
    The second one is valid because $a \le \beta_t$.
    
    Note that due to~\eqref{eq:gamma_min} and $\sqrt{\frac{\alpha}{1 - \alpha}} < 2 \sqrt{\alpha (1 - \alpha)}$, we have $\sqrt{\frac{a}{1 - \alpha}} < x_{a, t},\, t \ge 1$, plus, as $a \le \beta_0$, $x_{a, 1} \le x_1$, hence, induction is valid for all $t$ for the lower bound (in contrast with the upper bound case, where we explicitly demand $x_t > \sqrt{\frac{b}{1 - \alpha}}$ for $T$ consecutive timesteps).
\end{proof}

\begin{remark}
    \label{rem:beta_undet_upper_bound}
    An important special case when the upper bound $x_{b, t}$ is valid for all $t$ is if $\frac{\alpha}{1 - \alpha} \sqrt{b} \le \sqrt{a}$ and $x_0 > \sqrt{\frac{b}{\alpha}}$.
    Then, while $x_t \ge \sqrt{\frac{b}{\alpha}} > \sqrt{\frac{b}{1 - \alpha}}$ the bound is valid due to the second statement of the lemma.
    As soon as $x_t$ crosses the $\sqrt{\frac{b}{\alpha}}$ threshold, it can never ``hop'' over it again due to~\eqref{eq:gamma_below} and $x_t \ge x_{a, t} > \sqrt{\frac{a}{\alpha}} \ge \frac{\sqrt{\alpha b}}{1 - \alpha},\, \forall t$; at the same time, $x_{b, t} > \sqrt{\frac{b}{\alpha}},\, \forall t$ due to~\eqref{eq:gamma_dec}.
\end{remark}

Based on the convergence results of $\beta$-determined sequences, the following corollary allows estimating the convergence rates of $\beta$-undetermined sequences.

\begin{corollary}
    \label{cor:beta_undet_cr_bounds}
    Given Lemma~\ref{lem:beta_det_bounds}, Lemma~\ref{lem:beta_undet_bounds}, and the reasoning from Remark~\ref{rem:beta_undet_upper_bound}, we obtain the following bounds on convergence rates of an arbitrary $\beta$-undetermined sequence~\eqref{eq:beta_undet_sequence}:
    \begin{enumerate}
        \item if $x_0 > \sqrt{\frac{a}{\alpha}}$, then $(1 - 2 \alpha)^t \left(x_0 - \sqrt{\frac{a}{\alpha}}\right) \le x_t - \sqrt{\frac{a}{\alpha}},\, \forall t$;
        \item if $x_0 > \sqrt{\frac{b}{\alpha}}$, then $x_t - \sqrt{\frac{b}{\alpha}} \le (1 - \alpha)^t \left(x_0 - \sqrt{\frac{b}{\alpha}}\right)$ while $x_t \ge \frac{\sqrt{\alpha b}}{1 - \alpha}$.
    \end{enumerate}
\end{corollary}

Our final important result about the $\beta$-undetermined sequences convergence is a case of convergence to the interval determined by the stationary points of the bounding $\beta$-determined sequences $x_{a, t}$ and $x_{b, t}$.
We formulate it in the following proposition (see Figure~\ref{fig:beta_undet_interval} for an illustration).

\begin{proposition}
    \label{prop:beta_undet_interval}
    An arbitrary $\beta$-undetermined sequence~\eqref{eq:beta_undet_sequence}, given $\frac{\alpha}{1 - \alpha} \sqrt{b} \le \sqrt{a}$, converges to the following interval:
    \[
        \sqrt{\frac{a}{\alpha}} \le x_t \le \sqrt{\frac{b}{\alpha}},\, t \gg 1.\footnote{These bounds are, in general, asymptotic, so, for complete correctness, $t \gg 1$ must be substituted with $t \to \infty$; however, excluding the degenerate cases, we can often observe that $x_t$ reaches the interval in finite time.}
    \]
    Furthermore, if $x_0 > \sqrt{\frac{b}{\alpha}}$, then $x_t$ converges to the interval linearly in $\mathcal{O}(1 / \alpha)$ time.
\end{proposition}
\begin{proof}
    Due to the first statement of Lemma~\ref{lem:beta_undet_bounds}, $x_t \ge x_{a, t} \to \sqrt{\frac{a}{\alpha}}$, hence, we deduce that the lower bound will eventually hold for $t \to \infty$.
    Since $\frac{\alpha}{1 - \alpha} \sqrt{b} \le \sqrt{a}$ and due to the reasoning in Remark~\ref{rem:beta_undet_upper_bound}, when $\sqrt{\frac{a}{\alpha}} \le x_t$ is fulfilled, the series either stays in the stated interval (if $x_t \le \sqrt{\frac{b}{\alpha}}$) and never crosses it or approaches it from above thanks to the upper $\beta$-deterministic bounding sequence $x_t \le x_{b, t} \to \sqrt{\frac{b}{\alpha}}$, so the upper bound is also (asymptotically) valid.
    
    If $x_0 > \sqrt{\frac{b}{\alpha}}$, Corollary~\ref{cor:beta_undet_cr_bounds} allows us to enclose $x_t$ (while it is above $\sqrt{\frac{b}{\alpha}}$) between two linear sequences converging to $\sqrt{\frac{a}{\alpha}}$ and $\sqrt{\frac{b}{\alpha}}$, respectively, with one-minus-rate proportional to $\alpha$.
    This is consistent with the convergence time $\mathcal{O}(1 / \alpha)$ since the convergence time of linear sequences is inversely proportional to the one-minus-rate value.
\end{proof}

\begin{figure}
  \centering
  \includegraphics[width=0.8\textwidth]{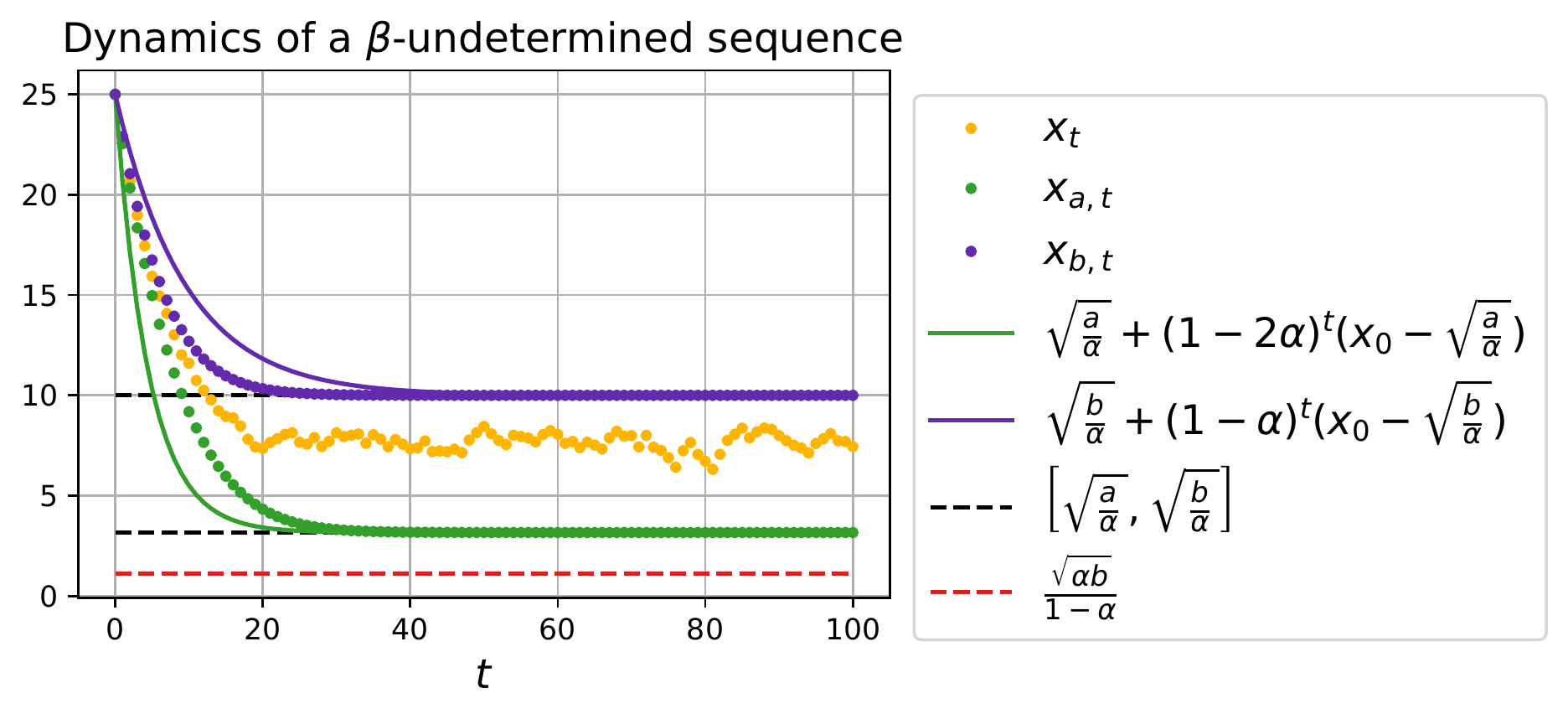}
  \caption{$\beta$-undetermined sequence~\eqref{eq:beta_undet_sequence} convergence to the $\left[ \sqrt{\frac{a}{\alpha}}, \sqrt{\frac{b}{\alpha}} \right]$ interval (Proposition~\ref{prop:beta_undet_interval}). Setting: $\alpha = 0.1$, $a = 1$, $b = 10$, $\beta_t \sim \mathcal{U}(a, b)$.}
  \label{fig:beta_undet_interval}
\end{figure}

\subsection{Proof of Proposition~\ref{prop:jump_time}}
\label{app:jump_time}

We prove Proposition~\ref{prop:jump_time} using the general convergence theory for so-called $\beta$-undetermined recurrent sequences of type $x_{t+1} = (1 - \alpha) x_t + \frac{\beta_t}{x_t}$, where $0 < \alpha < 0.5$ and $0 \le a \le \beta_t \le b < +\infty,\, \forall t$ (see Appendix~\ref{app:beta_undet_sequence}).
Note that the parameters norm dynamics~\eqref{eq:si_pnorm_dyn} is a special case of such a sequence with $x_t \coloneqq \rho_t^2$, $\beta_t \coloneqq \eta^2 \tilde{g}_t^2$, $a \coloneqq \eta^2 \ell^2$, $b \coloneqq \eta^2 L^2$, and $\alpha \coloneqq 2 \eta \lambda$ (recall that we suppress $\mathcal{O}\left((\eta \lambda)^2\right)$ terms).

\begin{proof}
    Denote $\kappa = \sqrt{\frac{\eta}{2 \lambda}}$.
    
    In the notation of $\beta$-undetermined sequences, the condition $\rho_0^2 > \kappa \ell$ translates into $x_0 > \sqrt{\frac{a}{\alpha}}$.
    Thus, by applying Corollary~\ref{cor:beta_undet_cr_bounds}, we can bound the convergence of parameters norm from below with the following linear sequence:
    \[
        \kappa \ell + (1 - 4 \eta \lambda)^t \left(\rho_0^2 - \kappa \ell\right) \le \rho_t^2.
    \]
    The necessary $\delta$-jump condition~\eqref{eq:jump_nec} can be equivalently reformulated as an upper bound on $\delta$:
    \[
        \kappa \ell < \frac{\eta L}{\sqrt{2 \delta}} \iff \delta < \eta \lambda \frac{L^2}{\ell^2}.
    \]
    If this condition is fulfilled, we can estimate the minimal time required for a $\delta$-jump~--- the moment when the lower bound on $\rho_t^2$ intersects the $\frac{\eta L}{\sqrt{2 \delta}}$ threshold.
    If $\rho_0^2 \le \frac{\eta L}{\sqrt{2 \delta}}$, obviously, $t_{\min} = 0$, else, by solving the following equation for $t$:
    \[
    \sqrt{\frac{\eta}{2 \lambda}} \ell + (1 - 4 \eta \lambda)^t \left(\rho_0^2 - \sqrt{\frac{\eta}{2 \lambda}} \ell\right) = \frac{\eta L}{\sqrt{2 \delta}},
    \]
    we obtain~\eqref{eq:t_min}.
    
    Again, $\rho_0^2 > \kappa L$ is equivalent to $x_0 > \sqrt{\frac{b}{\alpha}}$ and, due to Corollary~\ref{cor:beta_undet_cr_bounds}, the following upper bound on $\rho_t^2$ holds (at least while $\rho_t^2 \ge \kappa L$):
    \[
        \rho_t^2 \le \kappa L + (1 - 2 \eta \lambda)^t \left(\rho_0^2 - \kappa L \right).
    \]
    Now, if $\delta$ is so small that the sufficient condition for a jump~\eqref{eq:jump_suf} is fulfilled before $\rho_t^2$ converges to $\kappa L$, i.e.,
    \[\kappa L < \frac{\eta \ell}{\sqrt{2 \delta}} \iff \delta < \eta \lambda \frac{\ell^2}{L^2},\]
    we can similarly estimate the maximal required time for a $\delta$-jump~\eqref{eq:t_max} as the moment when the upper bound on $\rho_t^2$ intersects the $\frac{\eta \ell}{\sqrt{2 \delta}}$ threshold.
\end{proof}

\subsection{Proof of Proposition~\ref{prop:norm_eq}}
\label{app:norm_eq}

As in the previous section, we prove Proposition~\ref{prop:norm_eq} using the general theory on $\beta$-undetermined sequences (see Appendix~\ref{app:beta_undet_sequence}).
We remarked above that the parameters norm dynamics~\eqref{eq:si_pnorm_dyn} is a special case of such a sequence with parameters $a \coloneqq \eta^2 \ell^2$, $b \coloneqq \eta^2 L^2$, and $\alpha \coloneqq 2 \eta \lambda$.

\begin{proof}
   According to Proposition~\ref{prop:beta_undet_interval}, if for a $\beta$-undetermined sequence $x_t$ the condition $\frac{\alpha}{1 - \alpha} \sqrt{b} \le \sqrt{a}$ is fulfilled, then one can show that $x_t \in \left[\sqrt{\frac{a}{\alpha}}, \sqrt{\frac{b}{\alpha}}\right], t \gg 1$;
    furthermore, if $x_0 > \sqrt{\frac{b}{\alpha}}$, then $x_t$ converges to the interval linearly in $\mathcal{O}(1 / \alpha)$ time.
    For the parameters norm dynamics, the condition $\frac{\alpha}{1 - \alpha} \sqrt{b} \le \sqrt{a}$ is equivalent (up to $\mathcal{O}\left((\eta \lambda)^2\right)$ terms) to $2 \eta \lambda L \le \ell$ as $\frac{\alpha}{1 - \alpha} = \frac{2 \eta \lambda}{1 - 2 \eta \lambda} = 2 \eta \lambda + \mathcal{O}\left((\eta \lambda)^2\right)$.
    Now, if it holds, we can apply Proposition~\ref{prop:beta_undet_interval} and conclude the proof.
\end{proof}

\begin{remark}
    \label{rem:elr_eq}
    We can reformulate the same result in terms of the effective learning rate $\tilde{\eta}_t = \eta / \rho_t^2$:
    \[
    2 \eta \lambda L \le \ell \le \tilde{g}_t \le L \implies \frac{\sqrt{2\eta \lambda}}{L} \le \tilde{\eta}_t \le \frac{\sqrt{2\eta \lambda}}{\ell},\,t \gg 1.
    \]
\end{remark}

\subsection{Discussion on $2 \eta \lambda L \le \ell$ condition}
\label{app:ell_cond}
In this section, we discuss the assumption $2 \eta \lambda L \le \ell$ made in Proposition~\ref{prop:norm_eq}, implying that the lower and the upper effective gradient norm bounds must not differ too much.
First of all, we would like to remark that this condition is generally fulfilled in practice for small $\eta \lambda$ product even when the bounds $\ell$ and $L$ are taken globally, i.e., they satisfy $\ell \le \tilde{g}_t \le L,\, \forall t$.
We also note that~\citet{wan2020spherical} made a very close assumption in their main Theorem~1 (Assumption~3).
However, even if it is not fulfilled, our generalized parameters norm equilibrium result is still valid to some extent.

First, consider the case when $0 < \ell < 2 \eta \lambda L$.
Then, according to the general $\beta$-undetermined sequences theory presented in Appendix~\ref{app:beta_undet_sequence}, the lower bound $\kappa \ell \le \rho_t^2$ remains valid for large $t$.
If $\rho_t^2$ falls below $2 \eta \lambda L$, it can potentially ``hop'' over the upper bound of the interval $\kappa L$.
However, due to $\tilde{g}_t \le L$ and property~\eqref{eq:gamma_dec_func} of $\beta$-determined sequences (see Appendix~\ref{app:beta_det_sequence}) $\rho_t^2$ is still upper bounded by the value $(1 - \eta \lambda)^2 \kappa \ell + \frac{\eta^2 L^2}{\kappa \ell}$. 
Hence, globally, the parameters norm stays bounded even when $2 \eta \lambda L \le \ell$ does not hold.
Furthermore, according to the second statement of Corollary~\ref{cor:beta_undet_cr_bounds}, once $\rho_t^2$ exceeds the $\kappa L$ value, it immediately starts converging to it again.
So the same $\left[\kappa \ell, \kappa L \right]$ interval of attraction is still preserved. 

Now, we argue that setting $\ell = 0$, i.e., bounding the effective gradient norm from below with zero, is vacuous.\footnote{Excluding, perhaps, some exceptional degenerate cases when the function and hyperparameters are chosen so that the dynamics converge to a stationary point in a finite number of steps.}
Again, we remark that the assumption about separating $\ell$ from zero was made, e.g., by~\citet{wan2020spherical}.
\citet{arora2018theoretical} show that effective gradients (in case of learning without WD) decay sublinearly, which by itself means that in finite time horizon, it is always reasonable to set $\ell > 0$.
Moreover, as we show, parameters norm evolves linearly, i.e., faster than the effective gradients; therefore, it must quickly acclimate to local $\ell$, $L$ changes and hence respect the boundaries from Proposition~\ref{prop:norm_eq}.
But even based on general results on gradient-based optimization, we anticipate that, in general, $\ell$ should not approach zero.
We can rewrite the expression for $\rho_t^2$~\eqref{eq:si_pnorm_dyn} in the following way: 
\begin{align}
    \rho_{t}^2 &=
    (1 - \eta \lambda)^2 \rho_{t-1}^2 + \eta^2 g_{t-1}^2 = 
    \dots = 
    (1 - \eta \lambda)^{2t} \rho_{0}^2 + \eta^2 \sum_{t' = 0}^{t-1} (1 - \eta \lambda)^{2(t-t'-1)} g_{t'}^2 = \\
    &= (1 - \eta \lambda)^{2t} \rho_{0}^2 + \eta^2 \frac{1 - (1 - \eta \lambda)^2}{1 - (1 - \eta \lambda)^{2t}} \bar{g}_t^2 \approx \{ t \gg 1 \} \approx (1 - \eta \lambda)^{2t} \rho_{0}^2 + C \bar{g}_t^2, \label{eq:pnorm_approx}
\end{align}
where $\bar{g}_t$ is an exponential moving average of the gradient norm and $C = 2 \eta^3 \lambda + \mathcal{O}\left((\eta \lambda)^2\right)$ is constant. 
It is well-known that for first-ordered methods, the lower gradient norm bound generally decays sublinearly~\cite{carmon2017lower}.
Note that the cosine between adjacent iterates~\eqref{eq:cos} depends only on the $g_t^2 / \rho_t^2$ ratio.
For large $t$, this ratio, due to~\eqref{eq:pnorm_approx}, is determined only by the $g_t^2 / \bar{g}_t^2$ ratio since the first term decays linearly, i.e., faster than $g_t^2$.
It is reasonable to conjecture that $g_t$ oscillates around its mean value $\bar{g}_t$ hence hindering stabilization of the training dynamics which, in turn, implies that the effective gradient does not vanish.
Thus, implying $\ell > 0$ seems to be a reasonable assumption.

\section{Experimental details}
\label{app:details}

{\bf Datasets and architectures.} 
We conduct experiments with two convolutional architectures, namely a three-layer convolutional neural network (ConvNet) and ResNet-18, on CIFAR-10~\cite{cifar10} and CIFAR-100~\cite{cifar100} datasets. We use the implementation of both architectures available at~\url{https://github.com/g-benton/hessian-eff-dim}. CIFAR datasets are distributed under the MIT license, and the code is under Apache-2.0 License. To make the majority of neural network weights scale-invariant, we insert additional BN layers according to Appendix C of~\citet{li2020exponential}.  We use the standard PyTorch initialization for all layers.
We use ResNet of standard width. For ConvNet, we use the width factor of 32 for fully scale-invariant networks on CIFAR-10 and the width factor of 64 for all experiments on CIFAR-100 and experiments with practical modifications on CIFAR-10.

{\bf Fully scale-invariant setup.}
Most of the experiments are conducted with the scale-invariant modifications of both architectures obtained using the approach of~\citet{li2020exponential}. In addition to inserting extra BN layers, we fix all non-scale-invariant weights, i.e., BN parameters and the last layer's parameters. For BN layers, we use zero mean and unit variance. We fix the bias vector at random initialization and the weight matrix at rescaled random initialization for the last layer. 
In most of the experiments, we rescale the last layer's weight matrix so that its norm equals 10, but we discuss other scales in Appendix~\ref{app:last_layer}.

{\bf Training.}
We train all networks using SGD with a batch size of 128 and various weight decays and learning rates. 
In the experiments with momentum, we use the momentum of 0.9. In the experiments with data augmentation, we use standard CIFAR augmentations: random crop (size: 32, padding: 4) and random horizontal flip. All models were trained on NVidia Tesla V100 or NVidia GeForce GTX 1080. Obtaining the results reported in the paper took approximately 1K GPU hours.

{\bf Full-batch GD experiments.} 
Full-batch GD training experiments are conducted on the 4.5K-sized random subset of the train dataset. The test set in this experiment consists of 5K randomly chosen test objects. 

{\bf Logging.} In all experiments except Figures~\ref{fig:one_cycle}, \ref{fig:one_cycle_gd}, and~\ref{fig:cos_bounds} we log all metrics after each epoch, computing train loss and its gradients by making an additional pass through the training dataset. 
We log all metrics after each (S)GD step in three specified figures, computing train loss and its gradients over a batch.

\section{Full-batch gradient descent}
\label{app:gd}

\begin{figure}
  \centering
  \includegraphics[width=\textwidth]{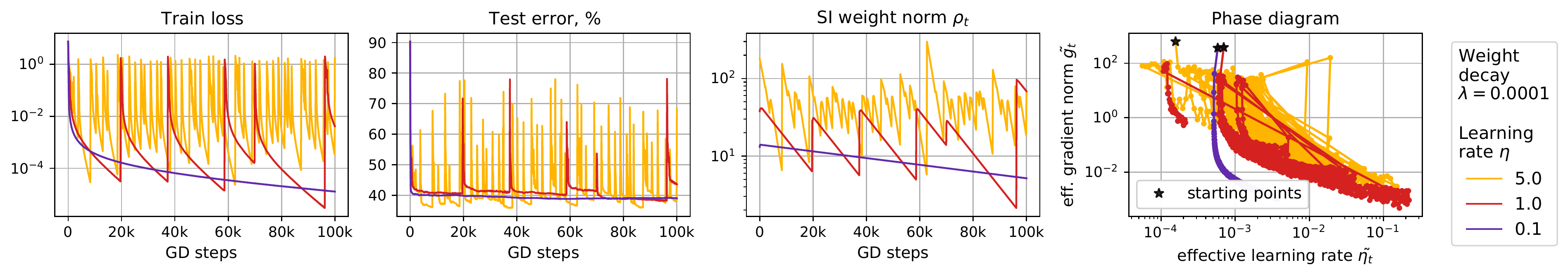}
  \caption{
  Periodic behavior of scale-invariant ConvNet on CIFAR-10 trained using full-batch GD with the weight decay of 0.0001 and different learning rates.
  }
  \label{fig:cycles_demo_gd}
\end{figure}

\begin{figure}
  \centering
  \includegraphics[width=0.8\textwidth]{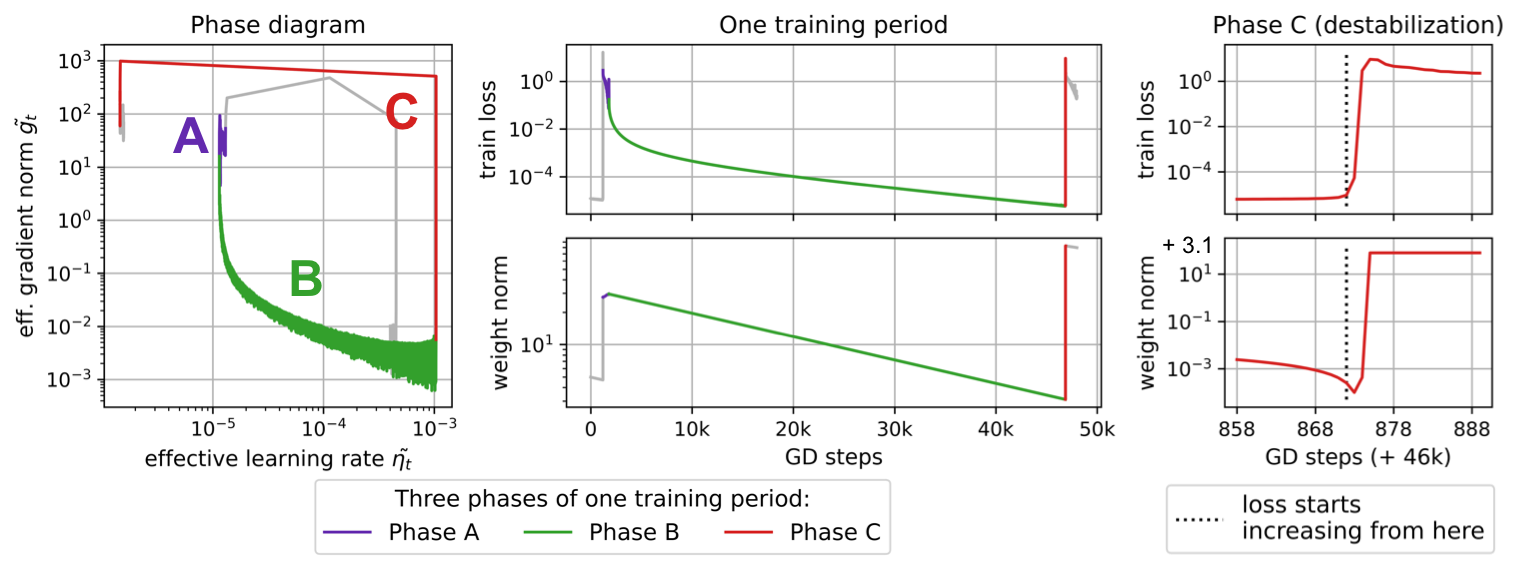} 
  \caption{
  A closer look at one training period for scale-invariant ConvNet on CIFAR-10 trained using full-batch GD with weight decay of 0.001 and the learning rate of 0.5. Three phases of the training period are highlighted. 
  }
  \label{fig:one_cycle_gd}
\end{figure}

\begin{figure}[t]
  \centering
  \centerline{
  \begin{tabular}{cc}
  {\small Fix weight norm at initialization} & {\small Fix weight norm before destabilization}  \\
  \includegraphics[width=0.5\textwidth]{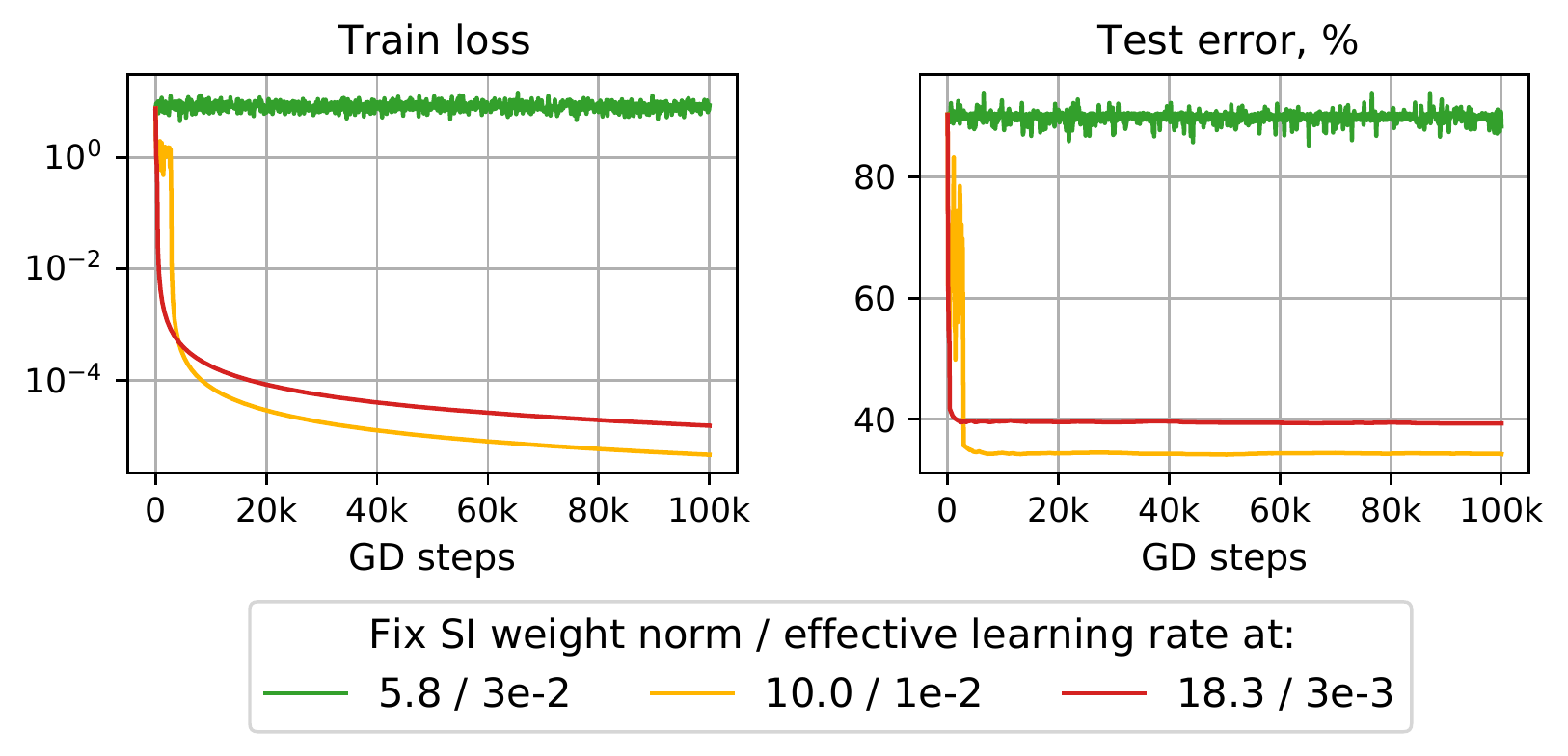} & \includegraphics[width=0.5\textwidth]{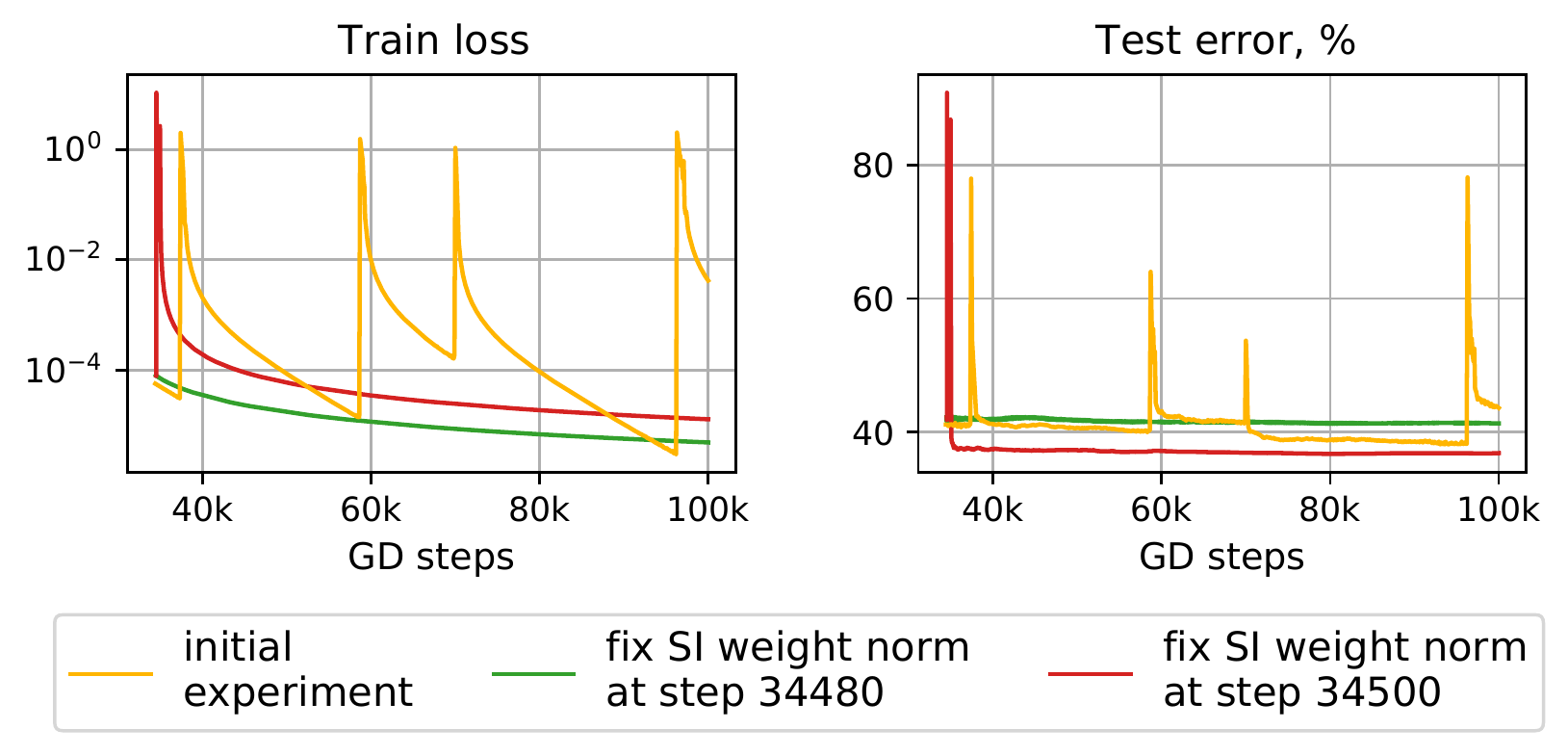}
  \end{tabular}}
  \caption{
  The absence of the periodic behavior for training with the fixed weight norm. Scale-invariant ConvNet on CIFAR-10 trained using full-batch GD with weight decay of 0.0001 and learning rate of 1.0. Left pair: the weight norm is fixed at random initialization of different scales. Right pair: the weight norm is fixed at some step of regular training before destabilization.
  }
  \label{fig:fix_elr_gd}
\end{figure}

In the main paper, we presented the periodic behavior results for SGD. In this section, we show that the periodic behavior is observed for full-batch GD training and hence is not a consequence of stochastic training. We replicate all experiments of Section~\ref{sec:1}: Figure~\ref{fig:cycles_demo_gd} visualizes training dynamics for different learning rate values, Figure~\ref{fig:one_cycle_gd} presents a closer look at one period of training (see also Figure~\ref{fig:cos_bounds} for the plots of cosines between adjacent steps), and Figure~\ref{fig:fix_elr_gd} replicates the ablation experiment with fixing the weight norm. All the effects discussed in the main text for the SGD case hold for the GD case.  We note that phase $B$ is longer for full-batch GD training because the absence of stochasticity allows stable training at lower train loss, and destabilization occurs later. 

\section{Bounds on the effective gradient norm and $\delta$-jumps}
\label{app:grad_bounds}
\begin{figure}[t]
  \centering
  \centerline{
 \begin{tabular}{cc}
 {\small SGD} & {\small Full-batch GD}\\
 \includegraphics[width=0.5\textwidth]{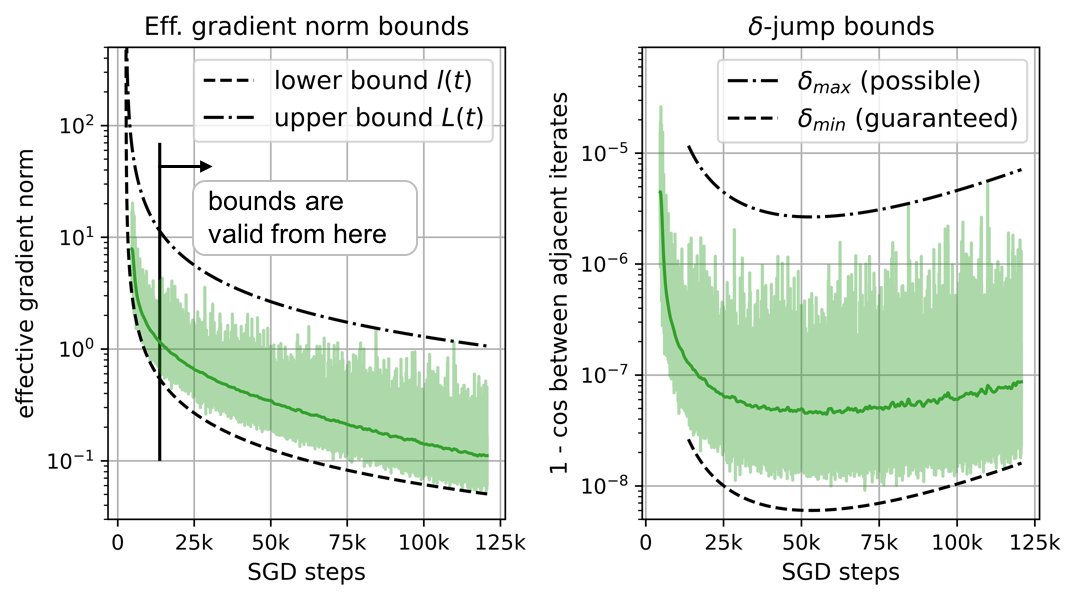} & \includegraphics[width=0.5\textwidth]{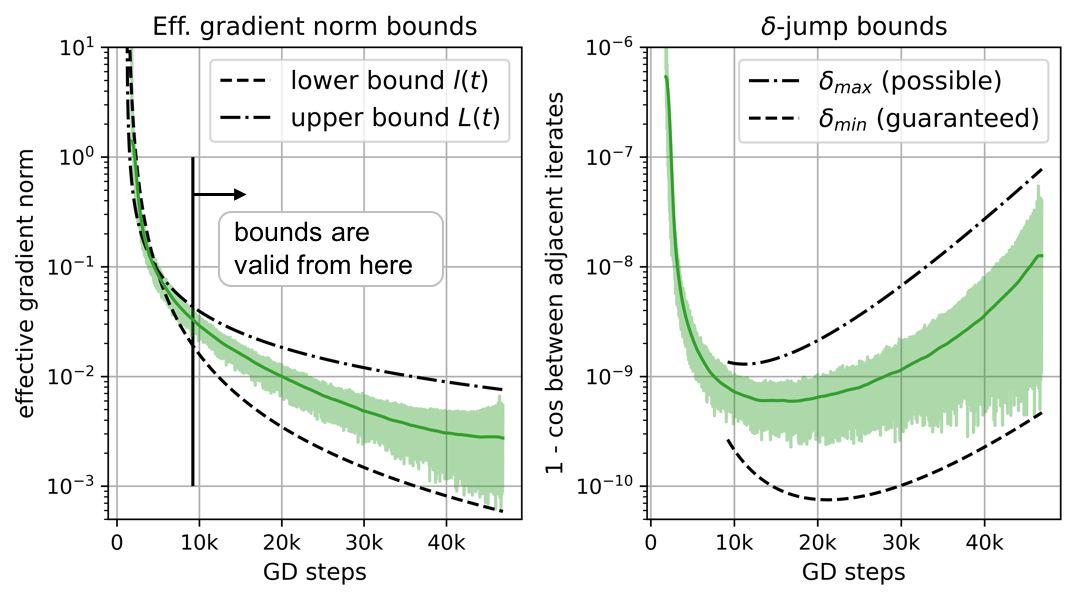}
  \end{tabular}}
  \caption{
  Effective gradient norm and cosine distance between weights at adjacent (S)GD steps, presented along with their smoothed trends. Phase $B$ of one period of training scale-invariant ConvNet on CIFAR-10 is shown. Weight decay~/~learning rate: 0.001~/~0.01 for SGD, 0.0001~/~0.5 for GD. $\delta$-jump bounds are obtained using the bounds on the effective gradient norm.
  }
  \label{fig:cos_bounds}
\end{figure}

In Section~\ref{sec:1}, we compared cosine distance between weights at adjacent SGD steps of phase $B$ with theoretically derived bounds for $\delta$-jumps from Section~\ref{sec:deltajumps}. In Figure~\ref{fig:cos_bounds}, right pair, we present a similar comparison for the full-batch GD case: the effect of both bounds and the cosine metric itself growing in the second half of the phase is even more prominent for the GD case than for SGD.
Below we describe how we choose the local bounds $\ell$ and $L$ on the effective gradient norm $\tilde{g}_t$ 
which are used in the theoretical bounds. All bounds are visualized in Figure~\ref{fig:cos_bounds}.

In both GD and SGD cases, we chose $\ell(t)$ and $L(t)$ as smooth functions of $t$. Note that taking such dynamical bounds does not contradict our theoretical results (see Remark~\ref{rem:grad_bounds}). For the SGD case, we chose  $\ell(t) = \frac{c}{t - t_0}$ and $L(t) = \frac{C}{t - t_0}$, where $t_0$ is the first iteration of the considered training period. For the GD case we used the same approach, but had to take $\ell(t) = \frac{c}{(t - t_0)^2}$ to better mimic the behavior of the lower envelope of the effective gradients norm. We handpick constants $0 < c < C$ and iteration $t_{\mathrm{valid}}$ separately for SGD and GD cases so that
\begin{equation}
    \label{eq:grad_bounds}
    \ell(t) \le \tilde{g}_t \le L(t)
\end{equation} 
for all $t\geqslant t_{\mathrm{valid}}$ in phase $B$.

\section{Optimization of common scale-invariant functions with weight decay}
\label{app:simplefunc}
\begin{figure}
  \centering
  \centerline{
 \begin{tabular}{c}
 {\small No weight decay~-- convergence} \\
 \includegraphics[width=\textwidth]{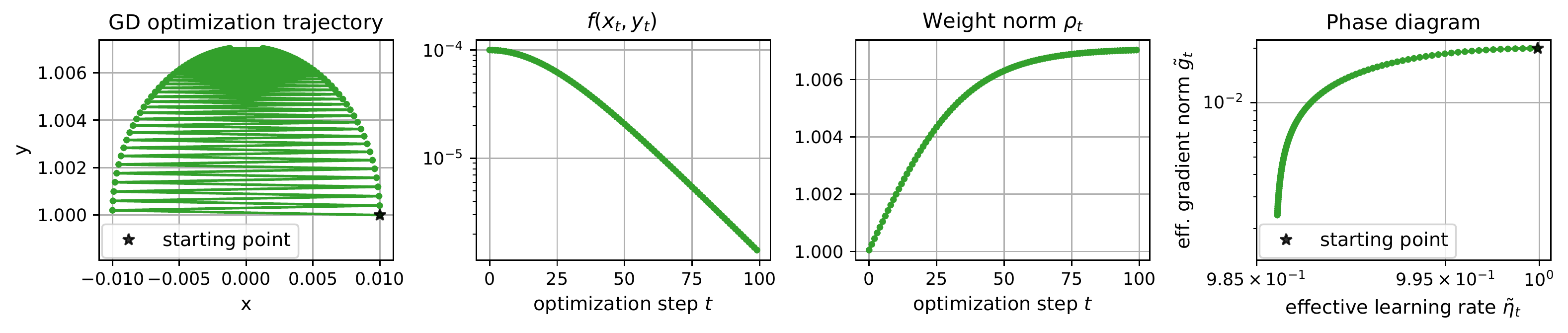} \\
 {\small Weight decay $\lambda = 0.01$~-- periodic behavior} \\
  \includegraphics[width=\textwidth]{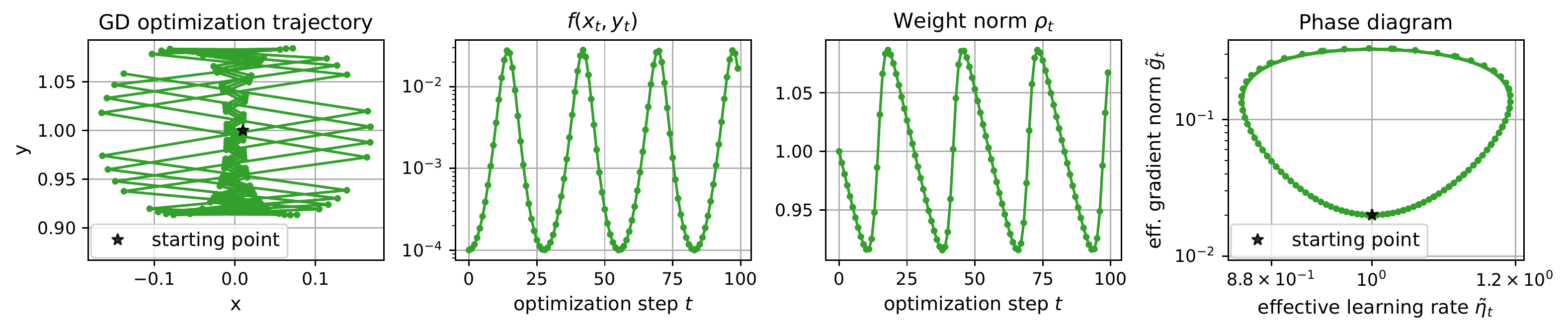} \\
  {\small Weight decay $\lambda = 0.01$~-- phases of a single period} \\
  \includegraphics[width=\textwidth]{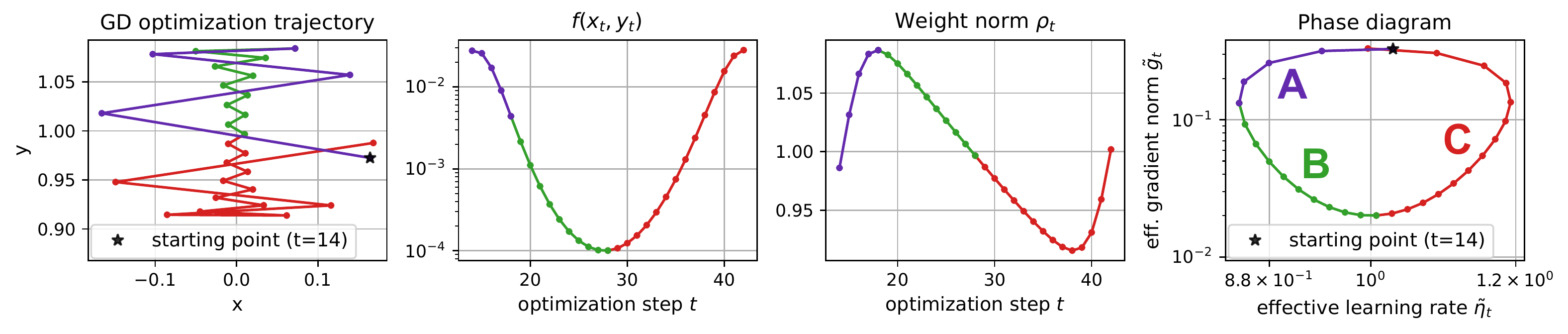} \\
  \end{tabular}}
  \caption{Minimization of a simple scale-invariant function $f(x,y) = x^2/(x^2+y^2)$ with and without weight decay. For all experiments the initial point $(x_0,y_0)=(0.01, 1.0)$, learning rate $\eta = 1$. }
  \label{fig:simplefunc}
\end{figure}

In this section, we show that periodic behavior may be observed not only when training neural networks but also during gradient decent optimization of common scale-invariant functions with weight decay and a constant learning rate. As an example we consider a function of two variables $f(x,y) = \frac{x^2}{x^2+y^2}$, which is naturally scale-invariant. The minimum value of $f$ equals $0$ and is achieved at any point with $x=0$. 

If we minimize $f$ without weight decay, the optimization procedure converges to a stationary point since its effective learning rate monotonically decays, as can be seen in the top row of Figure~\ref{fig:simplefunc}. This behavior accords with the results of~\citet{arora2018theoretical}. 

However, with weight decay we can observe the same periodicity of the optimization dynamics as demonstrated by experiments with neural networks (see the middle row of Figure~\ref{fig:simplefunc}). Moreover, in this case, the optimization experiences the same three phases in the period (see the bottom row of Figure~\ref{fig:simplefunc}, which is analogous to Figure~\ref{fig:one_cycle} in the main text). 

This confirms that the periodicity of optimization dynamics is a general property of scale-invariant functions optimized with weight decay and is not specific to neural networks.

\section{Influence of learning rate and weight decay on the periodic behavior of scale-invariant networks}

\subsection{Fixed learning rate -- weight decay product}
\label{app:fixed_lr_wd}
\begin{figure}
  \centering
  \centerline{
 \begin{tabular}{c}
 {\small Fixed product $\eta\times\lambda = 1e-3$} \\
 \includegraphics[width=\textwidth]{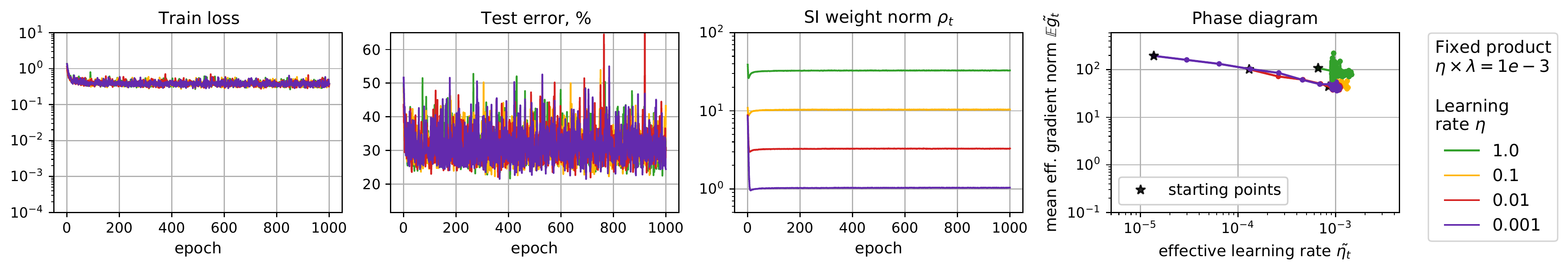} \\
 {\small Fixed product $\eta\times\lambda = 1e-4$} \\
  \includegraphics[width=\textwidth]{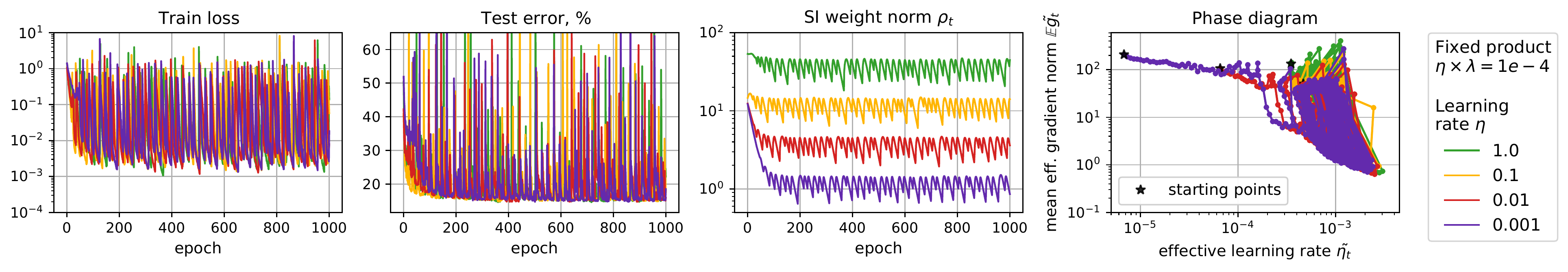} \\
  {\small Fixed product $\eta\times\lambda = 1e-5$} \\
  \includegraphics[width=\textwidth]{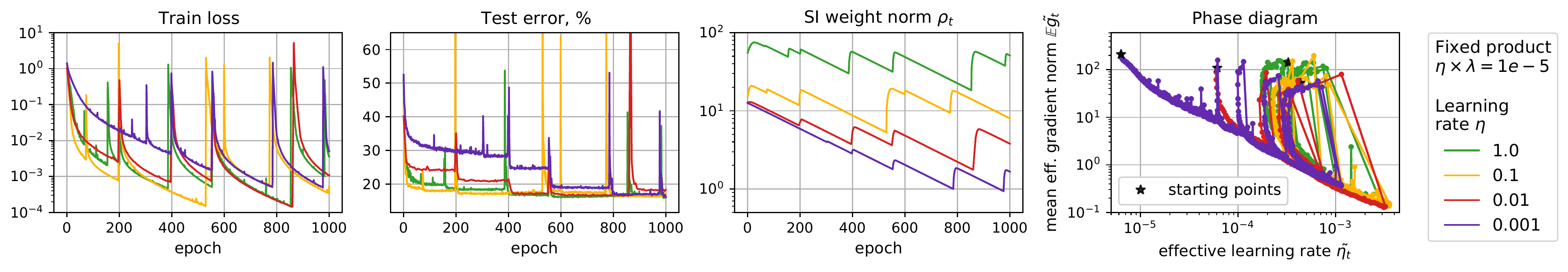} \\
  {\small Fixed product $\eta\times\lambda = 1e-6$} \\
  \includegraphics[width=\textwidth]{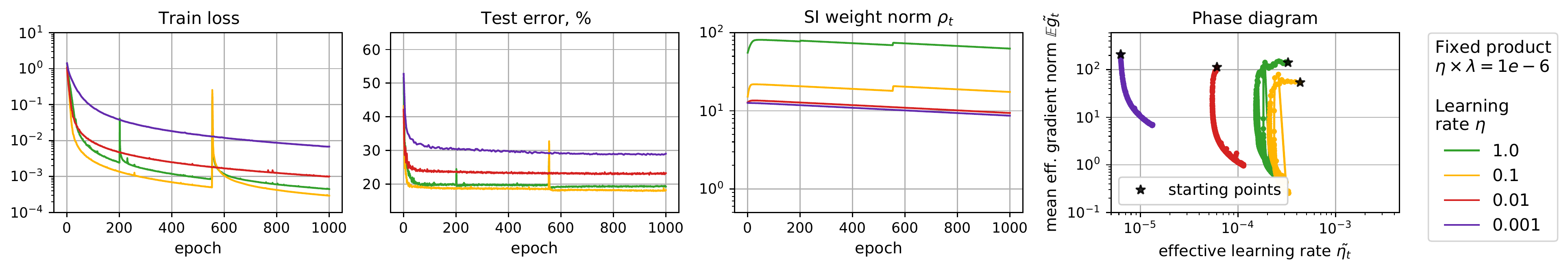} \\
  \end{tabular}}
  \caption{Training dynamics of scale-invariant ConvNet on CIFAR-10 trained with fixed learning rate -- weight decay products. Axes limits are the same in each column for convenient comparison.}
  \label{fig:fix_product_convnet_cifar10}
\end{figure}

\begin{figure}
  \centering
  \centerline{
 \begin{tabular}{cc}
 {\small Fixed product $\eta\times\lambda = 1e-4$} & {\small Fixed product $\eta\times\lambda = 1e-5$}\\
 \includegraphics[width=0.5\textwidth]{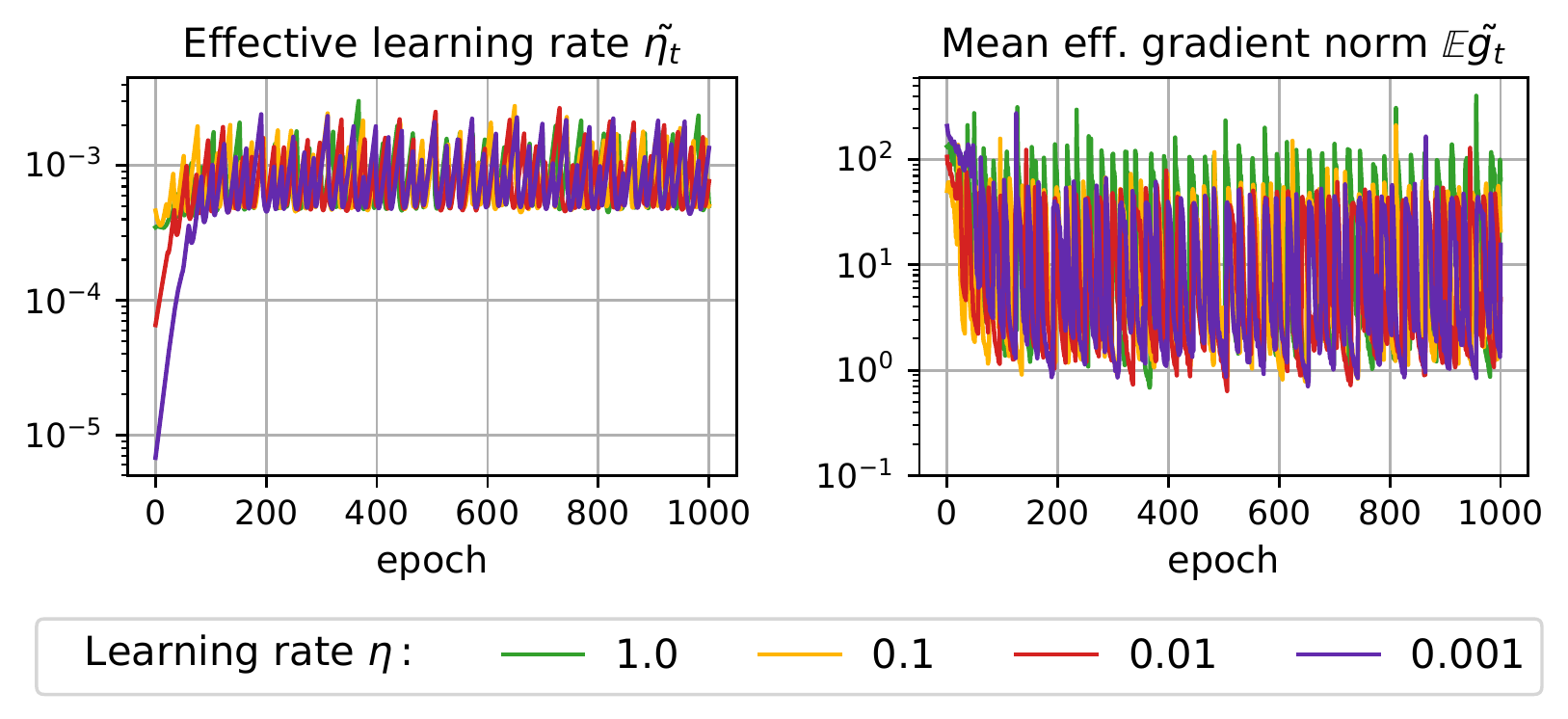} & \includegraphics[width=0.5\textwidth]{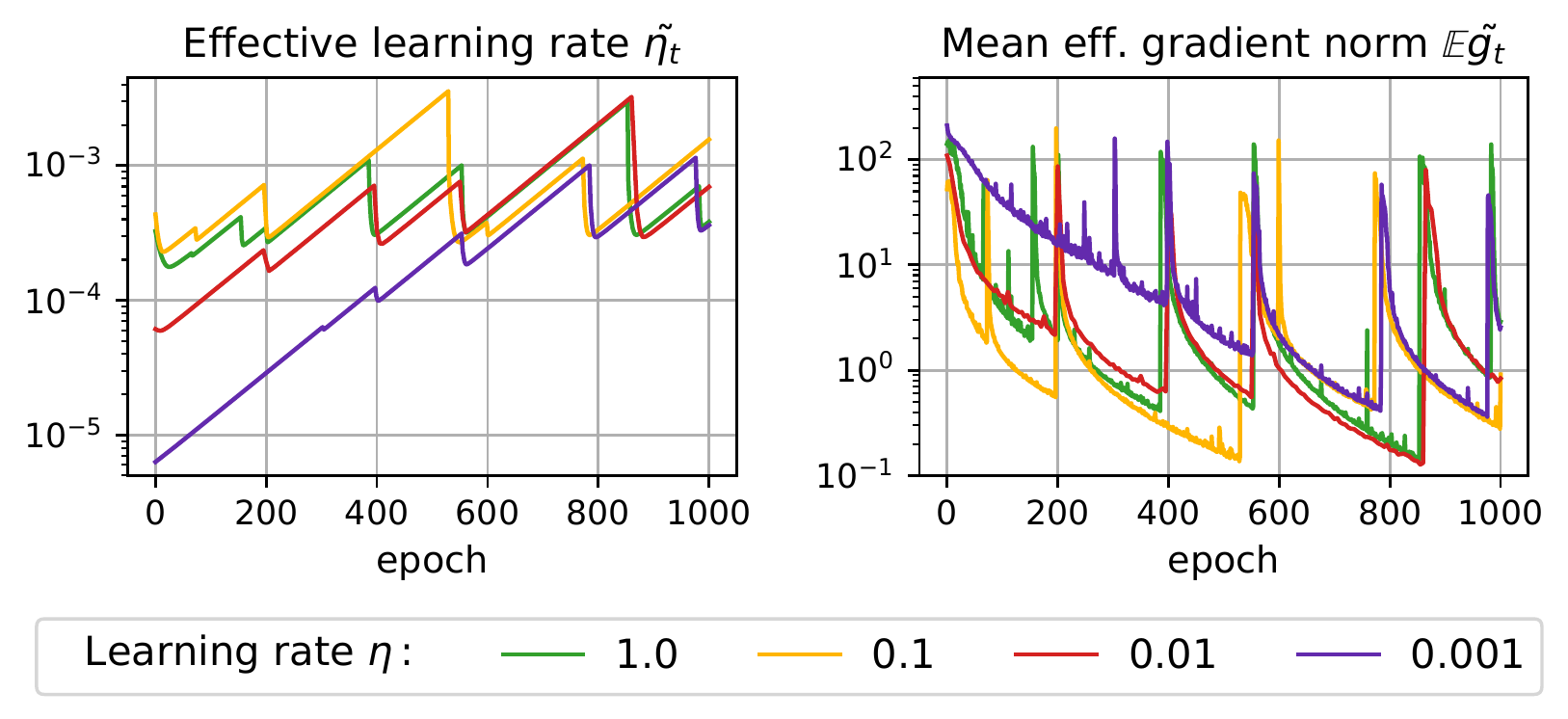}
  \end{tabular}}
  \caption{A closer look at dynamics of the effective learning rate and mean effective gradient norm of scale-invariant ConvNet on CIFAR-10 trained with two different fixed learning rate -- weight decay products. Axes limits are the same for corresponding metrics for convenient comparison.}
  \label{fig:fix_product_elr_grad}
\end{figure}

In this section, we discuss the effect of the learning rate -- weight decay product on the training process. Figure~\ref{fig:fix_product_convnet_cifar10} visualizes training progress for different values of the product (plot rows) and variable ratio of two specified hyperparameters (different lines in each row). We observe that training converges to similar consistent behavior with the fixed learning rate -- weight decay product. Specifically, the frequency of the periods, the minimal achieved train loss and test error, and the ranges of the effective gradient norm and the effective learning rate are similar across different lines in one row. The last-mentioned ranges are visualized in more detail for selected setups in Figure~\ref{fig:fix_product_elr_grad}. The described empirical results agree with Remark~\ref{rem:elr_eq} in Appendix~\ref{app:norm_eq}. Mainly, the remark states that with a fixed learning rate -- weight decay product and bounded effective gradient norm, training converges to a bounded effective learning rate, and the effective learning rate bounds depend only on the effective gradient norm bounds. In practice, we observe that the last-mentioned bounds are similar across different ratios of weight decay and learning rate (see Figure~\ref{fig:fix_product_elr_grad}). Thus, the effective learning rate bounds are also similar across different ratios (see Figure~\ref{fig:fix_product_elr_grad}).  

However, although the characteristics of the \emph{consistent} periodic behavior are similar across different ratios of the learning rate and the weight decay when their product is fixed, the length of the \emph{warm-up} stage may vary. The reason is that we use the standard initialization for all networks, i.e., the same initial weight norm for all combinations of hyperparameters. At the same time, given different ratios of weight decay and learning rate, the weight norm converges to different ranges (see Figure~\ref{fig:fix_product_convnet_cifar10} and Proposition~\ref{prop:norm_eq}). The final weight norm may substantially differ from the initial weight norm, and the larger the difference, the longer the warm-up stage.

We note, however, that, according to Proposition~\ref{prop:hyp_resc_inv} in Appendix~\ref{app:hyp_resc_inv}, if we fixed the direction of initialization (i.e., the point on the unit sphere) and then appropriately rescaled it (proportionally to the square root of the learning rate), the training dynamics would be exactly the same for different ratios of learning rate and weight decay, given their product is unchanged, including the warm-up stage. 

\subsection{Fixed weight decay and different learning rates}
\label{app:fully_invariant}
\begin{figure}
  \centering
  \centerline{
 \begin{tabular}{c}
 {\small ConvNet on CIFAR-100} \\
  \includegraphics[width=\textwidth]{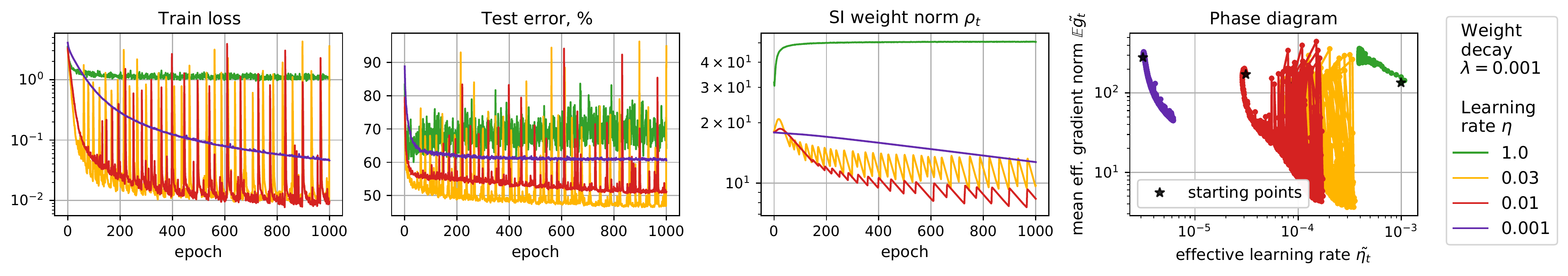}\\
  {\small ResNet-18 on CIFAR-10} \\
  \includegraphics[width=\textwidth]{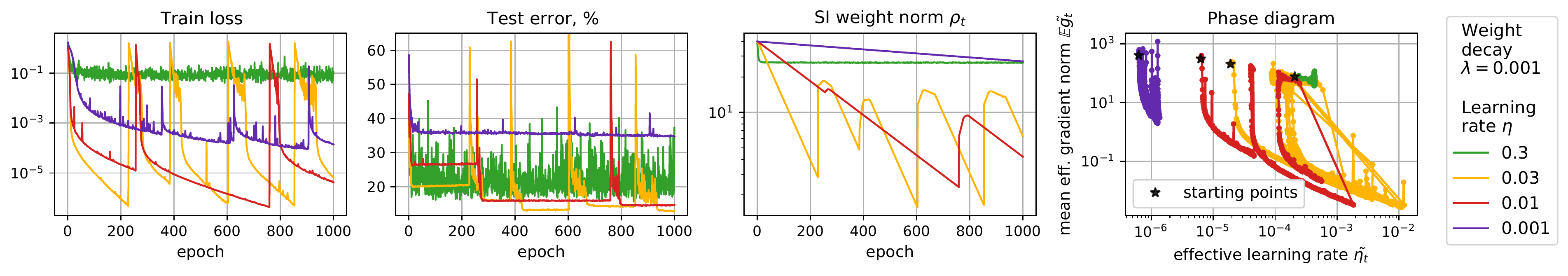}\\
  {\small ResNet-18 on CIFAR-100} \\
  \includegraphics[width=\textwidth]{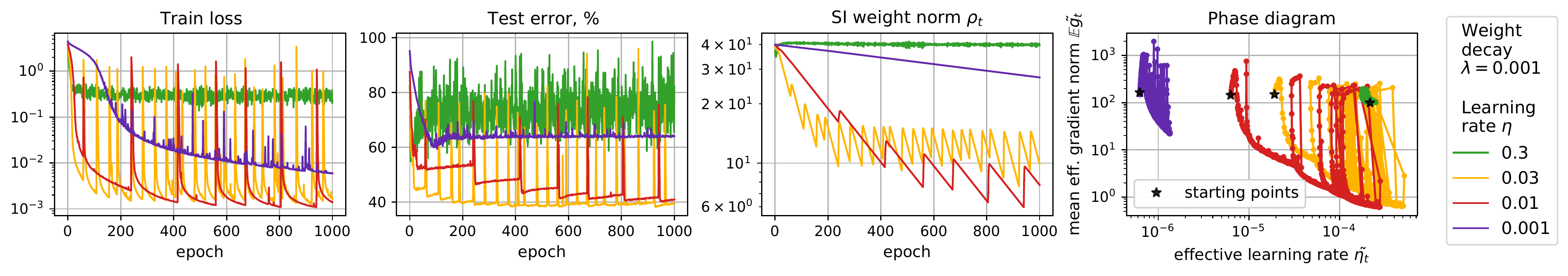}
  \end{tabular}}
  \caption{Training dynamics of scale-invariant networks trained with fixed weight decay and different learning rates.}
  \label{fig:different_lrs}
\end{figure}

Figure~\ref{fig:different_lrs} supplements Figure~\ref{fig:cycles_demo} and shows how the learning rate affects the periodic behavior for different dataset-architecture pairs when the weight decay is fixed. For CIFAR-100, we had to increase the ConvNet's width factor up to 64 and the last layer's weight norm up to 20 to ensure the network is able to learn the train dataset and achieve low train loss. The general picture is the same as described in Section~\ref{sec:2}: the periodic behavior is absent for too low or too high learning rates and present for a range of learning rate values, which also allow lower test error. Interestingly, for ResNet on CIFAR-10 with the learning rate of 0.03, phase $A$ is noisy and quite long because of the relatively high learning rate, but training still proceeds to phase $B$, while for larger learning rate, training gets stuck at high train loss.

\subsection{Fixed learning rate and different weight decays}
\label{app:var_wd}
\begin{figure}
  \centering
  \centerline{
 \begin{tabular}{c}
 {\small ConvNet on CIFAR-10} \\
 \includegraphics[width=\textwidth]{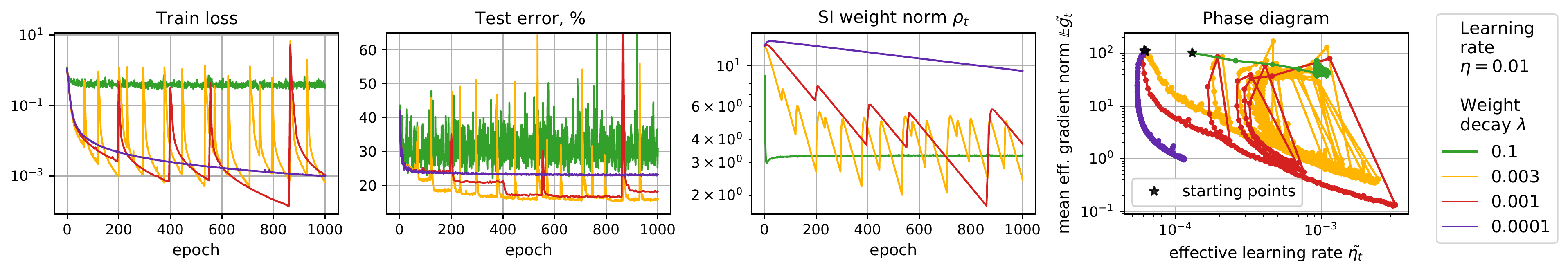} \\
 {\small ConvNet on CIFAR-100} \\
  \includegraphics[width=\textwidth]{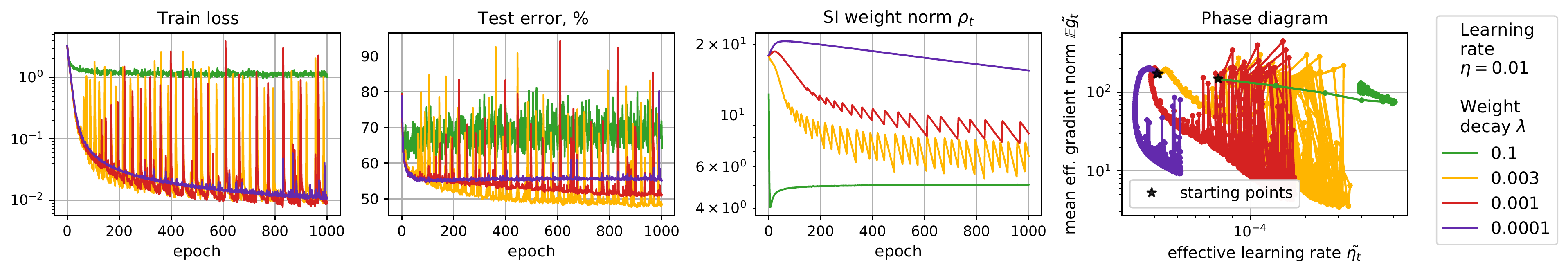} \\
  {\small ResNet-18 on CIFAR-10} \\
  \includegraphics[width=\textwidth]{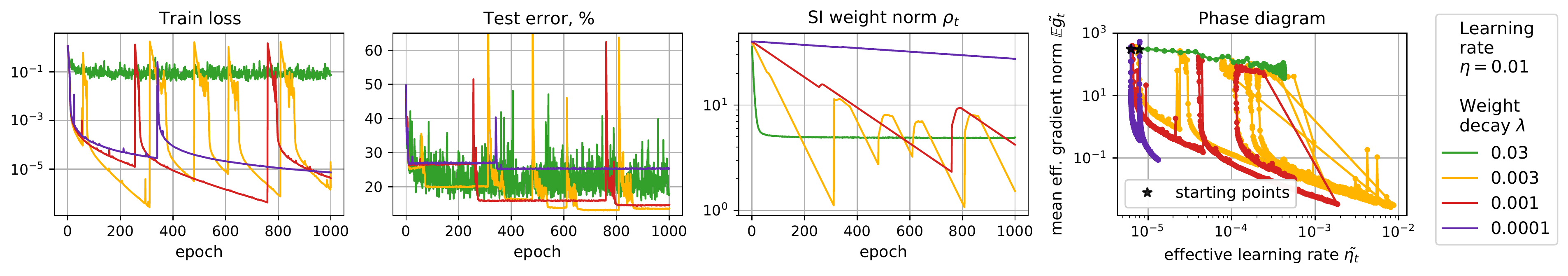}\\
  {\small ResNet-18 on CIFAR-100} \\
  \includegraphics[width=\textwidth]{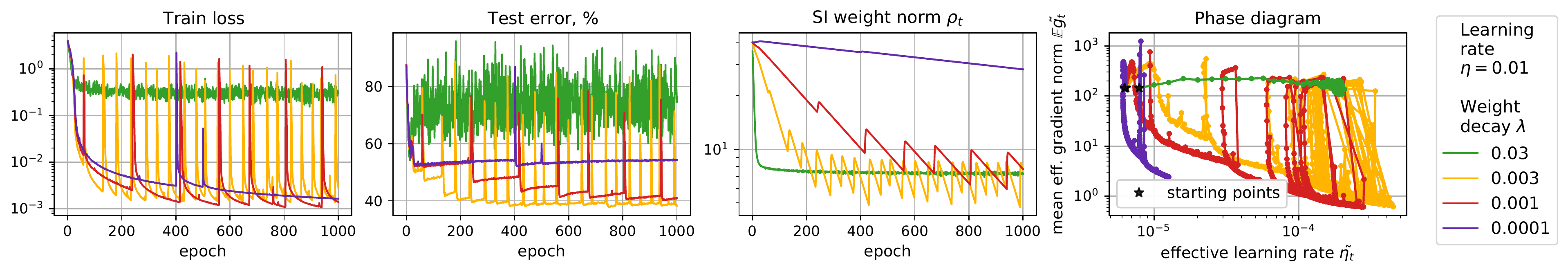} \\
  \end{tabular}}
  \caption{Training dynamics of scale-invariant networks trained with the fixed learning rate and different weight decays.}
  \label{fig:different_wds}
\end{figure}

Figure~\ref{fig:different_wds} shows the periodic behavior when the learning rate is fixed, and the weight decay is varied for different dataset-architecture pairs. The general observations are the same as when the learning rate is varied with the fixed weight decay. Notably, the periodic behavior is absent for too low or too high weight decay coefficients and present for a range of weight decay values, which also allow reaching lower test error. Further, using a larger weight decay increases the frequency of the periods.

\section{Influence of the last layer weight matrix norm}
\label{app:last_layer}
\begin{figure}
  \centering
  \centerline{
 \begin{tabular}{c}
 {\small ConvNet on CIFAR-10 ($\eta=0.03$, $\lambda=0.001$)} \\
 \includegraphics[width=\textwidth]{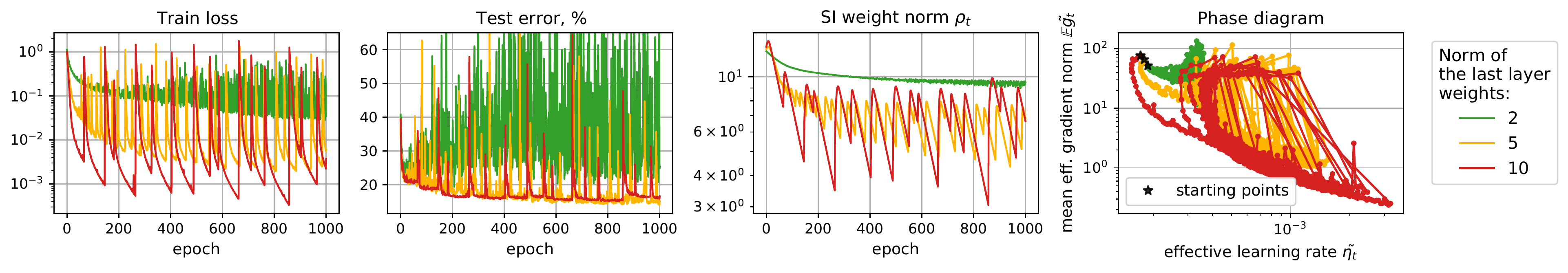} \\
 {\small ResNet-18 on CIFAR-100 ($\eta=0.03$, $\lambda=0.001$)} \\
 \includegraphics[width=\textwidth]{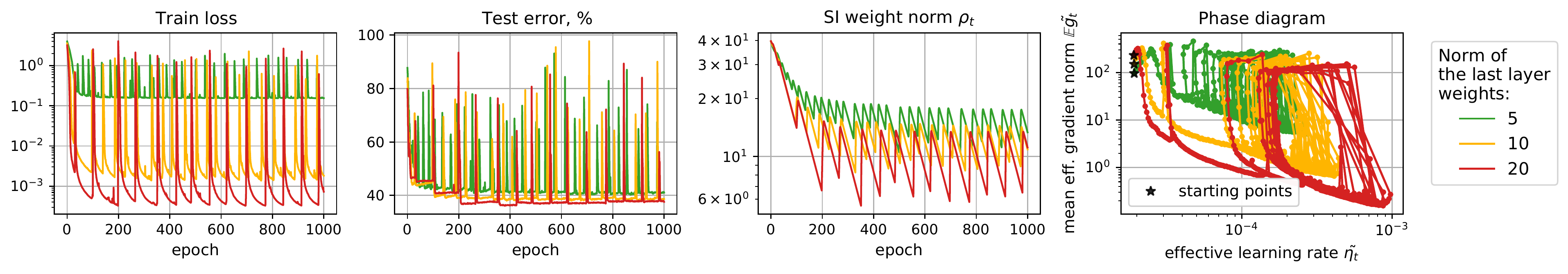} \\
  \end{tabular}}
  \caption{Influence of the last layer weight matrix norm on the periodic behavior.}
  \label{fig:different_scales}
\end{figure}

In scale-invariant neural networks, we fix the weights of the last layer. 
Moreover, we renormalize the weight matrix to the specified weight norm, which becomes a new hyperparameter.
This hyperparameter determines the level of the neural network's confidence in its predictions, and, in the main text, we set it to a large value (10) to achieve high confidence and to make our setup closer to the conventional neural network training (when all parameters are trained). In this section, we discuss the influence of the specified hyperparameter on periodic behavior.

Figure~\ref{fig:different_scales} shows results for
ConvNet on CIFAR-10 and ResNet on CIFAR-100 and
different values of the last layer's weight norm. The lowest presented last layer's weight norms are close to the norms obtained at random initialization without rescaling. Using low last layer's weight norm leads to low network's confidence which prohibits reaching low train loss and may result in the absence of the periodic behavior. In the main text, we use larger values of the last layer's weight norm, which circumvents this issue. 

\section{Minima achieved at different training periods}
\label{app:ensembles}
\begin{figure}
  \centering
  \centerline{
 \begin{tabular}{cc}
 {\small ConvNet on CIFAR-100} & {\small ResNet-18 on CIFAR-10} \\
  \includegraphics[width=0.5\textwidth]{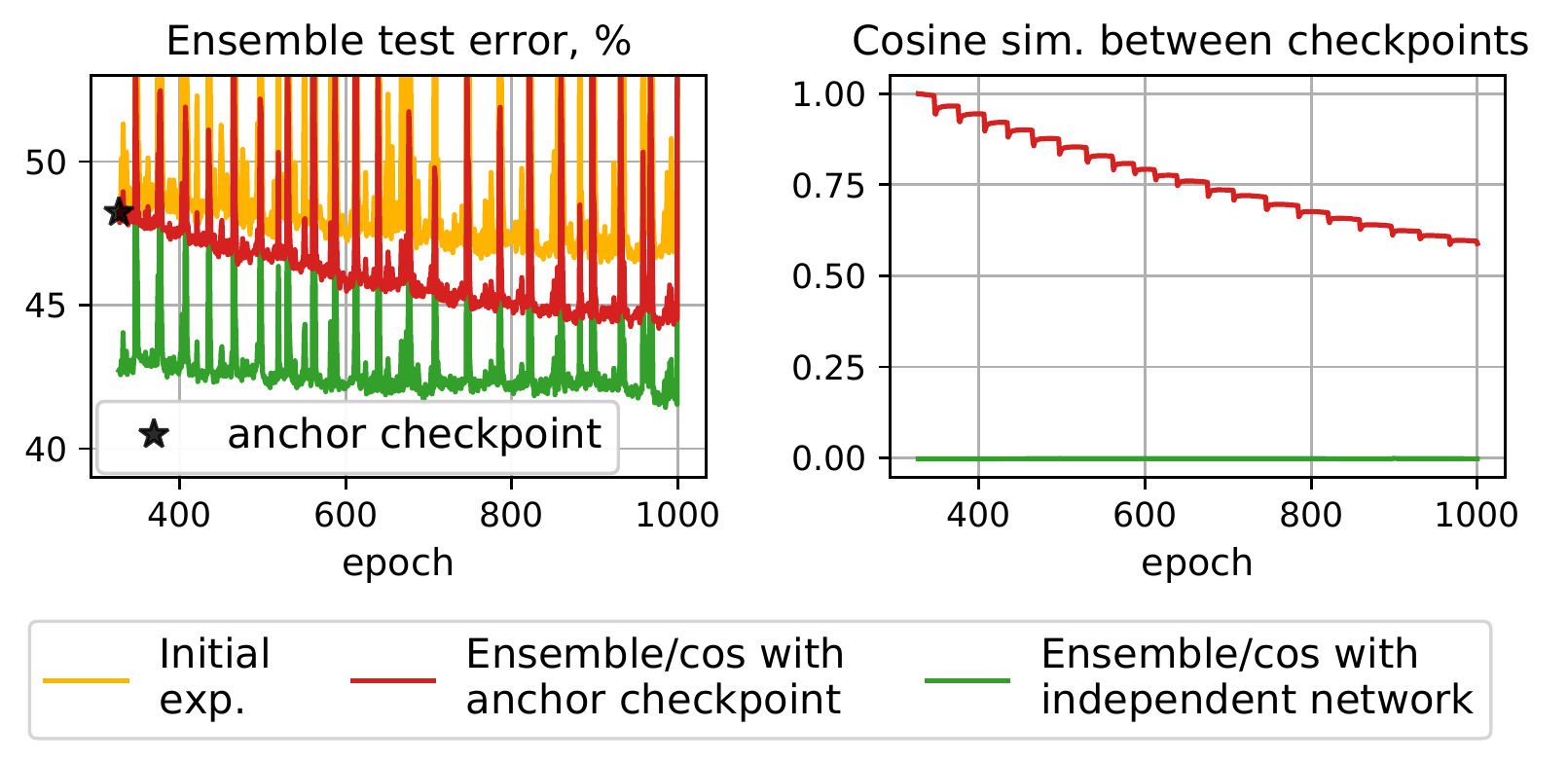}&\includegraphics[width=0.5\textwidth]{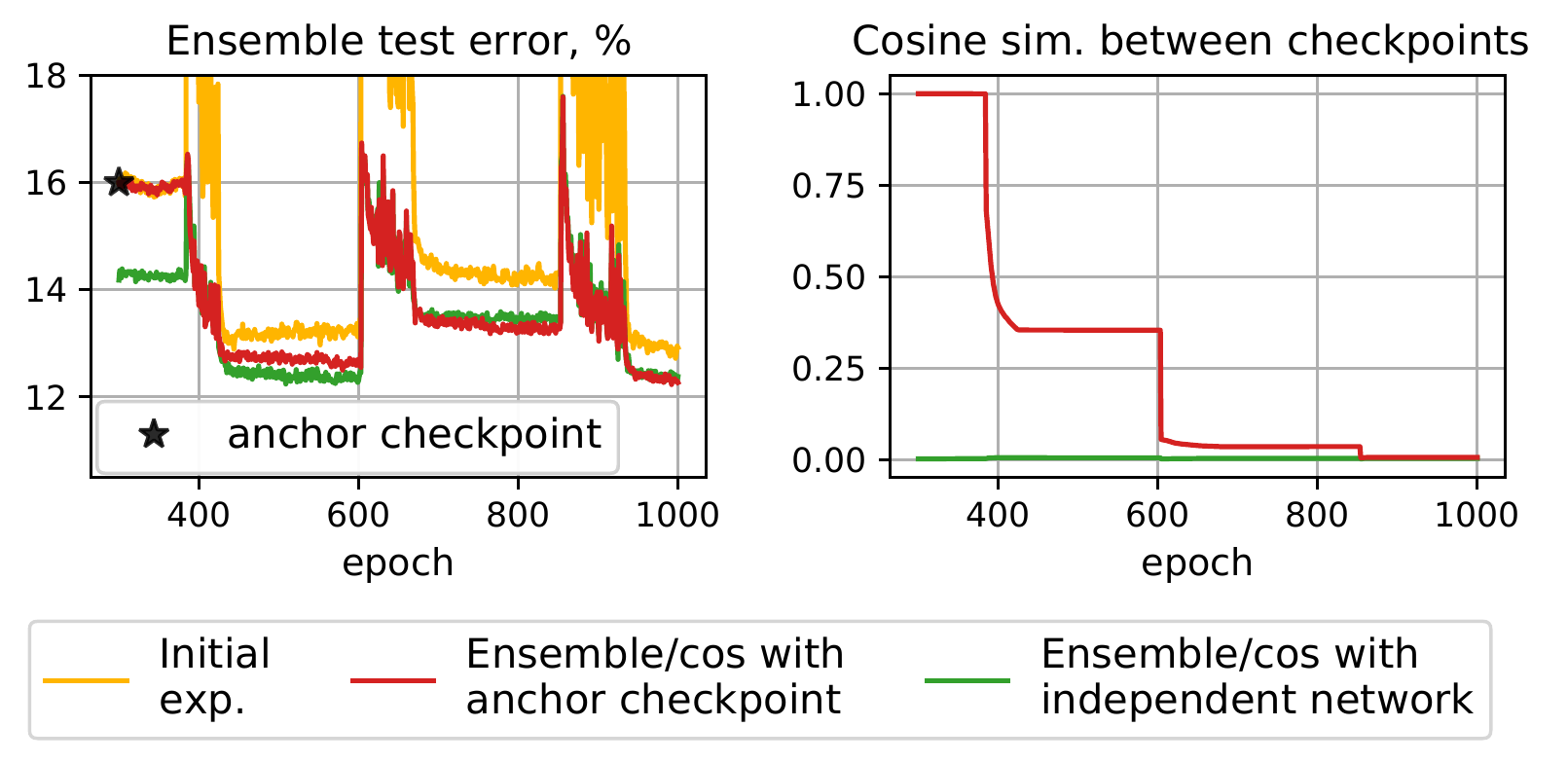} \\
  \end{tabular}}
  \caption{Similarity in the weight space (cosine sim.) and in the functional space (ensemble test error) for different checkpoints of training scale-invariant ConvNet on CIFAR-100 (left pair) and ResNet on CIFAR-10 (right pair) using SGD with weight decay of 0.001 and learning rate of 0.03.}
  \label{fig:ensembles_app}
\end{figure}

Figure~\ref{fig:ensembles_app} supplements Figure~\ref{fig:ensembles} for analyzing the weight/functional similarity of optima achieved at different training periods. The general observations are the same as in Section~\ref{sec:2}. Interestingly, the ensemble of two models spawned by optima from different periods can reach the error of two independent networks ensemble for both architectures on the CIFAR-10 dataset and does not reach one on the CIFAR-100 dataset (in given epochs budget).

\section{Practical modifications}
\label{app:other_setups}

\begin{figure}
  \centering
  \begin{tabular}{c}
  {\small ConvNet on CIFAR-100} \\
  \includegraphics[width=0.5\textwidth]{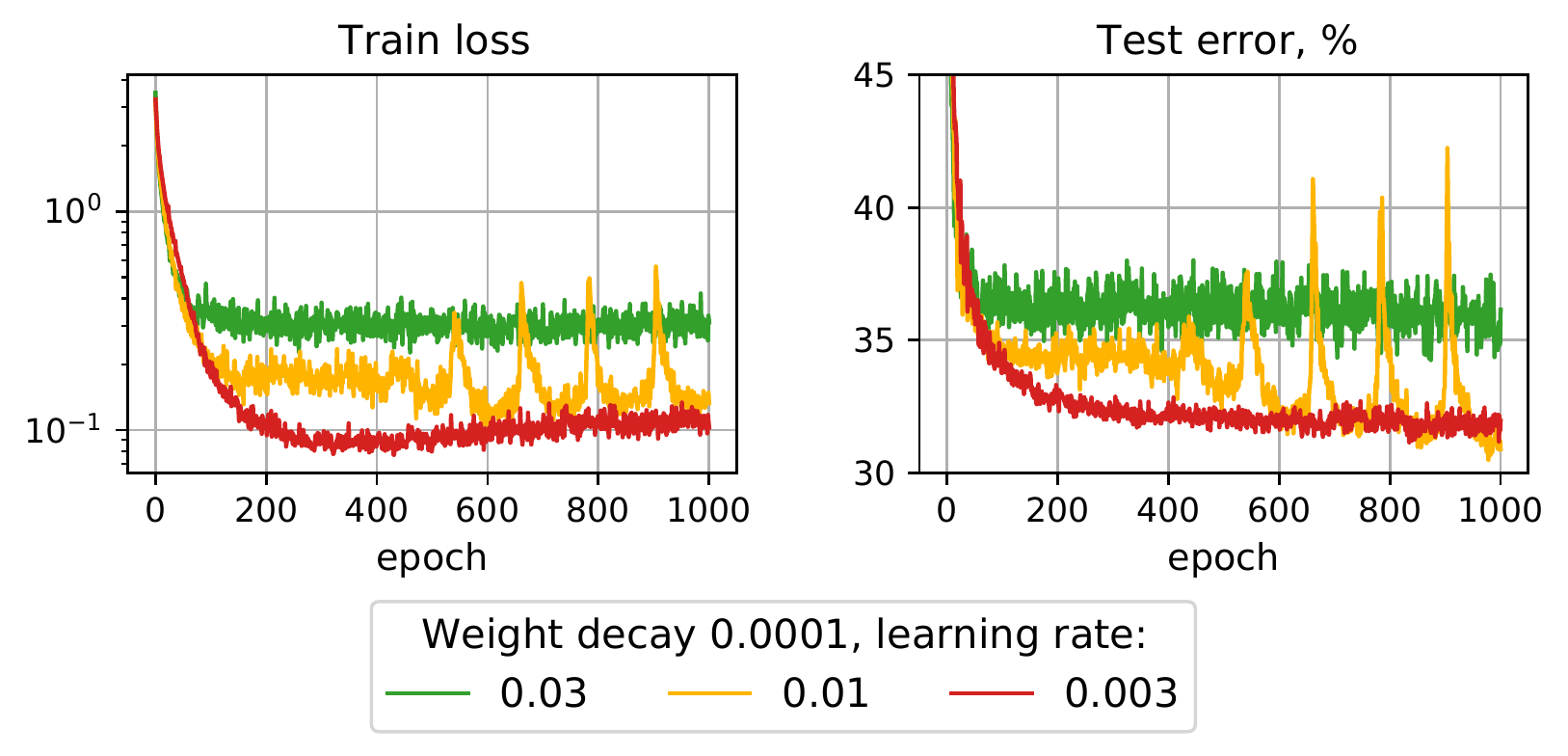}
  \end{tabular}\\
  \centerline{
 \begin{tabular}{cc}
 {\small ResNet-18 on CIFAR-10} & {\small ResNet-18 on CIFAR-100} \\
 \includegraphics[width=0.5\textwidth]{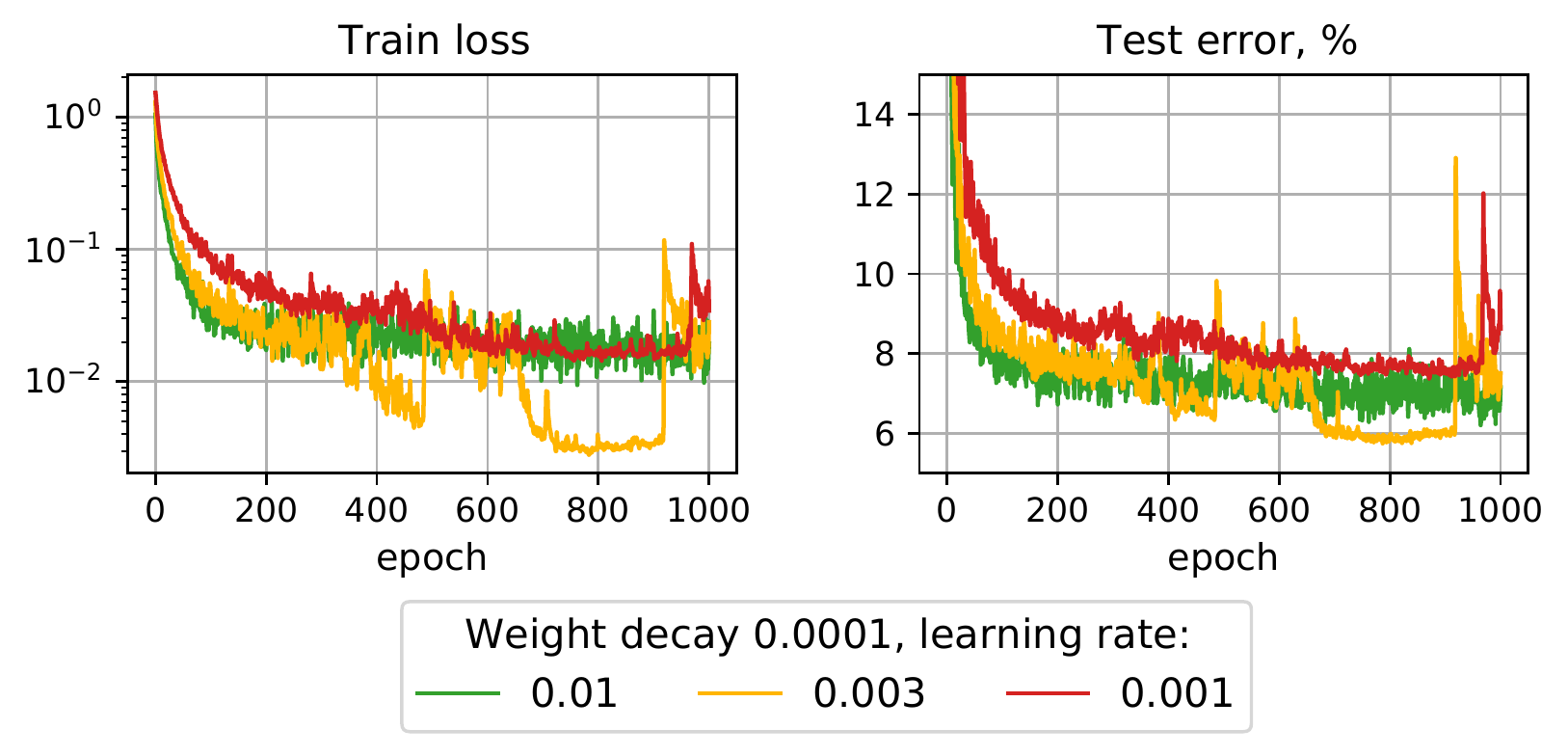} & \includegraphics[width=0.5\textwidth]{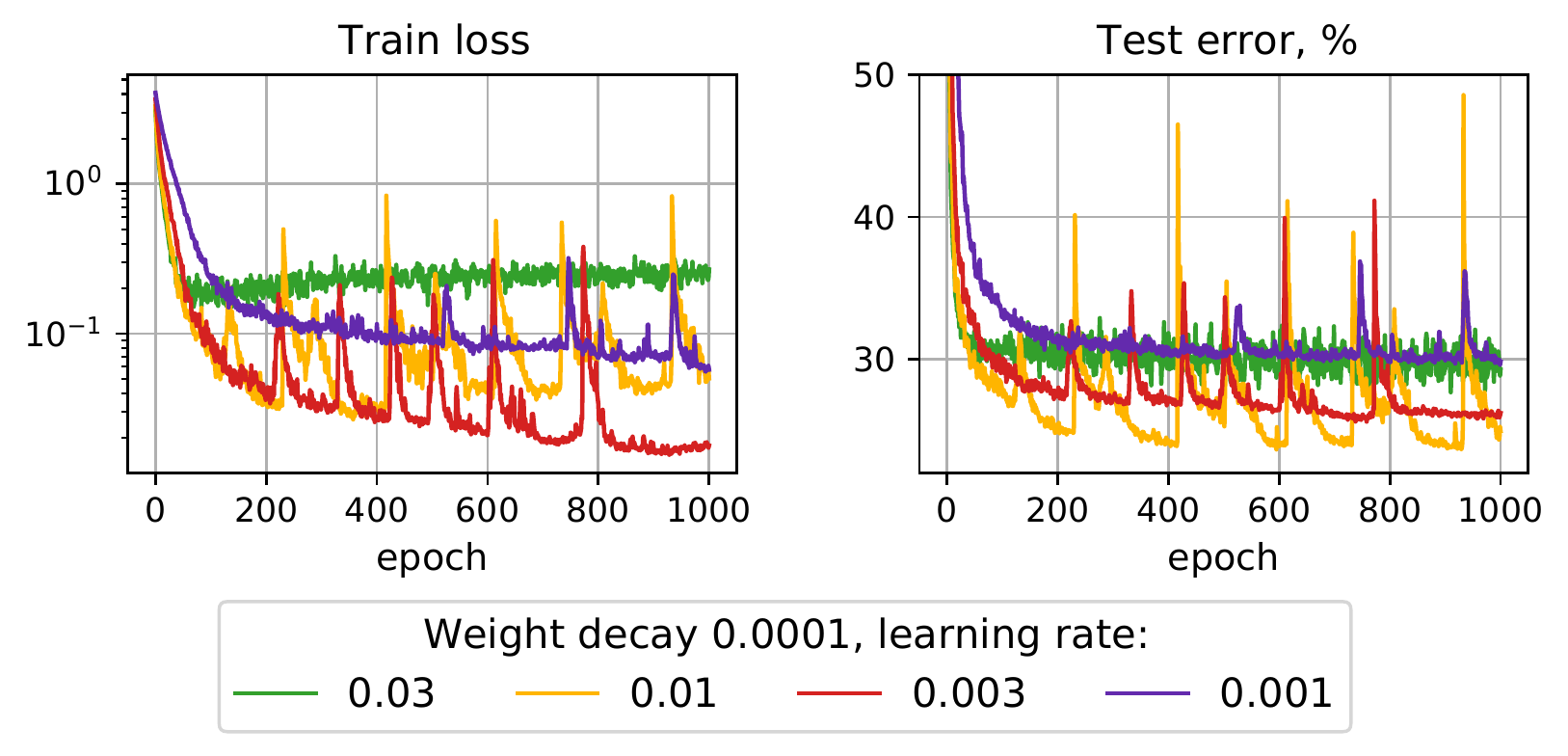}\\
  \end{tabular}}
  \caption{Training dynamics of networks trained with more practical modifications, i.e., with learnable non-scale-invariant parameters, momentum, and augmentation (all modifications together).}
  \label{fig:practical_app}
\end{figure}

\begin{figure}
  \centering
  \centerline{
 \begin{tabular}{c}
 {\small Layer Normalization} \\
 \includegraphics[width=\textwidth]{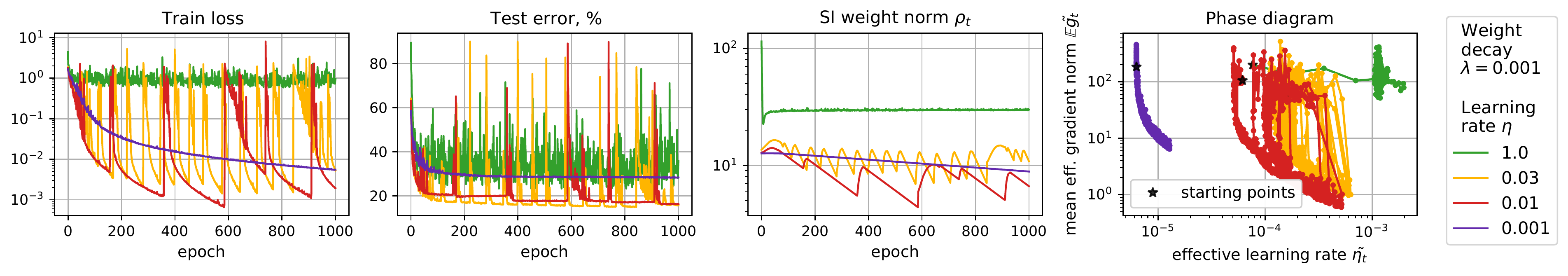} \\
 {\small Instance Normalization} \\
 \includegraphics[width=\textwidth]{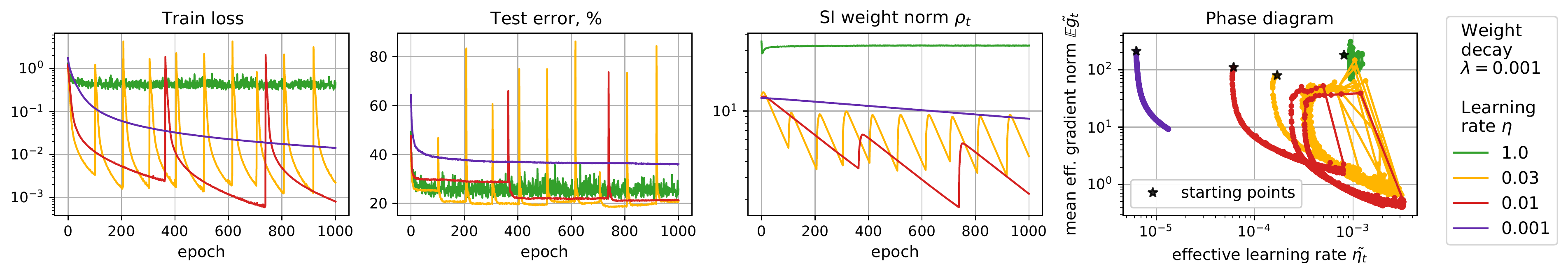} \\
  \end{tabular}}
  \caption{Training dynamics of scale-invariant ConvNet with other normalization approaches on CIFAR-10.}
  \label{fig:different_norms}
\end{figure}

\begin{figure}
  \centering
  \centerline{
 \includegraphics[width=\textwidth]{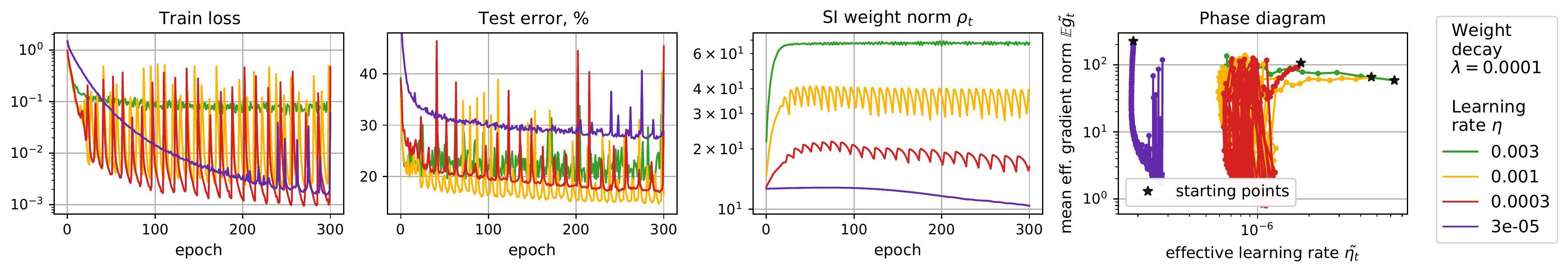} }
  \caption{Training dynamics of scale-invariant ConvNet on CIFAR-10 trained using Adam.}
  \label{fig:adam}
\end{figure}

Figure~\ref{fig:practical_app} supplements Figure~\ref{fig:mom_aug_full} and shows the presence of the periodic behavior in a more practical setting, i.e., with trainable non-scale-invariant parameters, momentum, and data augmentation, for ConvNet on CIFAR-100 and ResNet on CIFAR-10 and CIFAR-100. For a more detailed discussion, see Section~\ref{sec:3} in the main text. 

We also consider training neural networks with a more sophisticated optimizer, Adam~\cite{adam}, and show the presence of the periodic behavior for ConvNet on CIFAR-10 in Figure~\ref{fig:adam}.

In order to show that our results extrapolate to other normalization approaches besides batch normalization, we train ConvNet on CIFAR-10 using layer normalization~\cite{ba2016layer} and instance normalization~\cite{ulyanov2016instance} and demonstrate the presence of the periodic behavior in this setting in Figure~\ref{fig:different_norms}.

\section{Comparison with previous works}
\label{app:comp_setups}

In this section, we compare our experimental setup with that of the prior art and point out the main factors for why previous experiments mostly do not show periodic behavior. 

As we stated in Section~\ref{sec:3}, periodic behavior is not usually observed when training normalized neural networks due to a relatively small epochs budget and usage of learning rate schedules. 
Moreover, some hyperparameters settings can make periods too slow or even unreachable, which both hinder observation of the periodic behavior in practice (see, e.g., the smallest and the highest learning rate curves in Figure~\ref{fig:cycles_demo}). 
Finally, the use of data augmentation and/or models that are too simple to learn a given dataset does not allow even the first period to be completed within a reasonable time frame.
These are the key reasons periodic behavior was mainly not reported in the literature previously.
Below we discuss the particular aspects of several most related works.

One of the works closest to ours,~\citet{li2020reconciling}, discovers the unstable behavior of full-batch GD training of scale-invariant networks and at the same time reports convergence to a constant equilibrium when training with SGD.
We suppose that the experiments of~\citet{li2020reconciling} with full-batch GD depict exactly our periodic behavior.
Speaking of SGD experiments, we suspect that, despite a large epochs budget,~\citet{li2020reconciling} did not encounter periodic behavior in most of their experiments due to data augmentation, different hyperparameters settings, and learning rate schedules.
In other words, they mainly observed a prolonged phase $A$ in their experiments without reaching the end of even the first period, which may seem like convergence to a stable equilibrium.

\citet{wan2020spherical}, who also study the convergence of scale-invariant parameters dynamics to the equilibrium (which, however, is now \emph{dynamical}, i.e., depends on the behavior of effective gradients, in par with our work), did not find periods in their experiments as well.
This can also be attributed to data augmentation and learning rate schedules but most importantly to short training, which does not allow finishing phase $A$ of the first period, as seen by the increasing effective gradients norm throughout training in Figure~2 therein.

As mentioned in the main text,~\citet{li2020understanding} discovered that training weight-normalized neural networks with improperly selected weight decay may become unstable and even result in training failure since the numerical gradient updates are beyond the representation of float.
This is the extreme case of destabilization in phase $C$ when scale-invariant parameters approach the origin too close and the gradients blow up so that training is already unable to recover due to numerical issues.
In our experiments, such situations did not occur, however, we hypothesize that they can be encountered when training very large networks equipped with both weight normalization and feature normalization, which may amplify the destabilization effect of approaching the origin.
Other experiments of~\citet{li2020understanding} did not reveal the periodic behavior for the same reasons as above: data augmentation, insufficient training duration, and learning rate schedules.

\end{document}